\documentclass[11pt,a4paper]{article}

\usepackage[margin=1in]{geometry}

\usepackage{xcolor}

\usepackage[utf8]{inputenc}

\usepackage[T1]{fontenc}
\usepackage{lmodern}
\usepackage[round]{natbib}

\usepackage{amsmath,amssymb,amsfonts,amsthm,mathtools}
\usepackage{booktabs,tabularx}
\usepackage{enumitem}
\usepackage{float,comment}
\usepackage[protrusion=true,expansion=false]{microtype}
\usepackage[hidelinks]{hyperref}
\newfloat{algorithm}{tbp}{loa}
\floatname{algorithm}{Algorithm}
\floatstyle{ruled}
\restylefloat{algorithm}
\newtheorem{theorem}{Theorem}
\newtheorem{lemma}[theorem]{Lemma}
\newtheorem{proposition}[theorem]{Proposition}

\newtheorem{assumption}[theorem]{Assumption}

\newcommand{\E}{\mathbb{E}}
\newcommand{\Prob}{\mathbb{P}}
\newcommand{\R}{\mathbb{R}}
\newcommand{\F}{\mathcal{F}}
\newcommand{\G}{\mathcal{G}}
\newcommand{\I}{\mathcal{I}}
\newcommand{\Reg}{\mathcal{R}}
\newcommand{\one}{\mathbf{1}}
\newcommand{\eps}{\varepsilon}
\newcommand{\norm}[1]{\left\lVert #1\right\rVert}
\newcommand{\abs}[1]{\left|#1\right|}

\newcommand{\clip}[1]{\operatorname{clip}_{[-B_u,B_u]}\!\left(#1\right)}
\newcommand{\lam}{\lambda}
\newcommand{\thetahat}{\widehat\theta}
\newcommand{\lambdahat}{\widehat\lambda}
\newcommand{\lambdabar}{\overline\lambda}
\newcommand{\bhat}{\widehat b}
\newcommand{\trans}{\mathsf T}

\hypersetup{
  pdftitle={Logarithmic Regret via Passive Change Detection in Piecewise-Stationary Self-Tuning Regulation},
  pdfauthor={A.Ch. Madhusudanarao and Rahul Singh},
  pdfsubject={Self-tuning regulation and passive change detection},
  pdfkeywords={self-tuning regulation, ARX systems, piecewise stationarity, change-point detection, regret}
}

\begin{document}

\title{Logarithmic Regret via Passive Change Detection in Piecewise-Stationary Self-Tuning Regulation}
\author{A.Ch. Madhusudanarao\\
\small Department of Computer Science and Automation\\
\small Indian Institute of Science, Bengaluru, India\\
\small \texttt{madhusudanar@iisc.ac.in}
\and
Rahul Singh\\
\small Laboratoire de Recherche de l'EPITA, Paris, France\\
\small \texttt{rahulsingh0188@gmail.com}}
\date{}
\maketitle

\begin{abstract}
We study minimum-variance control of an unknown autoregressive system with exogenous inputs and coefficients that change at unknown times. Under bounded independent disturbances, fixed detection gaps, stability and feasibility conditions, and sufficient time between changes, we prove \(O((C+1)\log((T+1)/\delta))\) regret with probability at least \(1-\delta\), where \(T\) is the horizon and \(C\) the number of changes. Unlike switching bandits, where unselected arms can change unobserved, admissible plant changes provide information during exploitation: the correct feasible controller leaves only the disturbance in the output, whereas a detectable change raises output energy under the old controller. PIECE-CD explores initially and after alarms, then uses gated recursive least squares for control. Its energy test compares windowed output power with a threshold above the noise floor; the extension to unstable controller mismatches also monitors the reference controller's input proposal. We control false alarms across the horizon and prove logarithmic detection delay. Inputs are clipped to prescribed bounds. Logarithmic regret also holds under an explicit condition ensuring that clipping becomes inactive after a finite burn-in. Under the stated feasibility conditions, the extended detector covers destabilizing changes with detectable excess energy over a fixed window.
\end{abstract}

\section{Introduction}
\label{introduction1}

Self-tuning regulation (STR) seeks minimum-variance control of an unknown
autoregressive system with exogenous inputs (ARX). The ideal controller
cancels the predictable part of the next output, leaving only the
disturbance. Classical work established stability and asymptotic
efficiency
\citep{goodwin1981discrete,becker1985adaptive,kumar1987self,lai1986asymptotically,lai1987asymptotically};
\citet{singhfinite} obtained finite-time logarithmic regret for a stationary
plant. We ask whether this rate survives unknown changes in the plant
coefficients. A controller must detect these changes and identify the new
plant, while avoiding repeated exploration caused by false alarms.

The key observation is that minimum-variance control can supply evidence of
a change during exploitation. With the correct feasible controller,
$y_{t+1}=w_{t+1}$, so output power equals the disturbance variance
$\sigma_w^2$. A change that raises output power under the old controller
can therefore be detected from ordinary control observations. This differs
from switching bandits, where an unselected arm can change unobserved and
even a two-breakpoint class has an $\Omega(\sqrt T)$ minimax lower bound
\citep{garivier2011upper}. The distinction is informational: our guarantee
requires a fixed detectable energy gap and sufficient time between
changes, and does not apply to arbitrary nonstationarity.

Our regret measures the predictable output energy left uncancelled by the
learner; its expectation equals the expected excess output energy above
the disturbance level
(\eqref{eq:a-regret}--\eqref{eq:a-regret-expectation}).

Our algorithm, PIECE-CD, extends Probing Inputs for Exploration in
Certainty Equivalence (PIECE). It explores initially and after an alarm,
then controls with gated recursive least squares (RLS) around an estimated
reference gain. Its principal version combines an output-energy test with
a boundary alarm on the reference controller's requested input. The latter
controls clipping errors when a changed plant is driven by the old gain.
Outputs remain dependent, and alarm times depend on the observed history;
both features matter for the guarantees.

\paragraph{Contributions.}
\begin{enumerate}[leftmargin=1.5em,itemsep=0.25em,topsep=0.3em]
\item \textbf{A common detector and logarithmic regret.}
Theorem~\ref{thm:a-stability-agnostic} proves no false alarms,
logarithmic detection delay, and
$\Reg_T=O((C+1)\log((T+1)/\delta))$ with probability at least $1-\delta$,
where $C$ is the number of changes. Under own-regime feasibility, the
theorem covers stable mismatches
with a stationary power gap and potentially unstable mismatches with a
finite-window energy gap, without identifying which condition applies.
For the compact minimum-phase ARX class, non-Schur mismatches imply a
uniform finite-window gap. The rate holds for fixed model, feasibility,
and detectability bounds under the stated dwell-time requirements.

\item \textbf{Identification and control after random restarts.}
We establish excitation and estimation bounds for complete exploration
blocks conditional on the history at each restart. An RLS potential bound
controls both accepted RLS actions and reference fallback actions.
Chronological failure-event bounds combine these results with detection
without assuming independent episodes or conditioning on future alarms.

\item \textbf{Supporting guarantees under stable mismatches.}
An energy-only baseline isolates the stationary-gap argument
(Theorem~\ref{thm:a-main}). A separate extension replaces uniform
clipping feasibility by a steady-state slack condition and a finite
burn-in, preserving logarithmic regret
(Theorem~\ref{thm:a-nonclipping}).
\end{enumerate}

\subsection{Related Work}

Our closest predecessor is stationary PIECE \citep{singhfinite}.
Its covariance analysis treats sparse, recurring exploration episodes.
Here the exploration phase is a contiguous block after a data-dependent
restart. This requires conditional excitation guarantees and a control
analysis that remains valid before the next change or alarm, together with
a detector for the dependent observations generated by feedback.

Detect-and-restart methods are well established in piecewise-stationary
bandits, including CUSUM-based CD-UCB \citep{liu2018change}, M-UCB
\citep{cao2019nearly}, GLR-klUCB \citep{besson2022efficient}, and the modular
DAB framework \citep{huang2026detection}. Their exploration schedules
address changes hidden by partial feedback. Additional structural
information can permit logarithmic gap-dependent bounds, as when all arms
change simultaneously \citep{mukherjee2019distribution}. PIECE-CD similarly
uses a restricted detectable class; its monitoring signal comes from
outputs observed during exploitation. Sequential detection provides the
statistical background
\citep{veeravalli2014quickest,xie2021sequential,tartakovsky2014sequential}.

Fully unknown online linear-quadratic regulation has square-root local
minimax regret \citep{simchowitz2020naive}. STR instead has no positive
input penalty, and its ideal controller leaves only the innovation in the
output. This property makes the noise floor an observable reference for
detecting the changes admitted here.

\section{MODEL AND SCOPE}
\label{sec:model1}

We consider a scalar ARX system of known orders $p,q\ge1$
whose coefficients change $C$ times over a horizon $T$.
Let
$0=\tau^{(0)}<\tau^{(1)}<\cdots<\tau^{(C)}
<\tau^{(C+1)}=T$
denote the deterministic segment boundaries, with the $C$ interior
boundaries unknown. On segment $c\in\{0,\ldots,C\}$,
\begin{align}
y_{t+1}
 &=\sum_{i=1}^{p}a_i^{(c)}y_{t-i+1}
   +\sum_{j=1}^{q}b_j^{(c)}u_{t-j+1}+w_{t+1},
   \label{eq:arx1}\\[-0.3em]
&\hspace{6em}\tau^{(c)}\le t<\tau^{(c+1)}. \notag
\end{align}
Here $y_t$ is the output, $u_t$ the control input, and $w_t$
the disturbance. The unknown parameter vector is constant
within each segment:
\begin{align}
\theta^{(c)}
 &=(a_1^{(c)},\ldots,a_p^{(c)},
    b_1^{(c)},\ldots,b_q^{(c)})^{\trans}. \notag
\end{align}
Write $\theta_t=\theta^{(c)}$ for $\tau^{(c)}\le t<\tau^{(c+1)}$.

We use Euclidean vector norms, induced spectral matrix norms, and
$\norm{x}_M^2=x^{\trans}Mx$.
We use two regressors:
\begin{align}
\psi_t
 &=(y_t,\ldots,y_{t-p+1},
    u_{t-1},\ldots,u_{t-q+1})^{\trans}, \notag\\
\phi_t
 &=(y_t,\ldots,y_{t-p+1},
    u_t,\ldots,u_{t-q+1})^{\trans}. \notag
\end{align}
The control regressor $\psi_t\in\R^d$, with $d=p+q-1$,
contains quantities available before $u_t$ is chosen.
The identification regressor $\phi_t\in\R^{p+q}$ additionally
includes $u_t$, so that
$y_{t+1}=(\theta^{(c)})^{\trans}\phi_t+w_{t+1}$.
The lagged-input block in $\psi_t$ is empty when $q=1$.

For a generic parameter vector
$\theta=(a_1,\ldots,a_p,b_1,\ldots,b_q)^{\trans}$
with $b_1\ne0$, define the minimum-variance gain
\begin{align}
\lambda(\theta)
 &:=-\frac1{b_1}
 (a_1,\ldots,a_p,b_2,\ldots,b_q)^{\trans}.
 \label{eq:gain1}
\end{align}
On segment $c$,
\[
y_{t+1}=b_1^{(c)}\{u_t-\lambda(\theta^{(c)})^{\trans}\psi_t\}+w_{t+1}.
\]
When feasible, the ideal action
$u_t^\star=\lambda(\theta^{(c)})^{\trans}\psi_t$ therefore leaves
$y_{t+1}=w_{t+1}$. We measure learning by the predictable output energy
\begin{align}
\Reg_T:=\sum_{t=0}^{T-1}(y_{t+1}-w_{t+1})^2.
\label{eq:a-regret}
\end{align}
Let $\mathcal F_t$ include the observations and randomization available
through the choice of $u_t$. Then $v_{t+1}:=y_{t+1}-w_{t+1}$ is
$\mathcal F_t$-measurable and
$\E[v_{t+1}w_{t+1}\mid\mathcal F_t]=0$. Hence
\begin{align}
\E\Reg_T=\E\sum_{t=0}^{T-1}(y_{t+1}^2-w_{t+1}^2).
\label{eq:a-regret-expectation}
\end{align}
This is an expectation identity; the realized excess output cost also
contains $2\sum_t v_{t+1}w_{t+1}$. Our high-probability bounds concern
$\Reg_T$.

We first state the conditions for the stable-mismatch baseline. The
extension in Section~\ref{sec:a-relaxations} replaces its mismatch
and feasibility conditions; the model, noise, and open-loop switching
conditions remain in force. Appendices~\ref{sec:model}--\ref{sec:algorithm}
give the corresponding constants and exploration conditions.
\begin{assumption}[Baseline conditions]
\label{ass:a-standing}
The controller knows a compact class
$\Theta\subset\R^{p+q}$ containing every $\theta^{(c)}$.
Each $\theta\in\Theta$ has a Schur-stable AR polynomial,
a minimum-phase input polynomial, and
$0<b_{\min}\le b_1\le\overline b$.
The autonomous output dynamics are uniformly exponentially
stable under admissible switching.
The disturbances are i.i.d., independent of the available
past, and satisfy
$\E w_t=0$, $\E w_t^2=\sigma_w^2>0$, and $|w_t|\le B_w$,
where $\sigma_w^2$ and $B_w$ are known.

Uniform clipping feasibility holds:
$|\lambda(\theta)^{\trans}\psi|\le B_u$
for every $\theta\in\Theta$ and every regressor $\psi$
reachable under admissible switching, bounded disturbances,
and inputs in $[-B_u,B_u]$.

For each change $c\in\{1,\ldots,C\}$, let
$\theta^-=\theta^{(c-1)}$ and $\theta^+=\theta^{(c)}$.
The unclipped loop obtained by applying the old gain
$\lambda(\theta^-)$ to the new plant $\theta^+$
has matrix $A_\Delta:=A(\theta^-,\theta^+)$, where
$A(\theta^-,\theta^+):=A_{\rm cl}(\lambda(\theta^-),\theta^+)$, satisfying
$\norm{A_\Delta^k}\le C_\Delta\rho_\Delta^k$ for all $k\ge0$,
with known common $C_\Delta<\infty$ and $\rho_\Delta\in(0,1)$.
Its stationary output power satisfies
\begin{align}
J_y(\lambda(\theta^-);\theta^+)
 \ge \sigma_w^2+\Delta
 \label{eq:a-gap}
\end{align}
for a known fixed $\Delta>0$. Here
$J_y=\sigma_w^2\sum_{k\ge0}(e_1^\trans A_\Delta^ke_1)^2$,
where $e_1$ is the first coordinate vector in $\R^d$.

For exploration length $H$ and delay bound $D$ defined below,
the dwell conditions, when $C\ge1$, are
\begin{align}
\tau^{(1)}&\ge H,\qquad T-\tau^{(C)}\ge D,\notag\\
\tau^{(c+1)}-\tau^{(c)}&\ge D+H\quad(1\le c<C).
\label{eq:a-dwell}
\end{align}
The final exploration may be truncated at $T$.
\end{assumption}

Stability and detectability serve different purposes. Stability controls
the discrepancy between the learner and the exact old-gain loop; the power
gap supplies the detection signal. A parameter change preserving the
minimum-variance gain can remain at the noise floor and is excluded by the
gap condition. The algorithm uses known model, noise, excitation, and
gap bounds for tuning, while the segment parameters and change times are
unknown. The alternative feasibility conditions are stated with their
respective results in Sections~\ref{sec:a-relaxations}
and~\ref{sec:a-nonclipping}.

Switching stability and clipped inputs give known deterministic bounds
$|y_t|\le B_Y$ and $\norm{\psi_t}\le B_\psi$; their explicit values are
given in \eqref{eq:Bpsi}.

\section{PIECE-CD}
\label{sec:a-algorithm}

We first specify the energy-only baseline; Section~\ref{sec:a-relaxations}
adds the boundary alarm for potentially unstable mismatches.
At initialization and after each alarm, PIECE-CD estimates a fixed
reference gain from $H$ exploration rounds. During exploitation, a gate
keeps RLS actions close to that reference while the detector monitors
output energy. Use $\delta_0=\delta/[32(T+1)]$ for estimation and RLS
and $\delta_d=\delta/2$ for detection.

\paragraph{Exploration and the reference controller.}
Let $r$ denote the start time of the current exploration
phase, with $r=0$ at initialization. After an alarm upon
observing $y_{t+1}$, the next exploration phase starts at
$r=t+1$. For $t=r,\ldots,r+H-1$, draw i.i.d.\
mean-zero inputs from a prescribed distribution with
positive variance and support in $[-B_u,B_u]$,
independently of the disturbances.
Let $r_0=\max\{p,q\}$, choose $H>r_0$, and retain the rows
$\mathcal I_r=\{r+r_0,\ldots,r+H-1\}$, whose lagged
measurements all follow the restart.

Fix a ridge parameter $\nu>0$. Write $I_k$ for the
$k\times k$ identity matrix and $\Pi_\Theta$ for a Euclidean
nearest-point projection onto $\Theta$. After completing
exploration, compute
\begin{align}
V_r^\theta
 &=\nu I_{p+q}
   +\sum_{t\in\mathcal I_r}\phi_t\phi_t^{\trans}, \notag\\
\widetilde\theta_r
 &=(V_r^\theta)^{-1}
   \sum_{t\in\mathcal I_r}\phi_ty_{t+1}, \notag\\
\widehat\theta_r
 &=\Pi_\Theta(\widetilde\theta_r),
 \qquad \widehat b_r=b_1(\widehat\theta_r).
 \label{eq:a-ridge}
\end{align}
The reference gain
$\overline\lambda_r=\lambda(\widehat\theta_r)$ and
$\widehat b_r$ remain fixed until the next restart.

The computable parameter and reference-gain error bounds
$E_\theta(H,\delta_0)$ and $E_\lambda(H,\delta_0)$ are
defined in \eqref{eq:Etheta}--\eqref{eq:gain-error-bound}.
For the prescribed choice of $H$ and a complete exploration block
within segment $c$, with probability at least $1-2\delta_0$ conditional on
the restart history,
\begin{align}
\norm{\widehat\theta_r-\theta^{(c)}}
 &\le E_\theta(H,\delta_0), \notag\\
E_\theta(H,\delta_0)
 &=O\!\left(\sqrt{\frac{\log((T+1)/\delta)}{H}}\right).
 \label{eq:a-est-rate}
\end{align}
This reference accuracy sets the gate tolerance:
$g_H=4E_\lambda(H,\delta_0)$.
Write $\eps_u=5E_\lambda(H,\delta_0)B_\psi$ for the resulting
reference-tracking tolerance used in the analysis.

\paragraph{Gated control and recursive estimation.}
Let $\widehat\lambda_t$ denote the current RLS gain estimate
and $V_t$ its regularized $d\times d$ Gram matrix.
Initialize both from the exploration data:
\[
V_{r+H}=\nu I_d+\sum_{t\in\mathcal I_r}\psi_t\psi_t^{\trans},
\qquad
\widehat\lambda_{r+H}=\overline\lambda_r.
\]
At each subsequent control round, form the input proposal
$z_t$ and clip it to the prescribed bounds:
\begin{align}
z_t&=\begin{cases}
\widehat\lambda_t^{\trans}\psi_t,
&|(\widehat\lambda_t-\overline\lambda_r)^{\trans}\psi_t|
 \le g_H\norm{\psi_t},\\
\overline\lambda_r^{\trans}\psi_t,&\text{otherwise},
\end{cases} \notag\\
u_t&=\min\{B_u,\max\{-B_u,z_t\}\}.
\label{eq:a-gate}
\end{align}
The gate limits departure from the reference action used for detection;
clipping enforces the prescribed input bound.

After observing $y_{t+1}$, construct the response
$x_{t+1}^\lambda$ and update RLS:
\begin{align}
x_{t+1}^\lambda
 &=u_t-\frac{y_{t+1}}{\widehat b_r},
&V_{t+1}
 &=V_t+\psi_t\psi_t^{\trans}, \notag\\
\widehat\lambda_{t+1}
 &=\widehat\lambda_t+V_{t+1}^{-1}\psi_t\cdot(x_{t+1}^\lambda-\widehat\lambda_t^{\trans}\psi_t).
\label{eq:a-rls}
\end{align}
RLS uses every exploitation observation, including rounds with a
reference fallback or an active clip. The analysis accounts for bias
from estimating $b_1^{(c)}$.

\paragraph{Monitoring and restarting.}
Let the positive integers $h$ and $m$ denote the sampling
interval and detection-window size, respectively, chosen
below. After exploration, fix the recording times
$s_j=r+H+jh$, $j=1,2,\ldots$.
At each $s_j$, store $y_{s_j}^2$ and retain the most recent
$m$ values. Once $j\ge m$, raise an alarm if
\begin{align}
\widehat J(s_j)
 :=\frac1m\sum_{i=j-m+1}^{j}y_{s_i}^2
 \ge \sigma_w^2+\frac{\Delta}{2}.
\label{eq:a-detector}
\end{align}
Control and RLS updates continue every round; only the energy test is
subsampled.
An alarm at $s_j$ clears the estimates and stored values,
sets $r=s_j$, and starts a new exploration phase.

\paragraph{Choosing the exploration and detection lengths.}
Choose the smallest integer $H>r_0$ satisfying
\eqref{eq:block-excitation-size}, \eqref{eq:H-rls}, and
\eqref{eq:H-detector} in Appendix~\ref{sec:algorithm}.
These conditions ensure informative exploration and
sufficiently accurate reference estimates for control
and detection.
Choose the smallest $h\ge1$ satisfying
\[
\frac{\sigma_w^2C_\Delta^2\rho_\Delta^{2h}}{1-\rho_\Delta^2}
\le\frac{\Delta}{8}.
\]
Recorded outputs remain dependent; their spacing permits concentration
around lag-$h$ conditional second moments.

Let $B_o$ be the known common bound in
\eqref{eq:detector-bounds} on disturbances, actual outputs,
and outputs of the new plant driven by the exact old gain
from the state at the change.
Set the window size $m$ and associated detection-delay
bound $D$ to
\begin{align}
m&=\left\lceil
 \frac{512B_o^4}{\Delta^2}
 \log\frac{4T^2}{\delta_d}\right\rceil,
&D&=(m+1)h.
\label{eq:a-tuning}
\end{align}
This window size controls fluctuations across the horizon's
repeated tests. For fixed problem constants, $h=O(1)$ and
$H,m,D=O(\log((T+1)/\delta))$.

\begin{algorithm}[H]
\caption{PIECE-CD: stable-mismatch baseline}
\label{alg:a-piece-cd}
\small
\begin{enumerate}[leftmargin=1.25em,itemsep=0.15em,topsep=0.2em]
\item Initialize $r=0$. Stop at $T$, including during an incomplete phase.
\item Explore for $H$ rounds; compute \eqref{eq:a-ridge}, the reference
gain, and the RLS initialization. Clear the detector buffer and fix
the sampling times $s_j=r+H+jh$.
\item At every exploitation round, apply \eqref{eq:a-gate}, observe
$y_{t+1}$, and update \eqref{eq:a-rls}. At sampling times, update the
buffer and test \eqref{eq:a-detector} once it is full.
\item If an alarm occurs on observing $y_{t+1}$, set $r=t+1$ and return
to step 2; otherwise continue step 3.
\end{enumerate}
\end{algorithm}

\section{FINITE-TIME ANALYSIS}
\label{sec:a-analysis}

The proof bounds detection delays and false alarms, then controls
prediction error on each unchanged exploitation interval. These bounds
yield the stable-mismatch baseline in Theorem~\ref{thm:a-main}.
Section~\ref{sec:a-relaxations} extends the detector analysis to
potentially unstable mismatches under alternative input feasibility.

\paragraph{Detection guarantee.}
Let $\mathcal E_{\rm ref}$ be the event that every reached, complete,
unchanged exploration block whose restart precedes the first detector
failure produces an accurate reference gain. Here a detector failure is
a false alarm or a change not detected by its $D$-round deadline; the
precise reference-event definition is given in
\eqref{eq:reference-event}.
Let $\mathcal E_{\rm det}$ be the event of no false alarms and exactly
one alarm per change. Writing $\widehat\tau^{(c)}$ for the output time
of the alarm corresponding to change $c$, this event also requires
\begin{align}
\tau^{(c)}<\widehat\tau^{(c)}\le\tau^{(c)}+D,
\qquad c=1,\ldots,C.
\label{eq:a-delay}
\end{align}

\begin{theorem}[Alarm correspondence]
\label{thm:a-detection}
Under the complete assumptions in Appendices~\ref{sec:model} and
\ref{sec:algorithm} and the tuning above,
\begin{align}
\Prob(\mathcal E_{\rm ref}\cap\mathcal E_{\rm det}^{c})
\le\delta_d.
\label{eq:a-gated-detection}
\end{align}
\end{theorem}

\begin{proof}
See Theorem~\ref{thm:detection} and Appendix~\ref{app:detection}.
\end{proof}

This joint bound controls detector failure on accurate-reference
histories; reference-estimation failures enter
Theorem~\ref{thm:a-main} separately.

\paragraph{Why the detector works.}
Let $\mathcal F_t$ denote the information available when $u_t$ is
selected, including the controller's randomization and the selected
input. Suppose the current reference is accurate and the plant has not
changed since its exploration. The gate and input feasibility then give,
at each recorded output time $s$,
\begin{align}
y_s=w_s+v_s,\qquad |v_s|\le\varepsilon_y,
\label{eq:a-null-split}
\end{align}
where $v_s:=y_s-w_s$ is $\mathcal F_{s-1}$-measurable and
$\varepsilon_y$ is the deterministic output-tracking bound in
\eqref{eq:tracking-errors}. Thus $w_s$ is fresh noise relative to $v_s$.
The tuning makes $v_s^2$ small enough that concentration of $w_s^2$
and $2w_sv_s$ keeps unchanged windows below the threshold with high
probability.

After a change from $\theta^-$ to $\theta^+$ at control time $\tau$,
we compare the learner with an auxiliary process that applies the exact
old gain $\lambda(\theta^-)$ to the new plant. It starts from the
learner's regressor $\psi_\tau$ and uses the same future disturbances:
\begin{align*}
\psi_\tau^{\rm aux}&=\psi_\tau,\\
\psi_{t+1}^{\rm aux}
 &=A_\Delta\psi_t^{\rm aux}+e_1w_{t+1},
 \qquad t\ge\tau.
\end{align*}
Here $A_\Delta$ is the closed-loop matrix of the new plant under the old
gain, $e_1\in\R^d$ is the first coordinate vector, and
$y_s^{\rm aux}=e_1^{\trans}\psi_s^{\rm aux}$ is the auxiliary output;
see \eqref{eq:auxiliary}.
For a fixed calendar time $s$ with $s-h\ge\tau$, the fresh disturbances
over the preceding $h$ steps give
\begin{align}
\E[(y_s^{\rm aux})^2\mid\mathcal F_{s-h}]
&=(e_1^{\trans}A_\Delta^h\psi_{s-h}^{\rm aux})^2
 +\sigma_w^2\sum_{j=0}^{h-1}
       (e_1^{\trans}A_\Delta^je_1)^2
  \notag\\
&\ge\sigma_w^2+\frac{7\Delta}{8}.
\label{eq:a-finite-power}
\end{align}
The initial-state term is nonnegative, and the choice of $h$ loses at
most $\Delta/8$ from the stationary power gap.
On an accurate-reference history, the gate and stability also ensure
$|y_s^2-(y_s^{\rm aux})^2|\le\Delta/8$ until the next alarm or plant
change. With high probability, the average of $m$ fully post-change
auxiliary output squares lies at most $\Delta/8$ below its average
conditional power. Unless an alarm occurs earlier, the actual average
is therefore at least $\sigma_w^2+5\Delta/8$, above the threshold.
The sampling calendar supplies such a window by $\tau+D$; the
dwell-time condition prevents another change before this deadline.

Restart times depend on past observations. For each potential restart,
the proof uses a continuation with the same control and recording rules
but suppresses subsequent restarts. It agrees with the implemented
trajectory through the first alarm. Joint failure bounds over reached
histories give the global guarantee without conditioning on future
detector success; see Appendices~\ref{app:detection} and~\ref{app:global}.

With restarts and detection delays controlled, we next bound the RLS
regret accumulated during unchanged exploitation.

\begin{theorem}[Unchanged exploitation prefixes]
\label{thm:a-segment}
Suppose a full exploration block $[r,r+H)$ lies in one regime and its
Gram and estimation events hold. Under the baseline conditions and
tuning, conditional on the history at $r+H$, with probability at least
$1-\delta_0$, simultaneously for every exploitation prefix ending no
later than the next plant change, restart, or horizon,
\begin{align}
\sum_{t=r+H}^{r+H+n-1}(y_{t+1}-w_{t+1})^2
\le R_E(n,\delta_0).
\label{eq:a-segment}
\end{align}
Here $R_E$ is the explicit bound in \eqref{eq:segment-regret}.
For fixed model, noise, and exploration constants and the prescribed
choice of $H$, it satisfies
$R_E(T,\delta_0)=O(\log((T+1)/\delta))$.
\end{theorem}
Theorem~\ref{thm:segment-regret} and Appendix~\ref{app:rls} give the
conditional, time-uniform bound. Its main mechanism is as follows.

\paragraph{RLS argument.}
Consider an unchanged exploitation interval in segment $c$ after successful
exploration, and write $\lambda=\lambda(\theta^{(c)})$ and
$b_\star=b_1^{(c)}$. Define the RLS proposal error and the applied action
error by
\[
a_t=(\widehat\lambda_t-\lambda)^{\trans}\psi_t,
\qquad
e_t=u_t-\lambda^{\trans}\psi_t.
\]
Reference accuracy, the gate, and feasibility of the ideal action imply
$e_t^2\le a_t^2$. Thus it suffices to control the cumulative RLS proposal
error, even on rounds that use the reference input or clipping.

The response used by RLS satisfies
\begin{align*}
x_{t+1}^{\lambda}
 &=\lambda^{\trans}\psi_t+\alpha e_t+\eta_{t+1},\\
\alpha&:=1-\frac{b_\star}{\widehat b_r},\qquad
\eta_{t+1}:=-\frac{w_{t+1}}{\widehat b_r}.
\end{align*}
Here $\alpha e_t$ is the predictable bias caused by estimating $b_\star$,
whereas $\eta_{t+1}$ is bounded and conditionally mean zero.
To track estimation progress, define
\begin{align*}
Q_t
 &=(\widehat\lambda_t-\lambda)^{\trans}
   V_t(\widehat\lambda_t-\lambda),\\
\ell_t
 &=\psi_t^{\trans}V_{t+1}^{-1}\psi_t.
\end{align*}
Exploration controls both $|\alpha|$ and $\ell_t$. Consequently, the exact
RLS identity gives
\begin{align}
Q_{t+1}-Q_t
 \le -c_0a_t^2+2a_t\eta_{t+1}
      +2\ell_t\eta_{t+1}^2,
\label{eq:a-potential}
\end{align}
with $c_0>0$ defined in \eqref{eq:regret-constants}.

With probability at least $1-\delta_0$, a time-uniform martingale bound
controls the cumulative mixed term by a fraction of
$c_0\sum_t a_t^2$ plus a confidence term. Bounded noise and the
log-determinant bound on $\sum_t\ell_t$ control the final term
logarithmically. Summing \eqref{eq:a-potential}, using $Q_t\ge0$
and $Q_{r+H}\le \overline Q_H$, therefore bounds $\sum_ta_t^2$ simultaneously
over every exploitation prefix. Finally,
\[
(y_{t+1}-w_{t+1})^2
 =b_\star^2e_t^2\le\overline b^2a_t^2,
\]
which yields Theorem~\ref{thm:a-segment}.
Appendix~\ref{app:rls} gives the identities and concentration details.

\begin{theorem}[Stable-mismatch regret guarantee]
\label{thm:a-main}
Under Assumption~\ref{ass:a-standing}, the exploration setup in
Section~\ref{sec:a-algorithm}, and the precise model and exploration
bounds in Appendices~\ref{sec:model}--\ref{sec:algorithm}, choose the
smallest feasible $H$ and the detector tuning \eqref{eq:a-tuning}.  Then, with probability at least
$1-\delta$, there are no false alarms, every change has exactly one alarm
satisfying \eqref{eq:a-delay}, and
\begin{align}
\Reg_T
 \le CK_DD+(C+1)K_XH+(C+1)R_E(T,\delta_0),          \label{eq:a-main-bound}
\end{align}
where the fixed per-step delay and exploration bounds $K_D,K_X$ are defined in
\eqref{eq:delay-cost}--\eqref{eq:exploration-cost}. Consequently, when all problem constants (model orders, model-class and
noise bounds, initial state, exploration and ridge settings, stability
constants, clipping threshold, and power gap) are held fixed as $T$ grows,
\begin{align}
\Reg_T=O\!\left((C+1)\log\frac{T+1}{\delta}\right). \label{eq:a-main-rate}
\end{align}
\end{theorem}

\begin{proof}
See Theorem~\ref{thm:main} and Appendices~\ref{app:global}--\ref{app:tuning}.
\end{proof}

\paragraph{Proof sketch of Theorem~\ref{thm:a-main}.}
The regret splits into three parts: each change costs at most $K_DD$ before
it is detected (per-step cost $K_D$ times delay $D$), each of the $C+1$
exploration blocks costs at most $K_XH$, and
Theorem~\ref{thm:a-segment} bounds the regret on the remaining
unchanged exploitation periods.  For the probability, each restart can fail
in three ways (insufficient excitation, inaccurate estimate, or RLS
failure), each with probability at most $\delta_0$, and at most $T+1$
restarts are possible.  Adding the detector failure probability $\delta_d$
gives a total of $3(T+1)\delta_0+\delta_d<\delta$.  Because restart times
are random, failures are counted in time order, each only if its restart is
actually reached; Appendix~\ref{app:global} gives the details.

\section{DETECTION UNDER STABLE AND UNSTABLE MISMATCHES}
\label{sec:a-relaxations}

The baseline theorem in Section~\ref{sec:a-analysis} assumes that
the new plant remains stable under the old controller. We now allow this
mismatched closed loop to be unstable. Its stationary power may then be
unavailable, and clipping can alter its output-energy signal. We address
these issues by converting both detectability conditions below into a
common finite-window gap and combining the energy test with an alarm on
the reference gain's input proposal. Fixed-window comparisons replace the
comparison over the full detection delay used in the stable analysis.

\paragraph{Two sources of detectable energy.}
Let $\mathcal P\subseteq\Theta^2$ be the prescribed class of admissible
adjacent parameter pairs. For an actual change
$z_c=(\theta^{(c-1)},\theta^{(c)})$, $c=1,\ldots,C$, write
\[
A_c:=A_{z_c}:=A(\theta^{(c-1)},\theta^{(c)}),
\]
the closed-loop matrix of the new plant under the old gain.
We cover $\mathcal P$ by two possibly overlapping subfamilies.
In the stable subfamily $\mathcal P_{\sf S}$, known common constants
$C_{\sf S}<\infty$, $\rho_{\sf S}\in(0,1)$, and
$\underline\Delta_{\sf S}>0$ satisfy
\begin{align}
\norm{A_c^k}&\le C_{\sf S}\rho_{\sf S}^k
\quad(k\ge0),\notag\\
J_y(\lambda(\theta^{(c-1)});\theta^{(c)})
&\ge\sigma_w^2+\underline\Delta_{\sf S}.
\label{eq:a-stable-class}
\end{align}
In the finite-visible subfamily $\mathcal P_{\sf F}$, a known common
$\underline\Delta_{\sf F}>0$ instead satisfies
\begin{align}
\mathcal V(z_c)
:=\sigma_w^2\sum_{j=1}^{d}(e_1^\trans A_c^je_1)^2
\ge\underline\Delta_{\sf F},
\label{eq:a-visible-class}
\end{align}
where $e_1\in\R^d$ is the first coordinate vector.
The coefficients $e_1^\trans A_c^je_1$ describe a disturbance's effect on
later outputs, so $\mathcal V(z_c)$ measures its excess energy over the first
$d$ lags. This condition permits unstable mismatches.
Every actual pair belongs to
$\mathcal P\subseteq\mathcal P_{\sf S}\cup\mathcal P_{\sf F}$; the algorithm
uses the class-wide bounds for tuning but never determines which subfamily
covers a particular change.

\paragraph{A common finite-window gap.}
For an integer $h\ge1$ and a $d\times d$ matrix $A$, define
\[
G_h(A):=\sigma_w^2\sum_{j=1}^{h-1}(e_1^\trans A^je_1)^2.
\]
It measures the excess variance supplied by the preceding $h-1$
disturbances. For stable pairs, the omitted tail is bounded by
$r_{\sf S}(h)$ below. When both subfamilies are nonempty, choose
\begin{align}
r_{\sf S}(h)&:=
\frac{\sigma_w^2C_{\sf S}^2\rho_{\sf S}^{2h}}
{1-\rho_{\sf S}^2},\notag\\[-0.2em]
h_{\sf S}&:=\min\{h\ge1:r_{\sf S}(h)\le
\underline\Delta_{\sf S}/8\},\notag\\[-0.2em]
h_{\sf U}&:=\max\{d+1,h_{\sf S}\},\qquad
\underline\Delta_{\sf U}:=
\min\{7\underline\Delta_{\sf S}/8,\underline\Delta_{\sf F}\}.
\label{eq:a-union-gap}
\end{align}
Appendix~\ref{app:relaxations} gives the convention when a subfamily is
empty. Since $h_{\sf U}\ge h_{\sf S}$ bounds the stable tail and
$h_{\sf U}\ge d+1$ includes every term in
\eqref{eq:a-visible-class},
\begin{align}
G_{h_{\sf U}}(A_c)
\ge\begin{cases}
7\underline\Delta_{\sf S}/8,&z_c\in\mathcal P_{\sf S},\\
\underline\Delta_{\sf F},&z_c\in\mathcal P_{\sf F},
\end{cases}
\ge\underline\Delta_{\sf U}.
\label{eq:a-union-energy}
\end{align}
Stable pairs therefore need not satisfy a separate first-$d$ visibility
condition: their stationary gap supplies the required finite-window gap.

\paragraph{The detector and its tuning.}
Algorithm~\ref{alg:stability-agnostic} retains PIECE-CD's exploration,
estimation, gated RLS updates, and restart rule. Its energy test records
one output square every $h_{\sf U}$ rounds and raises an alarm when the
average of the latest $m_{\sf U}$ recorded squares reaches
$\sigma_w^2+\underline\Delta_{\sf U}/2$.
Let $B_{\sf U}$ be the common bound on disturbances, actual outputs, and
auxiliary old-gain outputs over an $h_{\sf U}$-step comparison, defined in
\eqref{eq:finite-lag-bounds}. The auxiliary bound is finite even for an
unstable mismatch because each comparison lasts only $h_{\sf U}$ steps.
Set the number of recorded outputs and the detection-delay bound to
\begin{align}
m_{\sf U}&:=\left\lceil
\frac{512B_{\sf U}^4}{\underline\Delta_{\sf U}^2}
\log\frac{4T^2}{\delta_d}\right\rceil,
&D_{\sf U}&:=(m_{\sf U}+1)h_{\sf U}.
\label{eq:a-union-tuning}
\end{align}

To handle clipping, replace uniform feasibility by closed
\emph{own-regime feasibility} (Assumption~\ref{ass:own-regime-feasibility}).
Each ideal gain must respect the input bound on every regressor reached
after at least $r_0$ controls under its own plant, from any initial
regressor of norm at most $B_\psi$, with bounded disturbances and arbitrary
inputs in $[-B_u,B_u]$. Equality is allowed. This is an alternative
condition: it need not follow from baseline feasibility because it also
allows initially unreachable regressors. The old gain need not remain
feasible after a change. The additional test monitors the
reference gain's input proposal.
After observing $y_{t+1}$ and forming $\psi_{t+1}$, it raises an alarm if
\begin{align}
|\overline\lambda_r^\trans\psi_{t+1}|
>B_u+\zeta_H,\qquad
\zeta_H:=E_\lambda(H,\delta_0)B_\psi.
\label{eq:a-boundary-alarm}
\end{align}
With an accurate reference, own-regime feasibility prevents this alarm
while the plant is unchanged. After a change, before this alarm fires,
the old ideal input exceeds the allowed interval by at most $2\zeta_H$,
and the applied input differs from it by at most $7\zeta_H$.
These bounds control the effect of clipping in the finite-step comparison.

Choose $H$ to satisfy the exploration and RLS conditions together with
\eqref{eq:finite-lag-H}, which controls finite-window tracking errors.
Smaller gaps require longer exploration and detection windows, increasing
the required dwell time and regret constant. Either test triggers a
restart, with simultaneous crossings counted as one alarm.
The following theorem gives the detection and logarithmic-regret guarantees.

\begin{theorem}[Detection and regret for the union class]
\label{thm:a-stability-agnostic}
Retain the compact-model, noise, and open-loop switching-stability
conditions of Assumption~\ref{ass:a-standing}, and impose the
own-regime feasibility above. Under the precise union-class conditions
in Assumption~\ref{ass:finite-lag-detectability},
suppose each actual adjacent pair satisfies at least one of
\eqref{eq:a-stable-class} and \eqref{eq:a-visible-class}, without
revealing which one.  Choose $H$ as the smallest feasible integer satisfying
\eqref{eq:block-excitation-size}, \eqref{eq:H-rls}, and
\eqref{eq:finite-lag-H}; choose
$m_{\sf U},D_{\sf U}$ by \eqref{eq:a-union-tuning}, and impose
\eqref{eq:a-dwell} with $D_{\sf U}$ in place of $D$.  Then Algorithm~\ref{alg:stability-agnostic}, with
probability at least $1-\delta$, has no false alarms, and each change has
exactly one alarm $\widehat\tau^{(c)}$ satisfying
\begin{align}
 \tau^{(c)}<\widehat\tau^{(c)}
 \le \tau^{(c)}+D_{\sf U},                          \label{eq:a-union-delay}
\end{align}
and
\begin{align}
 \Reg_T
 \le CK_DD_{\sf U}+(C+1)K_XH+(C+1)R_E(T,\delta_0). \label{eq:a-union-regret}
\end{align}
For fixed model and union-class bounds, including every applicable positive
gap bound for the nonempty cover components, this is
$O((C+1)\log((T+1)/\delta))$.  The algorithm never determines which
subfamily covers any particular change.
\end{theorem}

\begin{proof}
See Theorem~\ref{thm:stability-agnostic} in
Appendix~\ref{app:unified-detector}.
\end{proof}

\paragraph{Detection argument for Theorem~\ref{thm:a-stability-agnostic}.}
The proof starts a separate $h_{\sf U}$-step auxiliary process before each
stored output.  If the boundary test has not already crossed, the actual and
auxiliary output squares remain close over these finitely many steps.  For a
pair in $\mathcal P_{\sf S}$, its stationary gap and stable-tail bound give
conditional auxiliary energy at least
$\sigma_w^2+7\underline\Delta_{\sf S}/8$; for a pair in
$\mathcal P_{\sf F}$, finite visibility gives at least
$\sigma_w^2+\underline\Delta_{\sf F}$.  Equation~\eqref{eq:a-union-energy}
therefore supplies the same $\underline\Delta_{\sf U}$ to the concentration
argument in both cases.  

\paragraph{Finite-order visibility of unstable ARX mismatches.}
Spectral instability alone does not imply visibility
for arbitrary matrices.  For this compact minimum-phase ARX class,
\eqref{eq:spectral-implies-visibility} derives a uniform finite-order gap
from $\rho(A(\theta^-,\theta^+))\ge1$, so no separate visibility assumption is needed on
that subclass.

\paragraph{Why input feasibility is needed.}
Some restriction on saturation is unavoidable for the regret benchmark in
\eqref{eq:a-regret}.  For every clipped policy,
\begin{align}
 \Reg_T\ge b_{\min}^2\sum_{t=0}^{T-1}
 (|\lambda(\theta_t)^\trans\psi_t|-B_u)_+^2.
 \label{eq:a-saturation-lower}
\end{align}
Appendix~\ref{app:relaxations} also gives an exploitation bound with this
saturation energy as an additive term.  Persistent saturation can therefore
force linear regret, even when the parameter is known.

\section{EVENTUAL CLIPPING INACTIVITY}
\label{sec:a-nonclipping}

For stable detectable mismatches, a separate variant replaces uniform
clipping feasibility by a steady-state slack condition. Define
\[
P_\infty(B_u):=\sqrt{\left(\frac{C_y}{1-\rho_y}
(B_w+B_bB_u)\right)^2+(q-1)B_u^2},
\]
where $C_y,\rho_y$ are the uniform open-loop switching bounds and
$B_b$ bounds the sum of absolute input coefficients. Let
$\Lambda=\sup_{\theta\in\Theta}\norm{\lambda(\theta)}$.
The quantity $P_\infty$ bounds the forced-response contribution to the
regressor norm; the slack absorbs the decaying initial-state term.

\begin{theorem}[Regret under eventual clipping inactivity]
\label{thm:a-nonclipping}
Retain the baseline model, noise, switching-stability, stable power-gap,
and exploration conditions, but replace uniform feasibility by
\begin{align}
\mu_0:=B_u-\Lambda P_\infty(B_u)>0.
\label{eq:nc-slack1}
\end{align}
Choose the smallest baseline-feasible $H$ also satisfying
$\eps_u\le\mu_0/2$. After each full exploration, use the clipped reference
controller for $\tau_{\rm nc}$ rounds, with RLS and detection disabled,
where $\tau_{\rm nc}$ is given by \eqref{eq:tau-nc} with
$\mu_{\rm nc}=\mu_0-\eps_u$. At $t_0=r+H+\tau_{\rm nc}$, initialize
RLS from the same exploration data as in
Section~\ref{sec:a-algorithm}, and record outputs at $t_0+jh$, $j\ge1$.
Use \eqref{eq:a-dwell} with $H+\tau_{\rm nc}$ in place of $H$.
With probability at least $1-\delta$, there are no false alarms,
every change is detected within $D$ rounds, and
\begin{align}
\Reg_T\le{}& CK_DD+(C+1)K_XH+(C+1)\{K_D\tau_{\rm nc}+R_E(T,\delta_0)\}.
\label{eq:nc-regret1}
\end{align}
For fixed problem constants and positive slack, $\tau_{\rm nc}=O(1)$,
so the regret remains $O((C+1)\log((T+1)/\delta))$.
\end{theorem}
Proposition~\ref{prop:eventual-nonclipping} and the burn-in argument in
Appendix~\ref{sec:eventual-nonclipping} establish feasibility after the
burn-in, including the continuations needed for detection. Thus only the
additional bounded cost per restart enters the regret decomposition.

\section{LIMITATIONS}
\label{sec:positioning1}

The logarithmic horizon dependence holds for fixed change count,
confidence level, and problem constants. The bounds retain the factor
$C+1$ when changes become more frequent. A uniform energy gap and
sufficient dwell time are essential: hidden or arbitrarily weak changes
are outside the guarantees. The algorithm also requires supplied bounds
on noise, detectability, excitation, and the relevant stability and
finite-window sensitivities. Smaller gaps increase the delay and required
dwell time. Estimating these bounds, removing knowledge of the horizon,
and accommodating heavy-tailed disturbances remain open extensions.
The feasibility conditions restrict the nonlinear effects of saturation;
they do not cover arbitrary clipped dynamics.

\section{CONCLUSION}
We establish high-probability regret
$O((C+1)\log((T+1)/\delta))$ for minimum-variance control of unknown
piecewise-stationary ARX systems under fixed detectability gaps and the
stated feasibility and dwell-time conditions.
PIECE-CD uses output observations during exploitation to detect changes,
explores at initialization and after each alarm, and uses gated RLS to
control regret between changes.
A common detection rule covers both stable controller mismatches with a
stationary power gap and potentially unstable mismatches whose disturbance
response produces detectable excess energy within a fixed number of steps.

\clearpage
\section*{AI Use Statement}
Generative AI tools assisted with editing and reorganizing the exposition,
checking consistency between main-text statements and the appendices, and LaTeX
formatting and compilation checks. The authors take responsibility for
the final text, mathematical claims, proofs, and references.

\clearpage


\clearpage
\appendix
\onecolumn

\section*{APPENDIX CONTENTS}
\begingroup
\newcommand{\appendixcontentsentry}[1]{%
  \noindent\hyperref[#1]{\makebox[1.6em][l]{\ref*{#1}}\nameref*{#1}}%
  \dotfill\hyperref[#1]{\pageref*{#1}}\par
}
\appendixcontentsentry{sec:model}
\appendixcontentsentry{sec:algorithm}
\appendixcontentsentry{sec:detection}
\appendixcontentsentry{sec:regret}
\appendixcontentsentry{app:bounds}
\appendixcontentsentry{app:exploration}
\appendixcontentsentry{app:rls}
\appendixcontentsentry{app:detection}
\appendixcontentsentry{app:global}
\appendixcontentsentry{app:tuning}
\appendixcontentsentry{app:relaxations}
\appendixcontentsentry{app:unified-detector}
\endgroup

\section*{ROADMAP OF THE APPENDICES}
\begin{description}[
  leftmargin=2em,
  labelsep=0.5em,
  itemsep=0.45em,
  topsep=0pt,
  parsep=0pt
]
\item[\ref{sec:model}.]
Specifies the piecewise-stationary ARX model, regret criterion, and
assumptions on stability, noise, clipping feasibility, and detectability
of adjacent changes.

\item[\ref{sec:algorithm}.]
Gives the exploration estimator, gated RLS controller, output-energy
detector, tuning and dwell-time requirements, and complete PIECE-CD
algorithm.

\item[\ref{sec:detection}.]
Proves action tracking, sampled-square concentration, false-alarm control,
detection delay, and correspondence between alarms and true changes.

\item[\ref{sec:regret}.]
Derives exploitation regret on unchanged segments and combines it with
exploration and detection-delay costs to obtain the total regret guarantee.

\item[\ref{app:bounds}.]
Establishes deterministic output and regressor bounds, clipping projection
inequalities, auxiliary-process bounds, and conditioning conventions at
random restarts.

\item[\ref{app:exploration}.]
Proves conditional excitation, exploration Gram-matrix growth, and projected
ridge estimation bounds for a complete unchanged exploration block.

\item[\ref{app:rls}.]
Develops the RLS potential identity, time-uniform noise control, and
elliptical-potential bound underlying the segmentwise exploitation guarantee.

\item[\ref{app:detection}.]
Supplies detailed auxiliary-process power, sampling, coupling, and calendar
arguments supporting the detection results in Appendix~\ref{sec:detection}.

\item[\ref{app:global}.]
Constructs a common success event and partitions the horizon to complete
the proof of the global regret bound.

\item[\ref{app:tuning}.]
Verifies feasibility of the tuning requirements and derives the logarithmic
orders of exploration length, detection delay, and regret.

\item[\ref{app:relaxations}.]
Develops the clipping-feasibility relaxations, eventual nonclipping,
boundary alarms, finite-order visibility and union-class detection,
and regret bounds with saturation.

\item[\ref{app:unified-detector}.]
Presents the unified union-class algorithm and common tuning, then proves
detection and regret guarantees without identifying the realized mismatch
class.
\end{description}
\clearpage

\section{MODEL AND ASSUMPTIONS}
\label{sec:model}

Fix an integer horizon $T\ge1$, confidence level $\delta\in(0,1)$, and
integer model orders $p,q\ge1$. The deterministic but unknown segment
boundaries are
\[
0=\tau^{(0)}<\tau^{(1)}<\cdots<\tau^{(C)}<
\tau^{(C+1)}=T.
\]
On segment $c$, the parameter
\begin{align}
\theta^{(c)}
 &:=(a_1^{(c)},\ldots,a_p^{(c)},b_1^{(c)},\ldots,b_q^{(c)})'
 \notag
\end{align}
is active at control times $\tau^{(c)}\le t<\tau^{(c+1)}$, and the system obeys
\begin{align}
y_{t+1}
 &=\sum_{i=1}^{p}a_i^{(c)}y_{t-i+1}
   +\sum_{j=1}^{q}b_j^{(c)}u_{t-j+1}+w_{t+1},
   \qquad \tau^{(c)}\le t<\tau^{(c+1)}.
\label{eq:arx}
\end{align}
Write $\theta_t=\theta^{(c)}$ on this segment and
$Y_t=(y_t,\ldots,y_{t-p+1})'$. The initial history is fixed, with
$\norm{Y_0}<\infty$ and $|u_t|\le B_u$ for $t=-q+1,\ldots,-1$.
The control and identification regressors are
\begin{align}
\psi_t
 &:=(y_t,\ldots,y_{t-p+1},u_{t-1},\ldots,u_{t-q+1})',
\label{eq:psi}\\
\phi_t
 &:=(y_t,\ldots,y_{t-p+1},u_t,u_{t-1},\ldots,u_{t-q+1})',
\qquad d:=p+q-1.
\label{eq:phi}
\end{align}
For $\theta=(a_1,\ldots,a_p,b_1,\ldots,b_q)$ with $b_1\ne0$, define
\begin{align}
\lam(\theta):=-\frac{1}{b_1}
(a_1,\ldots,a_p,b_2,\ldots,b_q)'.
\label{eq:gain}
\end{align}
The corresponding gain representation is
\begin{align}
y_{t+1}
=b_1^{(c)}\{u_t-\lam(\theta^{(c)})'\psi_t\}+w_{t+1}.
\label{eq:gain-form}
\end{align}
The ideal input $u_t=\lam(\theta^{(c)})'\psi_t$ leaves only the disturbance
in the output. Our sample-path learning regret is
\begin{align}
\Reg_T:=\sum_{t=0}^{T-1}(y_{t+1}-w_{t+1})^2.
\notag
\end{align}

Let $\F_t$ contain the initial history, the controller randomization revealed
through time $t$, and all inputs and outputs available when $u_t$ is selected.
In particular, $u_t$ is $\F_t$-measurable.

\begin{assumption}[Known model class]
\label{ass:model}
The controller knows a compact set $\Theta\subset\R^{p+q}$ containing every
$\theta^{(c)}$. For each $\theta\in\Theta$, the autoregressive polynomial is
Schur stable, the input polynomial is minimum phase, and
\begin{align}
0<b_{\min}\le b_1(\theta)\le \overline b<\infty.
\label{eq:b-bounds}
\end{align}
The gain map in \eqref{eq:gain} is therefore Lipschitz on $\Theta$; fix a
strictly positive known Lipschitz bound $L_\lambda>0$ satisfying
\begin{align}
\norm{\lam(\theta)-\lam(\vartheta)}
\le L_\lambda\norm{\theta-\vartheta},
\qquad \theta,\vartheta\in\Theta.
\label{eq:gain-lipschitz}
\end{align}
Write $\Lambda:=\sup_{\theta\in\Theta}\norm{\lam(\theta)}$ and
$B_b:=\sup_{\theta\in\Theta}\sum_{j=1}^{q}|b_j(\theta)|$.
\end{assumption}

Let $A_y(\theta)$ be the companion matrix in \eqref{eq:Y-state}.
A parameter sequence is admissible if it takes values in $\Theta$ and is
piecewise constant at the specified segment boundaries. Individual Schur
stability does not ensure stability under switching, so we impose the
following uniform product bound.

\begin{assumption}[Uniform switching stability]
\label{ass:switching-stability}
There are $1\le C_y<\infty$ and $\rho_y\in(0,1)$ such that, for every
admissible parameter sequence and every $k\ge0$,
\begin{align}
\norm{A_y(\theta_{t+k-1})\cdots A_y(\theta_t)}
\le C_y\rho_y^k.
\label{eq:switching-stability}
\end{align}
\end{assumption}

\begin{assumption}[Noise]
\label{ass:noise}
Each $w_t$ is independent of $\F_{t-1}$, and the disturbances are identically
distributed. They satisfy
\begin{align}
\E w_t=0,\qquad \E w_t^2=\sigma_w^2>0,
\qquad |w_t|\le B_w\quad\text{almost surely}.
\notag
\end{align}
The detector knows $\sigma_w^2$ and $B_w$.
\end{assumption}

The algorithm always clips its input to $[-B_u,B_u]$. Consequently,
Assumptions~\ref{ass:model}--\ref{ass:noise} imply the
deterministic bounds
\begin{align}
B_Y&:=C_y\norm{Y_0}+\frac{C_y}{1-\rho_y}(B_w+B_bB_u),
&|y_t|&\le B_Y,
\notag\\
B_\psi&:=\{pB_Y^2+(q-1)B_u^2\}^{1/2},
&\norm{\psi_t}&\le B_\psi.
\label{eq:Bpsi}
\end{align}
A proof is given in Appendix~\ref{app:bounds}.

For the baseline analysis, the ideal action must also be feasible at every
reachable regressor. Here, reachability is from the fixed initial history
under admissible parameter sequences, disturbances bounded by $B_w$, and
inputs in $[-B_u,B_u]$.

\begin{assumption}[Uniform clipping feasibility]
\label{ass:clipping}
The ideal action belongs to the closed clipping interval:
\begin{align}
\abs{\lam(\theta)'\psi}\le B_u
\quad\text{for every }\theta\in\Theta
\quad\text{and every reachable }\psi\text{ with }
\norm{\psi}\le B_\psi.
\notag
\end{align}
The stronger condition $\Lambda B_\psi\le B_u$ is sufficient.
Equality is allowed; projection nonexpansiveness requires no positive margin.
\end{assumption}

For a gain $\ell\in\R^d$ and parameter $\theta$, let
$A_{\rm cl}(\ell,\theta)$ be the matrix satisfying
\begin{align}
\psi_{t+1}=A_{\rm cl}(\ell,\theta)\psi_t+e_1w_{t+1}
\label{eq:closed-loop-matrix}
\end{align}
when the unclipped input $u_t=\ell'\psi_t$ is applied. Here $e_1$ is the
first coordinate vector.
If the input is $u_t=\ell'\psi_t+e_t$, the regressor dynamics become
\[
\psi_{t+1}
=A_{\rm cl}(\ell,\theta)\psi_t
+g_u(\theta)e_t+e_1w_{t+1},
\]
where the input-channel vector and its uniform bound are defined by
\begin{align}
g_u(\theta)
:=\begin{cases}
b_1(\theta)e_1, &q=1,\\
(b_1(\theta),0_{p-1}',1,0_{q-2}')', &q\ge2,
\end{cases}
\qquad
G_u:=\sup_{\theta\in\Theta}\norm{g_u(\theta)}
\le\sqrt{\overline b^2+1}.
\label{eq:input-channel}
\end{align}
Here $0_k$ denotes the zero vector in $\R^k$; the final block is absent when
$q=2$.
When $A_{\rm cl}(\ell,\theta)$ is Schur stable, define its
stationary output power
\begin{align}
J_y(\ell;\theta)
:=\sigma_w^2\sum_{j=0}^{\infty}
\{e_1'A_{\rm cl}(\ell,\theta)^je_1\}^2.
\label{eq:stationary-power}
\end{align}
To justify this terminology, freeze $(\ell,\theta)$ and put
$A=A_{\rm cl}(\ell,\theta)$.  The zero-mean stationary solution of
$\psi_{t+1}=A\psi_t+e_1w_{t+1}$ is
\begin{align*}
 \psi_t^{\rm stat}=\sum_{j=0}^{\infty}A^je_1w_{t-j}.
\end{align*}
The series converges in mean square because $A$ is Schur.  Its first
coordinate is the stationary output.  Independence, zero mean, and
$\E w_t^2=\sigma_w^2$ eliminate all cross terms, so
\begin{align*}
 \E[(y_t^{\rm stat})^2]
 &=\sum_{j,k\ge0}(e_1'A^je_1)(e_1'A^ke_1)
       \E[w_{t-j}w_{t-k}]\\
 &=\sigma_w^2\sum_{j=0}^{\infty}(e_1'A^je_1)^2
 =J_y(\ell;\theta).
\end{align*}
Thus \eqref{eq:stationary-power} is the stationary second moment (and,
because the mean is zero, the variance) of the \emph{unclipped frozen-gain}
output process.  It is not asserted directly for the implemented clipped
trajectory; Lemmas~\ref{lem:gate} and \ref{lem:coupling} provide that
comparison.
In particular,
$J_y(\lam(\theta);\theta)=\sigma_w^2$.

\begin{assumption}[Stable and detectable adjacent changes]
\label{ass:detectability}
There are $C_\Delta<\infty$, $\rho_\Delta\in(0,1)$, and $\Delta>0$ such
that, for every actual adjacent pair
$(\theta^{(c-1)},\theta^{(c)})$,
\begin{align}
\norm{A_{\rm cl}(\lam(\theta^{(c-1)}),\theta^{(c)})^k}
&\le C_\Delta\rho_\Delta^k,\qquad k\ge0,
\label{eq:mismatch-stability}\\
J_y(\lam(\theta^{(c-1)});\theta^{(c)})
&\ge \sigma_w^2+\Delta.
\label{eq:power-gap}
\end{align}
\end{assumption}

The power-gap condition excludes changes that leave the old controller's
output power unchanged. Appendix~\ref{app:relaxations} gives two extensions:
a burn-in variant under eventual nonclipping, and a boundary-alarm variant
covering both stable power gaps and finite-order visibility.

Tuning requires valid bounds for $C_\Delta$, $\rho_\Delta$, and $\Delta$,
and the model and exploration bounds defining $B_\psi$, $B_o$, and
$\kappa_G$ in \eqref{eq:Bpsi}, \eqref{eq:detector-bounds}, and
\eqref{eq:gram-constants}. Conservative bounds suffice; the analysis does
not provide a data-driven calibration rule.

\section{THE PIECE-CD ALGORITHM}
\label{sec:algorithm}

Each PIECE-CD cycle consists of $H$ exploration rounds followed by gated RLS
control and output-energy monitoring. An alarm starts a new cycle.

\subsection{Exploration and Estimation}

Let $s:=p+q$, $r_0:=\max\{p,q\}$, and $n_H:=H-r_0$. Fix a ridge parameter
$\nu>0$. Exploration uses fresh i.i.d. mean-zero inputs, bounded by
$B_e\le B_u$ and of variance $\sigma_e^2>0$, independent of the
disturbance sequence. Conditional on $\F_r$, $u_r$ is fixed and the subsequent
exploration inputs remain independent with this distribution.
For a full unchanged exploration block satisfying
\eqref{eq:block-excitation-size}, Proposition~\ref{prop:restart-gram} gives
\begin{align}
\lambda_{\min}\!\left(\sum_{t=r+r_0}^{r+H-1}
\phi_t\phi_t'\right)\ge\kappa_G n_H
\notag
\end{align}
with conditional probability at least $1-\delta_0$ given $\F_r$.
The constant $\kappa_G>0$ in \eqref{eq:gram-constants} is uniform over
$\Theta$ and the bounded restart states.

Put $B_\phi:=(B_\psi^2+B_u^2)^{1/2}$ and
$S_\Theta:=\sup_{\theta\in\Theta}\norm{\theta}$. For a confidence level
$\delta_0$, define the deterministic estimation bound
\begin{align}
\mathcal L_H(\delta_0)
&:=2\log(1/\delta_0)
+s\log\!\left(1+\frac{HB_\phi^2}{s\nu}\right),
\notag\\
E_\theta(H,\delta_0)
&:=\frac{2\{B_w\sqrt{\mathcal L_H(\delta_0)}
+\sqrt{\nu}S_\Theta\}}
{\sqrt{\nu+\kappa_G n_H}},
\label{eq:Etheta}\\
E_\lambda(H,\delta_0)&:=L_\lambda E_\theta(H,\delta_0),
&g_H&:=4E_\lambda(H,\delta_0).
\label{eq:gain-error-bound}
\end{align}
For the estimator and reference gain in
\eqref{eq:ridge-theta}--\eqref{eq:reference-gain},
Proposition~\ref{prop:restart-estimate} establishes
$\norm{\thetahat_r-\theta}\le E_\theta$ and
$\norm{\lambdabar_r-\lam(\theta)}\le E_\lambda$
on the joint Gram and estimation event.

For the horizon-wide theorem we use
\begin{align}
\delta_d:=\delta/2,\qquad
\delta_0:=\frac{\delta}{32(T+1)}.
\label{eq:confidence-allocation}
\end{align}
This allocation remains valid without knowing $C$, because no more than
$T+1$ restart locations are possible.

\subsection{Tuning and Dwell Time}

On the successful reference event, Lemma~\ref{lem:gate} bounds the action
and unchanged-regime output errors by
\begin{align}
\eps_u:=\frac54 g_HB_\psi,\qquad
\eps_y:=\overline b\,\eps_u.
\label{eq:tracking-errors}
\end{align}
For a post-change fixed-old-gain auxiliary trajectory, set
\begin{align}
B_{\rm aux}
&:=C_\Delta B_\psi+
\frac{C_\Delta B_w}{1-\rho_\Delta},
&B_o&:=\max\{B_Y,B_{\rm aux},B_w\},
\label{eq:detector-bounds}\\
R_\Delta(H)
&:=\frac{C_\Delta G_u\eps_u}{1-\rho_\Delta}.
\notag
\end{align}
Choose $H>r_0$ as the smallest integer for which the sample-size condition
\eqref{eq:block-excitation-size} and the following inequalities hold:
\begin{align}
\frac{E_\theta(H,\delta_0)}{b_{\min}}&\le\frac1{16},
&\frac{B_\psi^2}{\nu+\kappa_Gn_H}&\le\frac1{16},
\label{eq:H-rls}\\
\eps_y^2&\le\frac{\Delta}{8},\qquad \eps_y\le B_o,
&2B_oR_\Delta(H)&\le\frac{\Delta}{8}.
\label{eq:H-detector}
\end{align}
Condition~\eqref{eq:H-rls} controls the RLS bias and leverage;
\eqref{eq:H-detector} controls the null bias and coupling error.

Choose the integer sampling interval $h$, window size $m$, and delay bound
$D$ by
\begin{align}
h&:=\min\left\{k\ge1:
\frac{\sigma_w^2C_\Delta^2\rho_\Delta^{2k}}
{1-\rho_\Delta^2}\le\frac{\Delta}{8}\right\},
\label{eq:h}\\
m&:=\left\lceil
\frac{512B_o^4}{\Delta^2}
\log\frac{4T^2}{\delta_d}\right\rceil,
&D&:=(m+1)h.
\label{eq:mD}
\end{align}
The spacing truncates the stationary-power tail at $\Delta/8$.
The window size controls both null and post-change deviations uniformly over
potential restart locations and test endpoints.

\begin{assumption}[Minimum dwell time]
\label{ass:dwell}
The first change satisfies $\tau^{(1)}\ge H$ when $C\ge1$. Every pair of
successive changes satisfies
\begin{align}
\tau^{(c+1)}-\tau^{(c)}\ge D+H,
\qquad c=1,\ldots,C-1,
\notag
\end{align}
and the terminal segment satisfies $T-\tau^{(C)}\ge D$ when $C\ge1$.
When each change is detected within $D$ rounds, these conditions leave a full
unchanged exploration block before the next change. The final exploration
block may be truncated at $T$.
\end{assumption}

\subsection{Algorithm}

For a vector $x$, let $\Pi_\Theta(x)$ be a measurable Euclidean projection
onto $\Theta$.  If $\Theta$ is not convex, any measurable nearest-point
selection can be used; the only property required below is
$\norm{\Pi_\Theta(x)-\theta}\le2\norm{x-\theta}$ for $\theta\in\Theta$.
Indeed, since $\theta$ itself is an admissible point in the nearest-point
minimization,
\begin{align*}
 \norm{\Pi_\Theta(x)-x}\le\norm{\theta-x}.
\end{align*}
The triangle inequality therefore gives
$\norm{\Pi_\Theta(x)-\theta}
 \le\norm{\Pi_\Theta(x)-x}+\norm{x-\theta}
 \le2\norm{x-\theta}$.  This factor transfers the unprojected ridge error to
the feasible estimate in Proposition~\ref{prop:restart-estimate}.  If
$\Theta$ is closed and convex, projection is nonexpansive and the factor two
can be replaced by one.

\medskip
\refstepcounter{algorithm}\label{alg:piece-cd}
\noindent\textbf{Algorithm \thealgorithm: The PIECE-CD algorithm}
\smallskip
\hrule
\smallskip
\begin{enumerate}[leftmargin=1.8em,itemsep=0.25em,topsep=0.25em]
\item[] \textbf{Input:} the known orders and model class; $T,\delta$; all
tuning quantities defined above; and the exploration distribution.
\item[] Initialize the restart time $r\leftarrow0$ and clear the detector
buffer.
\item[] \textbf{Loop.} While $r<T$, execute the following cycle.
  \begin{enumerate}[label=(\alph*),leftmargin=1.8em,itemsep=0.2em]
  \item \textbf{Explore.} For
  $t=r,\ldots,\min\{r+H-1,T-1\}$, draw the independent exploratory input,
  apply it, and observe $y_{t+1}$. If the horizon ends, stop.

  \item \textbf{Estimate.} Using the post-boundary rows
  $\I_r:=\{r+r_0,\ldots,r+H-1\}$, form
  \begin{align}
  V_r^\theta&:=\nu I_s+\sum_{t\in\I_r}\phi_t\phi_t',
  &\widetilde\theta_r&:=(V_r^\theta)^{-1}
  \sum_{t\in\I_r}\phi_ty_{t+1},
  \label{eq:ridge-theta}\\
  \thetahat_r&:=\Pi_\Theta(\widetilde\theta_r),
  &\bhat_r&:=b_1(\thetahat_r),
  &\lambdabar_r&:=\lam(\thetahat_r).
  \label{eq:reference-gain}
  \end{align}
  Initialize
  \begin{align}
  V_{r+H}:=\nu I_d+\sum_{t\in\I_r}\psi_t\psi_t',
  \qquad \lambdahat_{r+H}:=\lambdabar_r.
  \label{eq:rls-initialization}
  \end{align}
  Clear the detector buffer and declare the output-time sampling grid
  \begin{align}
  \mathcal S_r:=\{r+H+h+jh:j=0,1,\ldots\}.
  \notag
  \end{align}

  \item \textbf{Exploit and monitor.} For
  $t=r+H,r+H+1,\ldots,T-1$, set
  \begin{align}
  z_t:=
  \begin{cases}
  \lambdahat_t'\psi_t,
  &\abs{(\lambdahat_t-\lambdabar_r)'\psi_t}
  \le g_H\norm{\psi_t},\\
  \lambdabar_r'\psi_t,&\text{otherwise},
  \end{cases}
  \qquad u_t:=\clip{z_t}.
  \label{eq:gated-control}
  \end{align}
  Before observing $y_{t+1}$, declare it detector-eligible exactly when
  $t+1\in\mathcal S_r$. Observe $y_{t+1}$ and make the standard RLS update
  with pseudo-response
  \begin{align}
  x_{t+1}^\lambda&:=u_t-\frac{y_{t+1}}{\bhat_r},
  &V_{t+1}&:=V_t+\psi_t\psi_t',
  \label{eq:pseudo-response}\\
  \lambdahat_{t+1}
  &:=\lambdahat_t+V_{t+1}^{-1}\psi_t
  (x_{t+1}^\lambda-\lambdahat_t'\psi_t).
  \label{eq:rls-update}
  \end{align}

  If $y_{t+1}$ was declared eligible, append $y_{t+1}^2$ to the buffer.
  Retain only the most recent $m$ stored values. When the buffer is full,
  compute
  \begin{align}
  \widehat J(t+1):=\frac1m\sum_{z\in\mathrm{buffer}}z.
  \notag
  \end{align}
  If $\widehat J(t+1)\ge\sigma_w^2+\Delta/2$, raise an alarm, set
  $r\leftarrow t+1$, clear the buffer, and begin a new cycle.
  If no alarm is raised and $t=T-1$, stop.
  \end{enumerate}
\end{enumerate}
\smallskip
\hrule
\medskip


\section{CHANGE-DETECTION GUARANTEES}
\label{sec:detection}

Throughout this appendix, the assumptions and tuning of
Appendices~\ref{sec:model} and~\ref{sec:algorithm} are in force.
Fix a restart $r$ with a full \emph{uncontaminated} exploration block:
the parameter remains equal to $\theta^-$ throughout $[r,r+H)$.
Write $\lambda^-:=\lam(\theta^-)$ and define the reference event
\begin{align}
\G_r^{\rm est}:=
\left\{\norm{\lambdabar_r-\lambda^-}\le g_H/4\right\}
\notag
\end{align}
whose failure probability is controlled by
Proposition~\ref{prop:restart-estimate}.

\begin{lemma}[Gate and action tracking]
\label{lem:gate}
On $\G_r^{\rm est}$, at every exploitation time before the next alarm,
including times after an undetected change,
\begin{align}
\abs{u_t-(\lambda^-)'\psi_t}
\le\frac54g_H\norm{\psi_t}\le\eps_u.
\label{eq:action-tracking}
\end{align}
If no true change has occurred and
$a_t:=(\lambdahat_t-\lambda^-)'\psi_t$, then the applied action error
$e_t:=u_t-(\lambda^-)'\psi_t$ also satisfies
\begin{align}
e_t^2\le a_t^2.
\label{eq:gate-dominance}
\end{align}
\end{lemma}

\begin{proof}
If the recursive branch of \eqref{eq:gated-control} is selected, then
\[
\abs{z_t-(\lambda^-)'\psi_t}
 \le g_H\norm{\psi_t}+(g_H/4)\norm{\psi_t}.
\]
If the reference branch is selected, the same left-hand side is at most
$(g_H/4)\norm{\psi_t}$.  Since the ideal action
$(\lambda^-)'\psi_t$ lies in $[-B_u,B_u]$ by
Assumption~\ref{ass:clipping}, the metric projection onto this interval is
nonexpansive and \eqref{eq:action-tracking} follows after clipping.

For \eqref{eq:gate-dominance}, equality holds in the recursive branch before
clipping, and nonexpansiveness can only reduce the error.  In the reference
branch,
\begin{align*}
|a_t|
&\ge\abs{(\lambdahat_t-\lambdabar_r)'\psi_t}
  -\abs{(\lambdabar_r-\lambda^-)'\psi_t}\\
&>\frac34g_H\norm{\psi_t},
\end{align*}
whereas the reference action error is at most
$(g_H/4)\norm{\psi_t}$.  Clipping again cannot increase it.
\end{proof}

\subsection{Sampled-Square Concentration}

The sampling interval $h$ allows concentration around conditional moments
given the history $h$ steps earlier.

\begin{lemma}[Lagged sampled-square concentration]
\label{lem:sampled-square}
Let $h\le s_1<\cdots<s_m$ be finite lag-$h$ predictable output times, meaning
$\{s_k=t\}\in\F_{t-h}$ for every $t$, and suppose
$s_{k+1}-s_k\ge h$. If $X_{s_k}$ is $\F_{s_k}$-measurable and
$|X_{s_k}|\le B_o$, then, for every $a>0$,
\begin{align}
\Prob\!\left(
\frac1m\sum_{k=1}^{m}
\{X_{s_k}^2-\E[X_{s_k}^2\mid\F_{s_k-h}]\}\le-a
\right)
\le\exp\!\left(-\frac{ma^2}{2B_o^4}\right).
\label{eq:sampled-square}
\end{align}
The analogous upper-tail bound also holds.
\end{lemma}

\begin{proof}
Set
$Z_k=X_{s_k}^2-\E[X_{s_k}^2\mid\F_{s_k-h}]$.
With $\mathcal H_{k-1}=\F_{s_k-h}$ and, for $k<m$,
$\mathcal H_k=\F_{s_{k+1}-h}$, the spacing condition implies that $Z_k$ is
$\mathcal H_k$-measurable and has conditional mean zero.  Moreover,
$\mathcal H_m:=\F_{s_m}$ contains $Z_m$.  Lag predictability makes the
random-indexed sigma-fields well defined as stopped sigma-fields, and
$\F_{s_k}\subseteq\F_{s_{k+1}-h}$.  Moreover, $|Z_k|\le B_o^2$.
The Azuma--Hoeffding inequality
\citep{azuma1967weighted,hoeffding1963probability} gives
\eqref{eq:sampled-square}; the upper tail follows by applying it to $-Z_k$.
\end{proof}

\subsection{False Alarms}

For each potential restart $r$, define an \emph{alarm-suppressed
continuation} by retaining the control, RLS, grid, and buffer recursions
while recording threshold crossings without restarting. It coincides with
the implemented trajectory through the first alarm.

A block is \emph{null for restart $r$} if every stored output is generated
under the reference parameter $\theta^-$. If that parameter remains active
until control time $\tau$, a block ending at output time $\tau$ is null:
$y_\tau$ is generated by $u_{\tau-1}$. On $\G_r^{\rm est}$,
\eqref{eq:gain-form} and Lemma~\ref{lem:gate} give, for each stored output
in a null block,
\begin{align}
y_s=w_s+v_s,\qquad
v_s:=b_1(\theta^-)\{u_{s-1}-(\lambda^-)'\psi_{s-1}\},\qquad
|v_s|\le\eps_y.
\label{eq:null-decomposition}
\end{align}
The variable $v_s$ is $\F_{s-1}$-measurable, while $w_s$ is independent of
$\F_{s-1}$. Hence, for any null detector block $S$ of $m$ stored samples,
\begin{align}
\frac1m\sum_{s\in S}(y_s^2-\sigma_w^2)
&=\frac1m\sum_{s\in S}(w_s^2-\sigma_w^2)
+\frac2m\sum_{s\in S}w_sv_s
+\frac1m\sum_{s\in S}v_s^2.
\label{eq:null-energy}
\end{align}

\begin{lemma}[One-block false-alarm bound]
\label{lem:false-block}
Suppose \eqref{eq:H-detector} holds. Conditional on any $\F_{r+H}$ history
on which $\G_r^{\rm est}$ holds, any fixed potential block generated by the
predictable sampling grid that is null for restart $r$ crosses the threshold
with probability at most $\delta_d/(2T^2)$.
\end{lemma}

\begin{proof}
By \eqref{eq:null-decomposition}, every summand in the last term of
\eqref{eq:null-energy} satisfies $v_s^2\le\eps_y^2$.  Consequently,
\begin{align*}
 \frac1m\sum_{s\in S}v_s^2\le\eps_y^2\le\frac{\Delta}{8},
\end{align*}
where the final inequality is exactly the first condition in
\eqref{eq:H-detector}.
The standard bounded-increment exponential supermartingale for the
predictably sampled disturbance squares
\citep{hoeffding1963probability,howard2020timeuniform}, applied along the fixed
predictable block through its predetermined endpoint, gives
\begin{align}
\Prob\!\left(
 \frac1m\sum_{s\in S}(w_s^2-\sigma_w^2)>\frac\Delta8
 \right)
 &\le\exp\!\left(-\frac{m\Delta^2}{32B_w^4}\right).
\label{eq:w2-null}
\end{align}
The products $w_sv_s$ form bounded martingale differences, so the same
Azuma--Hoeffding construction
\citep{azuma1967weighted,hoeffding1963probability} applies.  If
$\eps_y=0$, this cross term vanishes identically.  Otherwise, the analogous
stopped exponential bound gives
\begin{align}
\Prob\!\left(\frac2m\sum_{s\in S}w_sv_s>\frac\Delta8\right)
 &\le\exp\!\left(-\frac{m\Delta^2}
 {512B_w^2\eps_y^2}\right).
\label{eq:cross-null}
\end{align}
With $m$ from \eqref{eq:mD} and $B_o\ge\max\{B_w,\eps_y\}$, each right-hand
side is at most $\delta_d/(4T^2)$.  Outside these two deviations, the left side
of \eqref{eq:null-energy} is at most $3\Delta/8<\Delta/2$.
\end{proof}

\subsection{Detection after a Change}

Suppose the current cycle has not alarmed when the parameter changes from
$\theta^-$ to $\theta^+$ at control time $\tau$. Starting from
$\psi_\tau^{\rm aux}=\psi_\tau$, define the auxiliary process driven by the
same future disturbances and the exact old gain:
\begin{align}
\psi_{t+1}^{\rm aux}
&=A_\Delta\psi_t^{\rm aux}+e_1w_{t+1},
&A_\Delta&:=A(\theta^-,\theta^+).
\label{eq:auxiliary}
\end{align}
Let $y_t^{\rm aux}$ be its first coordinate. By
\eqref{eq:mismatch-stability}, $|y_t^{\rm aux}|\le B_{\rm aux}$.

\begin{lemma}[Actual-to-auxiliary coupling]
\label{lem:coupling}
On $\G_r^{\rm est}$, until the first of the next alarm and the next true
change after $\tau$,
\begin{align}
\norm{\psi_t-\psi_t^{\rm aux}}
&\le R_\Delta(H),
\label{eq:state-coupling}\\
\abs{y_t^2-(y_t^{\rm aux})^2}
&\le2B_oR_\Delta(H)\le\Delta/8.
\label{eq:square-coupling}
\end{align}
\end{lemma}

\begin{proof}
Writing the actual input as
$u_t=\lam(\theta^-)'\psi_t+e_t$ and subtracting
\eqref{eq:auxiliary} yields
\[
\psi_{t+1}-\psi_{t+1}^{\rm aux}
=A_\Delta(\psi_t-\psi_t^{\rm aux})+g_u(\theta^+)e_t.
\]
Put $\delta_t:=\psi_t-\psi_t^{\rm aux}$.  Because the two processes start
from the same state, $\delta_\tau=0$.  Iterating the displayed recursion gives,
for every $t>\tau$ before the stopping endpoint,
\begin{align*}
 \delta_t
 &=\sum_{k=\tau}^{t-1}A_\Delta^{t-1-k}g_u(\theta^+)e_k.
\end{align*}
Lemma~\ref{lem:gate} gives $|e_k|\le\eps_u$, while
\eqref{eq:input-channel} and \eqref{eq:mismatch-stability} give
$\norm{g_u(\theta^+)}\le G_u$ and
$\norm{A_\Delta^j}\le C_\Delta\rho_\Delta^j$.  Therefore
\begin{align*}
 \norm{\delta_t}
 &\le C_\Delta G_u\eps_u
       \sum_{k=\tau}^{t-1}\rho_\Delta^{t-1-k}\\
 &\le \frac{C_\Delta G_u\eps_u}{1-\rho_\Delta}
  =R_\Delta(H),
\end{align*}
which is \eqref{eq:state-coupling}.  Since output is the first state
coordinate,
$|y_t-y_t^{\rm aux}|\le\norm{\delta_t}\le R_\Delta(H)$.  Both outputs are
bounded by $B_o$, and hence
\begin{align*}
 |y_t^2-(y_t^{\rm aux})^2|
 &=|y_t-y_t^{\rm aux}|\,|y_t+y_t^{\rm aux}|\\
 &\le2B_oR_\Delta(H)\le\frac{\Delta}{8},
\end{align*}
where the last inequality is the second condition in
\eqref{eq:H-detector}.  This proves \eqref{eq:square-coupling}.
\end{proof}

For every auxiliary output time $s$ whose preceding $h$ controls are all
post-change, independence gives the exact finite-horizon identity
\begin{align}
\E[(y_s^{\rm aux})^2\mid\F_{s-h}]
 &=(e_1'A_\Delta^h\psi_{s-h}^{\rm aux})^2
 +\sigma_w^2\sum_{j=0}^{h-1}(e_1'A_\Delta^je_1)^2.
\label{eq:finite-power}
\end{align}
Indeed, iterating \eqref{eq:auxiliary} for $h$ steps and taking the first
coordinate gives
\begin{align*}
 y_s^{\rm aux}
 &=e_1'A_\Delta^h\psi_{s-h}^{\rm aux}
   +\sum_{j=0}^{h-1}e_1'A_\Delta^je_1\,w_{s-j}.
\end{align*}
Conditional on $\F_{s-h}$, the first term is fixed.  Every future disturbance
is centered and independent of that field and of the other future
disturbances.  Thus the state--noise cross terms and the cross terms with
distinct disturbance indices have conditional expectation zero, while each
diagonal disturbance term contributes its coefficient squared times
$\sigma_w^2$.  This proves \eqref{eq:finite-power} at deterministic $s$.
The same identity holds at a finite lag-$h$ predictable random output time:
on each event $\{s=t\}\in\F_{t-h}$ it is the deterministic-time identity,
and summing over $t$ gives the equality conditional on the stopped field
$\F_{s-h}$.  The first term is nonnegative.  Equations~\eqref{eq:power-gap} and
\eqref{eq:h} imply the remaining inequality as follows.  With
$A_\Delta=A_{\rm cl}(\lambda(\theta^-),\theta^+)$,
\begin{align*}
 \sigma_w^2\sum_{j=0}^{h-1}(e_1'A_\Delta^je_1)^2
 &=J_y(\lambda(\theta^-);\theta^+)
   -\sigma_w^2\sum_{j=h}^{\infty}(e_1'A_\Delta^je_1)^2\\
 &\ge \sigma_w^2+\Delta
   -\sigma_w^2\sum_{j=h}^{\infty}\norm{A_\Delta^j}^2\\
 &\ge \sigma_w^2+\Delta
   -\frac{\sigma_w^2C_\Delta^2\rho_\Delta^{2h}}
          {1-\rho_\Delta^2}\\
 &\ge \sigma_w^2+\frac{7\Delta}{8}.
\end{align*}
Combining this bound with the nonnegative conditional-state term in
\eqref{eq:finite-power} proves
\begin{align}
\E[(y_s^{\rm aux})^2\mid\F_{s-h}]
 \ge\sigma_w^2+\frac{7\Delta}{8}.
\label{eq:alt-conditional-power}
\end{align}

\begin{lemma}[One-change detection]
\label{lem:one-change}
Fix a potential restart $r$ and a change time $\tau$ with
$r+H\le\tau$ and $\tau+D\le T$. Let $\mathcal A_{r,\tau}\in\F_\tau$
be the event that the cycle at $r$ is reached, $\theta^-$ is active throughout
$[r,\tau)$, $\G_r^{\rm est}$ holds, and the cycle remains in exploitation
after observing $y_\tau$. Suppose $\theta^+$ is active throughout
$[\tau,\tau+D)$. Let $\widehat\tau_r$ be this cycle's first alarm output
time strictly after $\tau$, with value $+\infty$ if none occurs. Then,
almost surely on $\mathcal A_{r,\tau}$,
\begin{align}
\Prob(\widehat\tau_r>\tau+D\mid\F_\tau)
\le\frac{\delta_d}{4T^2}.
\label{eq:conditional-one-change}
\end{align}
\end{lemma}

\begin{proof}
Condition on an arbitrary $\F_\tau$ history in
$\mathcal A_{r,\tau}$.  In the continuation in which alarms are suppressed,
the restart-relative sampling grid, the auxiliary initial state $\psi_\tau$,
and the old reference are now fixed, while all future disturbances remain
fresh.  Within fewer than $2h$ steps, the grid has an output time $s_1$ with
$s_1-h\ge\tau$.  The next $m-1$ grid times are $h$ steps apart, and the
last satisfies
\begin{align*}
s_m-\tau<2h+(m-1)h=(m+1)h=D.
\end{align*}
Thus the complete block exists within the finite horizon.
If any earlier block alarms, the claim is already true.
Otherwise the actual and
alarm-suppressed paths agree through this complete post-change block.

The conditional version of Lemma~\ref{lem:sampled-square} is valid because
the future innovations are fresh after $\F_\tau$ and
$\F_\tau\subseteq\F_{s_1-h}$ by construction.  It, together with
\eqref{eq:alt-conditional-power} and $a=\Delta/8$, gives
\begin{align*}
\Prob\!\left(
 \frac1m\sum_{k=1}^{m}(y_{s_k}^{\rm aux})^2
 <\sigma_w^2+\frac{3\Delta}{4}\right)
&\le\exp\!\left(-\frac{m\Delta^2}{128B_o^4}\right)
 \le\frac{\delta_d}{4T^2}.
\end{align*}
Outside this deviation, Lemma~\ref{lem:coupling} applies because no second
change occurs before the tested block and gives an actual average at least
$\sigma_w^2+5\Delta/8$, which crosses the detector threshold.  This proves
\eqref{eq:conditional-one-change}.  Equivalently, for every
$A\in\F_\tau$ with $A\subseteq\mathcal A_{r,\tau}$,
\begin{align*}
\Prob(A\cap\{\widehat\tau_r>\tau+D\})
\le\frac{\delta_d}{4T^2}\Prob(A),
\end{align*}
which is the eventwise form used below.
\end{proof}

\subsection{Alarm Correspondence}

\begin{theorem}[Alarm correspondence and delay]
\label{thm:detection}
Suppose Assumptions~\ref{ass:model}--\ref{ass:dwell} hold. Let
$\sigma_{\rm det}$ be the first output time at which a false alarm occurs or
the first deadline at which a true change remains unmatched, with
$\inf\varnothing=+\infty$. For each potential $r$, let
$\mathcal E_r^{\mathrm{reach}}$ be the event that $r$ is reached, $r<\sigma_{\rm det}$, and
$[r,r+H)$ is a full, unchanged exploration block. Define
\begin{align}
\mathcal E_{\rm ref}
:=\bigcap_{r=0}^{T}\{(\mathcal E_r^{\mathrm{reach}})^c\cup\G_r^{\rm est}\}.
\label{eq:reference-event}
\end{align}
Let
$\mathcal E_{\rm det}$ be the event that PIECE-CD has no false alarms and
raises exactly one alarm $\widehat\tau^{(c)}$ for each true change, with
\begin{align}
\tau^{(c)}<\widehat\tau^{(c)}\le\tau^{(c)}+D,
\qquad c=1,\ldots,C.
\label{eq:detection-delay}
\end{align}
Then
\begin{align}
\Prob(\mathcal E_{\rm ref}\cap\mathcal E_{\rm det}^{c})
\le\delta_d.
\label{eq:gated-detector-bound}
\end{align}
\end{theorem}

\begin{proof}
There are at most $T$ potential restart locations and at most $T$ potential
test endpoints after each location.  Fix such a pair.  Intersect its crossing
event with the events that the restart is reached, its reference event holds,
no earlier detector failure has occurred, and every control generating an
output in the block uses the same parameter as the reference exploration;
that is, the block is null for this restart.  Conditional on the history
at the end of exploration, apply Lemma~\ref{lem:false-block} to the fixed
predictable block in the continuation with alarms suppressed.  If the actual
test is reached, the continuation and actual trajectory agree through that
test.  Thus the probability of this \emph{joint gated event}, rather than a
probability conditional on survival to the test, is at most
$\delta_d/(2T^2)$.  A union bound over the restart--endpoint pairs gives a
total false-alarm contribution at most $\delta_d/2$.

For a true change $\tau^{(c)}$, intersect missed detection with the successful
history up to that change.  For each potential $r$, the event that the past
detector history is successful, the cycle at $r$ is active at $\tau^{(c)}$,
and its reference succeeds is an $\F_{\tau^{(c)}}$-measurable subset of
$\mathcal A_{r,\tau^{(c)}}$ from Lemma~\ref{lem:one-change}.
These events are disjoint over $r$.
The dwell conditions give $r+H\le\tau^{(c)}$,
$\tau^{(c)}+D\le T$, and no second change in
$[\tau^{(c)},\tau^{(c)}+D)$.  Multiplying
\eqref{eq:conditional-one-change} by each active-cycle indicator, taking
expectations, and summing over the disjoint partition gives
$\delta_d/(4T^2)$ for that change, without an additional restart factor.
There are at most $C\le T-1$ changes, so the total missed-delay contribution
is at most
$C\delta_d/(4T^2)\le (T-1)\delta_d/(4T^2)<\delta_d/4$.
A block containing outputs from both
sides of a true change is not a null block; if it crosses, it is a valid earlier
detection.  On the
complement of these gated deviations, induction over the change index gives
one-to-one alarm correspondence.  The deterministic calendar calculation in
Appendix~\ref{app:detection} shows that the complete post-change block ends
strictly before $\tau^{(c)}+D$; no probability is spent on boundary or
integer-rounding conventions.  The dwell condition ensures that each
alarm occurs before the next change and, when another change remains, leaves
an uncontaminated $H$-step exploration block for the next cycle.  The
false-alarm contribution is at most $\delta_d/2$ and the missed-delay
contribution is less than $\delta_d/4$, so their sum is below $\delta_d$.
This proves \eqref{eq:gated-detector-bound}.
\end{proof}

\section{REGRET ANALYSIS}
\label{sec:regret}

On the detection event of Theorem~\ref{thm:detection}, the horizon splits
into exploration, detection-delay, and unchanged exploitation intervals.
The first two have bounded per-round cost; the RLS potential argument
controls the cumulative cost of the third.

\subsection{Segmentwise Exploitation Regret}

Consider an unchanged exploitation interval following a restart at $r$.
Suppress the restart subscript and write
$b_\star=b_1(\theta)$, $\lambda=\lam(\theta)$,
$\widehat b=\bhat_r$, and $\overline\lambda=\lambdabar_r$. Define
\begin{align}
\alpha&:=1-\frac{b_\star}{\widehat b},
&\overline\alpha&:=\frac{E_\theta(H,\delta_0)}{b_{\min}},
\notag\\
\underline v_H&:=\nu+\kappa_Gn_H,
&\overline\ell&:=\frac{B_\psi^2}{\underline v_H},
\notag\\
c_0&:=1-2\overline\alpha
-2\overline\ell(1+\overline\alpha)^2,
\notag\\
B_\eta&:=B_w/b_{\min},
&\overline Q_H&:=E_\lambda(H,\delta_0)^2(\nu+n_HB_\psi^2),
\notag\\
\Gamma(n)&:=d\log\!\left(1+\frac{nB_\psi^2}{d\underline v_H}\right).
\label{eq:regret-constants}
\end{align}
The tuning condition \eqref{eq:H-rls} ensures $c_0>0.7$.

\begin{theorem}[Logarithmic exploitation regret]
\label{thm:segment-regret}
Suppose a single parameter $\theta$ is active throughout the full exploration
block $[r,r+H)$ and throughout an exploitation interval
$[r+H,r+H+N)$ of that same cycle, ending no later than its next alarm.
Suppose the Gram and estimation events in
Propositions~\ref{prop:restart-gram} and \ref{prop:restart-estimate} hold for
this same $\theta$. Conditional on the history at the end of exploration,
with probability at least $1-\delta_0$, simultaneously for every prefix
length $0\le n\le N$,
\begin{align}
\sum_{t=r+H}^{r+H+n-1}(y_{t+1}-w_{t+1})^2
&\le R_E(n,\delta_0),
\label{eq:segment-regret}\\
R_E(n,\delta_0)
&:=\overline b^2\left[
\frac{2\overline Q_H}{c_0}
+\frac{8B_\eta^2}{c_0^2}\log\frac1{\delta_0}
+\frac{4B_\eta^2}{c_0}\Gamma(n)
\right].
\notag
\end{align}
\end{theorem}

\begin{proof}
We give the core argument; Appendix~\ref{app:rls} records the matrix identities
and conditioning details.  Let
$d_t=\lambdahat_t-\lambda$,
$a_t=d_t'\psi_t$, and
$e_t=u_t-\lambda'\psi_t$.  Lemma~\ref{lem:gate} gives
$e_t^2\le a_t^2$.  From \eqref{eq:gain-form}, the pseudo-response in
\eqref{eq:pseudo-response} obeys
\begin{align}
x_{t+1}^\lambda
 &=\lambda'\psi_t+\alpha e_t+\eta_{t+1},
&\eta_{t+1}&:=-w_{t+1}/\widehat b.
\notag
\end{align}
Here $\eta_{t+1}$ is conditionally mean zero and $B_\eta$-sub-Gaussian.

Let $Q_t=d_t'V_td_t$ and
$\ell_t=\psi_t'V_{t+1}^{-1}\psi_t$.  To display the algebra, set
$\omega_t:=\alpha e_t+\eta_{t+1}-a_t$.  The RLS update and
$V_{t+1}=V_t+\psi_t\psi_t'$ give
\begin{align*}
 d_{t+1}&=d_t+V_{t+1}^{-1}\psi_t\omega_t,\\
 Q_{t+1}
 &=d_t'V_{t+1}d_t+2\omega_td_t'\psi_t
    +\omega_t^2\psi_t'V_{t+1}^{-1}\psi_t\\
 &=Q_t+a_t^2+2a_t\omega_t+\ell_t\omega_t^2.
\end{align*}
Substituting the definition of $\omega_t$ and collecting
$a_t^2-2a_t^2=-a_t^2$ gives the exact identity
\begin{align}
Q_{t+1}-Q_t
 =-a_t^2+2a_t(\alpha e_t+\eta_{t+1})
  +\ell_t(\alpha e_t+\eta_{t+1}-a_t)^2.
\label{eq:potential-identity}
\end{align}
On the exploration event,
$|\alpha|\le\overline\alpha$ and
$\ell_t\le\overline\ell$.  Using $|e_t|\le|a_t|$ in
\eqref{eq:potential-identity} yields
\begin{align}
Q_{t+1}-Q_t
 \le-c_0a_t^2+2a_t\eta_{t+1}+2\ell_t\eta_{t+1}^2.
\label{eq:potential-bound}
\end{align}
For completeness, conditional $B_\eta$-sub-Gaussianity and predictability of
$a_t$ imply that, for every fixed $\xi>0$,
\begin{align*}
 M_n:=\exp\!\left\{
  \xi\sum_{t=r+H}^{r+H+n-1}a_t\eta_{t+1}
  -\frac{\xi^2B_\eta^2}{2}
      \sum_{t=r+H}^{r+H+n-1}a_t^2
 \right\}
\end{align*}
is a nonnegative supermartingale starting at one.  Ville's inequality
\citep{ville1939etude,howard2020timeuniform} therefore gives, with probability
at least $1-\delta_0$, simultaneously over every prefix,
\begin{align*}
 \sum_ta_t\eta_{t+1}
 \le \frac{\xi B_\eta^2}{2}\sum_ta_t^2
      +\frac1\xi\log\frac1{\delta_0}.
\end{align*}
Taking the deterministic value $\xi=c_0/(2B_\eta^2)$ yields
\begin{align}
\sum_ta_t\eta_{t+1}
 \le\frac{c_0}{4}\sum_ta_t^2
    +\frac{2B_\eta^2}{c_0}\log\frac1{\delta_0}.
\label{eq:martingale-absorb}
\end{align}
Appendix~\ref{app:rls} gives the same argument directly from the one-step
conditional moment-generating-function inequality.
Moreover, the matrix determinant lemma gives
\begin{align}
\sum_t\ell_t
 \le\log\frac{\det V_{r+H+n}}{\det V_{r+H}}
 \le\Gamma(n).
\label{eq:elliptical-potential}
\end{align}
Summing \eqref{eq:potential-bound}, using $Q_t\ge0$,
$\eta_{t+1}^2\le B_\eta^2$, and \eqref{eq:martingale-absorb}--
\eqref{eq:elliptical-potential}, gives
\begin{align}
\sum_ta_t^2
 \le\frac{2Q_{r+H}}{c_0}
 +\frac{8B_\eta^2}{c_0^2}\log\frac1{\delta_0}
 +\frac{4B_\eta^2}{c_0}\Gamma(n).
\label{eq:a-sum}
\end{align}
Finally,
$Q_{r+H}\le \overline Q_H$ and, by \eqref{eq:gain-form},
$(y_{t+1}-w_{t+1})^2=b_\star^2e_t^2\le\overline b^2a_t^2$.
This proves \eqref{eq:segment-regret}.
\end{proof}

\subsection{Delay and Exploration Costs}

The deterministic state bound gives the per-step delay cost
\begin{align}
K_D:=(B_Y+B_w)^2.
\label{eq:delay-cost}
\end{align}
During exploration, $|u_t|\le B_e$ and $\norm{\psi_t}\le B_\psi$, so
\eqref{eq:gain-form} gives the per-step bound
\begin{align}
K_X:=\overline b^2(B_e+\Lambda B_\psi)^2.
\label{eq:exploration-cost}
\end{align}
Thus a delay interval of length at most $D$ costs at most $K_DD$, and an
exploration block costs at most $K_XH$.

\subsection{Total Regret}

\begin{theorem}[Finite-time regret of PIECE-CD]
\label{thm:main}
Let Assumptions~\ref{ass:model}--\ref{ass:dwell} hold. Let $H$ be the
smallest integer greater than $r_0$ that satisfies
\eqref{eq:block-excitation-size} and
\eqref{eq:H-rls}--\eqref{eq:H-detector}; choose $h,m,D$ by
\eqref{eq:h}--\eqref{eq:mD}, and run Algorithm~\ref{alg:piece-cd}. Then,
with probability at least $1-\delta$,
there are no false alarms, every change has exactly one alarm satisfying
\eqref{eq:detection-delay}, and
\begin{align}
\Reg_T
\le CK_DD+(C+1)K_XH+(C+1)R_E(T,\delta_0).
\label{eq:main-bound}
\end{align}
Consequently, for fixed orders and fixed model, noise, initialization,
exploration, ridge, stability, clipping, and gap bounds,
\begin{align}
\Reg_T=O\!\left((C+1)\log\frac{T+1}{\delta}\right).
\label{eq:main-rate}
\end{align}
\end{theorem}

\begin{proof}
If $T\le H$, Assumption~\ref{ass:dwell} forces $C=0$.  The algorithm then
performs only truncated exploration, and the deterministic exploration bound
gives $\Reg_T\le K_XT\le K_XH$, which is bounded by the right-hand side of
\eqref{eq:main-bound}.  Hence assume $T>H$ below.

At every potential full, unchanged restart block, gate the estimation and RLS
failure events by the event that this restart is reached with a successful
history, and then condition on the past.  Propositions~\ref{prop:restart-gram}
and \ref{prop:restart-estimate} cost at most $2\delta_0$, and the time-uniform
RLS event in Theorem~\ref{thm:segment-regret} costs another $\delta_0$.
The tower property and a union bound over at most $T+1$ locations give
$3(T+1)\delta_0=3\delta/32$.  Gating the false-alarm and delay events in the
same chronological way, Theorem~\ref{thm:detection} costs less than
$\delta_d=\delta/2$.  Their intersection therefore has probability greater
than $1-\delta$.  Appendix~\ref{app:global} gives the full induction; in
particular, it never conditions an early detector event on future estimator
success.

On this intersection, each of the $C$ true changes has a pre-alarm interval
of length at most $D$, giving $CK_DD$.  There is one initial exploration block
and at most one block after each alarm, giving at most $(C+1)K_XH$.  Removing
these intervals
leaves at most $C+1$ unchanged exploitation prefixes.  Apply
Theorem~\ref{thm:segment-regret} to each and upper-bound its length by $T$;
this gives the last term in \eqref{eq:main-bound}.

Appendix~\ref{app:tuning} constructs a feasible
$H=O(\log((T+1)/\delta))$; the smallest feasible value is therefore no larger.
Equations~\eqref{eq:h}--\eqref{eq:mD} give fixed $h$
and $D=O(\log((T+1)/\delta))$.  Finally,
$\overline Q_H$, $\Gamma(T)$, and $\log(1/\delta_0)$ are all
$O(\log((T+1)/\delta))$.  Substitution proves \eqref{eq:main-rate}.
When $C=0$, the delay term is absent, leaving the stationary logarithmic
bound rather than the meaningless expression $O(C\log T)$.
\end{proof}

\section{DETERMINISTIC BOUNDS AND RESTART CONVENTIONS}
\label{app:bounds}

Bounded applied inputs and Assumptions~\ref{ass:model}--\ref{ass:noise}
yield the pathwise bounds \eqref{eq:Bpsi}.

\subsection{Output and Regressor Bounds}

Let $Y_t=(y_t,\ldots,y_{t-p+1})'$. Equation~\eqref{eq:arx} has the
state form
\begin{align}
Y_{t+1}=A_y(\theta_t)Y_t+e_1\zeta_{t+1},\qquad
\zeta_{t+1}:=w_{t+1}+\sum_{j=1}^{q}b_j(\theta_t)u_{t-j+1}.
\label{eq:Y-state}
\end{align}
Since $|u_t|\le B_u$ by construction,
\begin{align}
|\zeta_{t+1}|\le B_w+B_bB_u.
\notag
\end{align}
Unrolling \eqref{eq:Y-state} and applying
\eqref{eq:switching-stability} gives
\begin{align}
\norm{Y_t}
&\le C_y\rho_y^t\norm{Y_0}
+C_y(B_w+B_bB_u)\sum_{j=0}^{t-1}\rho_y^j
\notag\\
&\le C_y\norm{Y_0}
+\frac{C_y}{1-\rho_y}(B_w+B_bB_u)=B_Y.
\notag
\end{align}
Every output coordinate is therefore bounded by $B_Y$. The definition
\eqref{eq:psi} and clipping give
\begin{align}
\norm{\psi_t}^2
\le pB_Y^2+(q-1)B_u^2=B_\psi^2,
\notag
\end{align}
and \eqref{eq:phi} gives $\norm{\phi_t}\le B_\phi$.

\subsection{Projection onto the Clipping Interval}

Let $\operatorname{clip}_{[-B,B]}(x)=(-B)\vee x\wedge B$. As the Euclidean
projection onto a closed convex set, it is nonexpansive:
\begin{align}
|\operatorname{clip}_{[-B,B]}(x)
-\operatorname{clip}_{[-B,B]}(z)|\le|x-z|.
\notag
\end{align}
If $z\in[-B,B]$, this becomes
\begin{align}
|\operatorname{clip}_{[-B,B]}(x)-z|\le|x-z|.
\label{eq:clip-to-ideal}
\end{align}
Assumption~\ref{ass:clipping} permits \eqref{eq:clip-to-ideal} with
$z=\lambda(\theta)'\psi_t$ for $\theta\in\Theta$. For an arbitrary target
$z$, the bound includes its distance from the clipping interval:
\begin{align}
|\clip{x}-z|
\le |x-z|+(|z|-B_u)_+.
\notag
\end{align}

\subsection{Auxiliary-Process Bounds}

For the fixed mismatch matrix $A_\Delta$ in \eqref{eq:auxiliary},
\eqref{eq:mismatch-stability} gives, for $t\ge\tau$,
\begin{align}
\norm{\psi_t^{\rm aux}}
&\le C_\Delta\rho_\Delta^{t-\tau}\norm{\psi_\tau}
+C_\Delta B_w\sum_{j=0}^{t-\tau-1}\rho_\Delta^j
\notag\\
&\le C_\Delta B_\psi+
\frac{C_\Delta B_w}{1-\rho_\Delta}=B_{\rm aux}.
\notag
\end{align}
Thus both actual and auxiliary outputs are bounded in absolute value by
$B_o$.

\subsection{Conditioning at a Random Restart}

An alarm time $r$ is an $\{\F_t\}$-stopping time. Conditional on $\F_r$,
the state $\psi_r$ is fixed with $\norm{\psi_r}\le B_\psi$, and future
disturbances retain the independence, moments, and bound in
Assumption~\ref{ass:noise}. The input $u_r$ is already $\F_r$-measurable;
subsequent exploratory inputs are fresh and independent of the disturbances.

Consider a conditional concentration bound that holds uniformly over these
restart histories, with failure probability at most $\varepsilon_{\rm fail}$.
For any eligibility event $\mathcal C_r\in\F_r$, the tower property gives
\begin{align*}
 \Prob(\mathcal C_r\cap\{\text{failure at }r\})
 &=\E\!\left[
   \one_{\mathcal C_r}
   \Prob(\text{failure at }r\mid\F_r)
   \right]\\
 &\le\varepsilon_{\rm fail}\Prob(\mathcal C_r)
 \le\varepsilon_{\rm fail}.
\end{align*}
A union bound over at most $T+1$ possible restart locations then controls
these failures across the horizon. This is the basis for the allocation
\eqref{eq:confidence-allocation}.

\section{EXPLORATION, EXCITATION, AND THE RESTARTED ESTIMATE}
\label{app:exploration}

The covariance analysis of stationary PIECE in \citet{singhfinite} concerns
sparse, recurring exploration episodes. Here we establish linear Gram
growth and estimation bounds for a contiguous exploration block after a
random restart.

\subsection{Conditional Excitation}

Fix a complete exploration block $[r,r+H)$ during which the parameter
remains equal to $\theta$. Put $s:=p+q$, $L:=r_0=\max\{p,q\}$, and
$\underline\sigma^2:=\min\{\sigma_w^2,\sigma_e^2\}$. At any exploration
time $t\ge r+L$, collect the most recent innovations in
\[
Z_t:=(w_{t-p+1},\ldots,w_t,u_{t-q+1},\ldots,u_t)'.
\]
Because the parameter is fixed within the exploration block, repeated
substitution in the ARX recursion gives
\begin{align}
\phi_t=M(\theta)Z_t+\xi_t.
\label{eq:innovation-representation}
\end{align}
To make the remainder explicit, let $\overline Z_t$ collect the members of
$(w_{t-L+1},\ldots,w_t,u_{t-L+1},\ldots,u_t)$ that do not already occur in
$Z_t$. Repeated substitution gives
$\xi_t=c_t+\overline M(\theta)\overline Z_t$, where $c_t$ is
$\F_{t-L}$-measurable. Conditional on $\F_{t-L}$, the vector $Z_t$ is
centered, is independent of $\overline Z_t$, and has covariance
\[
\operatorname{Cov}(Z_t\mid\F_{t-L})
=\operatorname{diag}(\sigma_w^2I_p,\sigma_e^2I_q).
\]
Thus $Z_t$ and $\xi_t$ are conditionally independent. The remainder depends
on the $\F_{t-L}$-measurable history and the omitted innovations.

To see the triangular structure explicitly, order the output coordinates and
their disturbances chronologically as
$(y_{t-p+1},\ldots,y_t)$ and $(w_{t-p+1},\ldots,w_t)$.
The equation for $y_j$ contains $w_j$ with coefficient one and only
disturbances dated at or before $j$.  Thus the output--disturbance block is
triangular with unit diagonal.  Ordering the current exploration inputs as
$(u_{t-q+1},\ldots,u_t)$ makes the input-coordinate block the identity.  All
remaining post-$\F_{t-L}$ innovations are precisely those collected in
$\overline Z_t$; older terms belong to $c_t$.  Hence, up to these fixed row
and column permutations, $M(\theta)$ has the block-triangular form
\begin{align}
M(\theta)=
\begin{bmatrix}
T_w(\theta)&T_u(\theta)\\
0&I_q
\end{bmatrix},
\notag
\end{align}
where $T_w(\theta)$ is triangular with unit diagonal: each $y_j$ contains
its contemporaneous disturbance $w_j$ with coefficient one.  Hence
$M(\theta)$ is nonsingular.  Continuity and compactness give
\begin{align}
\chi:=\inf_{\theta\in\Theta}\sigma_{\min}(M(\theta))>0,
\qquad
\kappa_{\rm inn}:=\underline\sigma^2\chi^2>0.
\notag
\end{align}
For completeness, let
$\Sigma_Z:=\operatorname{Cov}(Z_t\mid\F_{t-L})$.  Conditional independence
removes both cross-covariance terms, and hence
\begin{align*}
\operatorname{Cov}(\phi_t\mid\F_{t-L})
 &=M(\theta)\Sigma_ZM(\theta)'
   +\operatorname{Cov}(\xi_t\mid\F_{t-L})\\
 &\succeq \underline\sigma^2M(\theta)M(\theta)'
 \succeq \underline\sigma^2\sigma_{\min}(M(\theta))^2I_s.
\end{align*}
Also,
$\E[\phi_t\phi_t'\mid\F_{t-L}]$ equals this covariance plus the positive
semidefinite outer product of the conditional mean.  Therefore
\begin{align}
\E[\phi_t\phi_t'\mid\F_{t-L}]
\succeq\operatorname{Cov}(\phi_t\mid\F_{t-L})
\succeq\kappa_{\rm inn}I_s.
\label{eq:conditional-excitation}
\end{align}
The bound is uniform over the history in $\F_{t-L}$.

\subsection{Exploration Gram Bound}

Fix a restart $r$ and the trimmed set
$\I_r=\{r+r_0,\ldots,r+H-1\}$. Let
\begin{align}
\Phi_r&:=\begin{bmatrix}
\phi_{r+r_0}'\\
\vdots\\
\phi_{r+H-1}'
\end{bmatrix},
&G_r&:=\Phi_r'\Phi_r
=\sum_{t\in\I_r}\phi_t\phi_t'.
\notag
\end{align}
Set
\begin{align}
N_H:=1+\left\lfloor\frac{n_H-1}{L}\right\rfloor,
\qquad
\kappa_G:=\frac{\kappa_{\rm inn}}{2L}.
\label{eq:gram-constants}
\end{align}
Consider only the $L$-spaced rows
$t_k=r+L+(k-1)L$, $k=1,\ldots,N_H$. They are contained in
$\I_r$, and $N_H\ge n_H/L$.

The conditioning argument in Appendix~\ref{app:bounds} applies at a
stopping-time restart $r$. Although $u_r$ is already $\F_r$-measurable, the
trimming starts at $t=r+L$, so every innovation appearing in $Z_t$ has index
at least $r+1$. Consequently,
\eqref{eq:innovation-representation}--\eqref{eq:matrix-freedman-gram} apply
without treating $r$ as deterministic.

\begin{proposition}[Exploration Gram bound]
\label{prop:restart-gram}
Let $r$ be a stopping time with $r+H\le T$, and suppose a single
$\theta\in\Theta$ is active at every control time
$t=r,\ldots,r+H-1$. If
\begin{align}
N_H\ge
\frac{10B_\phi^4}{\kappa_{\rm inn}^2}
\log\frac{s}{\delta_0},
\label{eq:block-excitation-size}
\end{align}
then, conditional on $\F_r$,
\begin{align}
\Prob\!\left(
\lambda_{\min}(G_r)\ge\kappa_Gn_H,
\quad\operatorname{tr}(G_r)\le n_HB_\phi^2
\ \middle|\ \F_r\right)
\ge1-\delta_0.
\label{eq:restart-gram}
\end{align}
The same event implies
\begin{align}
\lambda_{\min}\!\left(\sum_{t\in\I_r}\psi_t\psi_t'\right)
\ge\kappa_Gn_H.
\label{eq:psi-principal}
\end{align}
\end{proposition}

\begin{proof}
Condition on $\F_r$.  For $X_k:=\phi_{t_k}\phi_{t_k}'$, use the stopped
filtration $\mathcal H_0=\F_r$ and
$\mathcal H_k=\F_{t_k}=\F_{r+kL}$.  Here $t_1-L=r$, and
$t_k-L=t_{k-1}$ for $k\ge2$, so \eqref{eq:conditional-excitation} gives
\[
A_k:=\E[X_k\mid\mathcal H_{k-1}]
\succeq\kappa_{\rm inn}I_s.
\]
Also $0\preceq X_k\preceq B_\phi^2I_s$.  Thus
$D_k:=X_k-A_k$ is a self-adjoint matrix-martingale difference with
$\norm{D_k}\le B_\phi^2$ and predictable quadratic variation bounded by
$N_HB_\phi^4I_s$.  Moreover,
$\kappa_{\rm inn}\le B_\phi^2$ because
$A_k\preceq B_\phi^2I_s$.  The lower-tail form of the matrix Freedman
inequality \citep{tropp2011freedman}, applied to $-\sum_kD_k$, gives
\begin{align}
\Prob\!\left(
\lambda_{\min}\!\left(\sum_{k=1}^{N_H}X_k\right)
\le\frac{\kappa_{\rm inn}N_H}{2}
\ \middle|\ \F_r\right)
&\le s\exp\!\left\{-\frac{\kappa_{\rm inn}^2N_H}
 {8B_\phi^4+(4/3)B_\phi^2\kappa_{\rm inn}}\right\}
\notag\\
&\le s\exp\!\left(
-\frac{\kappa_{\rm inn}^2N_H}{10B_\phi^4}\right).
\label{eq:matrix-freedman-gram}
\end{align}
Condition \eqref{eq:block-excitation-size} makes the right-hand side at most
$\delta_0$.  Since $G_r\succeq\sum_kX_k$ and
$N_H\ge n_H/L$, the claimed eigenvalue bound follows.  The trace bound is
deterministic because $\norm{\phi_t}\le B_\phi$.

Finally, $\sum_{t\in\I_r}\psi_t\psi_t'$ is the principal block of $G_r$
obtained by deleting the current-input row and column.  Eigenvalue
interlacing, or direct restriction of the Rayleigh quotient, proves
\eqref{eq:psi-principal}.
\end{proof}

\subsection{Projected Ridge Estimation}

\begin{proposition}[Restarted parameter and gain estimate]
\label{prop:restart-estimate}
Let $r$ be a stopping time with $r+H\le T$, and suppose a single
$\theta\in\Theta$ is active at every control time
$t=r,\ldots,r+H-1$. Given $\F_r$, event
\eqref{eq:self-normalized-ridge} has probability at least $1-\delta_0$.
On its intersection with the event in
Proposition~\ref{prop:restart-gram}, the estimator defined in
\eqref{eq:ridge-theta}--\eqref{eq:reference-gain} satisfies
\begin{align}
\norm{\thetahat_r-\theta}
&\le E_\theta(H,\delta_0),
\label{eq:theta-bound}\\
\norm{\lambdabar_r-\lam(\theta)}
&\le E_\lambda(H,\delta_0)=g_H/4.
\notag
\end{align}
Moreover, the RLS initialization satisfies
\begin{align}
\underline v_HI_d\preceq V_{r+H}
\preceq(\nu+n_HB_\psi^2)I_d
\label{eq:V0-bounds}
\end{align}
in the sense of its smallest- and largest-eigenvalue bounds.
If \eqref{eq:block-excitation-size} holds, the Gram event and the two
estimates above therefore hold jointly with conditional probability at least
$1-2\delta_0$.
\end{proposition}

\begin{proof}
Let $V=V_r^\theta$ and
$S_r=\sum_{t\in\I_r}\phi_tw_{t+1}$.  Since
$y_{t+1}=\phi_t'\theta+w_{t+1}$ during the uncontaminated block,
write $G_r=\sum_{t\in\I_r}\phi_t\phi_t'$ and calculate
\begin{align*}
 \sum_{t\in\I_r}\phi_ty_{t+1}
 &=\sum_{t\in\I_r}\phi_t(\phi_t'\theta+w_{t+1})\\
 &=G_r\theta+S_r=(V-\nu I_s)\theta+S_r.
\end{align*}
Multiplication by $V^{-1}$ gives
\begin{align*}
 \widetilde\theta_r
 &=\theta+V^{-1}(S_r-\nu\theta),
\end{align*}
and therefore
\begin{align}
\widetilde\theta_r-\theta=V^{-1}(S_r-\nu\theta).
\label{eq:ridge-error-identity}
\end{align}
Assumption~\ref{ass:noise} and Hoeffding's lemma make $w_{t+1}$ conditionally
$B_w$-sub-Gaussian.  The fixed-ridge self-normalized inequality
\citep[Theorem~1]{abbasi2011improved} gives, with conditional probability at
least $1-\delta_0$,
\begin{align}
\norm{S_r}_{V^{-1}}
 \le B_w\sqrt{
 \log\frac{\det V}{\nu^s\delta_0^2}}.
\label{eq:self-normalized-ridge}
\end{align}
Here again $s=p+q$, and $\nu^s=\det(\nu I_s)$ is the determinant of the
initial ridge matrix.
Also $\norm{\nu\theta}_{V^{-1}}\le\sqrt\nu S_\Theta$.  By the
arithmetic-geometric mean inequality and the trace bound, the determinant step
can be written explicitly.  Let $\gamma_1,\ldots,\gamma_s\ge0$ be the
eigenvalues of $G_r$.  Since $V=\nu I_s+G_r$,
\begin{align*}
 \frac{\det V}{\nu^s}
 &=\prod_{i=1}^s\left(1+\frac{\gamma_i}{\nu}\right)\\
 &\le\left(
   \frac1s\sum_{i=1}^s
       \left(1+\frac{\gamma_i}{\nu}\right)
   \right)^{s}
 =\left(1+\frac{\operatorname{tr}(G_r)}{s\nu}\right)^s\\
 &\le\left(1+\frac{n_HB_\phi^2}{s\nu}\right)^s.
\end{align*}
The last inequality uses
$\operatorname{tr}(G_r)=\sum_{t\in\I_r}\norm{\phi_t}^2
\le n_HB_\phi^2$.  Taking logarithms and using $n_H\le H$ yields
\begin{align}
\log\frac{\det V}{\nu^s}
 \le s\log\!\left(1+\frac{n_HB_\phi^2}{s\nu}\right)
 \le s\log\!\left(1+\frac{HB_\phi^2}{s\nu}\right).
\label{eq:det-bound}
\end{align}
Equations~\eqref{eq:restart-gram},
\eqref{eq:ridge-error-identity}--\eqref{eq:det-bound} imply the
unprojected Euclidean error bound equal to one half of
\eqref{eq:Etheta}.  A nearest-point projection onto a compact set containing
$\theta$ increases Euclidean error by at most a factor two, proving
\eqref{eq:theta-bound}.  If $\Theta$ is closed and convex, projection is
nonexpansive and the factor two may be removed.

The gain bound follows from \eqref{eq:gain-lipschitz}.  The lower eigenvalue
in \eqref{eq:V0-bounds} follows from
\eqref{eq:psi-principal}; the upper eigenvalue follows from
$\sum\psi_t\psi_t'\preceq(\sum\norm{\psi_t}^2)I_d$.
\end{proof}

The linear Gram bound \eqref{eq:restart-gram} is the excitation estimate
used in Appendix~\ref{app:tuning} to obtain
$H=O(\log((T+1)/\delta))$.


\section{THE RLS POTENTIAL ARGUMENT}
\label{app:rls}

We establish the potential estimates used in
Theorem~\ref{thm:segment-regret}.  Consider its unchanged exploration block
and exploitation interval, and condition on an $\F_{r+H}$ history on which
the Gram and estimation events hold.  We use the constants in
\eqref{eq:regret-constants}, set $t_0=r+H$, and let $t_1\ge t_0$ denote
the final control time of an exploitation prefix.

\subsection{Control Error and Prediction Error}

Let $\lambda$ be the true gain on the unchanged segment and let
$\overline\lambda$ be the exploration reference.  For a gate of radius
$g>0$, the recursive proposal is accepted when
$|(\widehat\lambda_t-\overline\lambda)'\psi_t|\le g\norm{\psi_t}$ and the
reference proposal is used otherwise.  Suppose
\begin{align}
\norm{\overline\lambda-\lambda}\le\epsilon,\qquad
g>\epsilon.                                           \notag
\end{align}
For $a_t=(\widehat\lambda_t-\lambda)'\psi_t$, the unclipped gated action
error $\widetilde e_t=z_t-\lambda'\psi_t$ obeys
\begin{align}
\widetilde e_t^2\le K_g a_t^2,\qquad
K_g:=\max\left\{1,\left(\frac{\epsilon}{g-\epsilon}\right)^2\right\}.
                                                               \label{eq:generic-gate-charge}
\end{align}
If $\psi_t=0$, then $a_t=\widetilde e_t=0$ and the claim is immediate.  If
$\psi_t\ne0$, then on the accepted branch $\widetilde e_t=a_t$.  On the
fallback branch,
\begin{align*}
|a_t|
&\ge |(\widehat\lambda_t-\overline\lambda)'\psi_t|
    -|(\overline\lambda-\lambda)'\psi_t|\\
&>(g-\epsilon)\norm{\psi_t},
\end{align*}
while $|\widetilde e_t|\le\epsilon\norm{\psi_t}$.  Since $g-\epsilon>0$,
division gives
\begin{align*}
|\widetilde e_t|
 \le\frac{\epsilon}{g-\epsilon}|a_t|.
\end{align*}
Squaring this inequality and taking the larger constant over the two branches
proves \eqref{eq:generic-gate-charge}.  If the ideal action lies in the
clipping interval, nonexpansiveness preserves
\eqref{eq:generic-gate-charge}.  In the algorithm,
$g=g_H$ and $\epsilon=E_\lambda=g_H/4$.  Therefore
$\epsilon/(g-\epsilon)=1/3$ and $K_g=1$.

\subsection{Exact Regression and Potential Identity}

On the segment,
\begin{align}
y_{t+1}=b_\star(u_t-\lambda'\psi_t)+w_{t+1}.         \label{eq:rls-fixed-system}
\end{align}
Let $\gamma=b_\star/\widehat b$, $\alpha=1-\gamma$, and
$e_t=u_t-\lambda'\psi_t$.  The pseudo-response is exactly
\begin{align}
x_{t+1}^\lambda
 &=u_t-y_{t+1}/\widehat b\notag\\
 &=\lambda'\psi_t+\alpha e_t+\eta_{t+1},
 &\eta_{t+1}&=-w_{t+1}/\widehat b.                   \notag
\end{align}
Because $\widehat\theta_r\in\Theta$ by projection and hence
$\widehat b\ge b_{\min}$ by \eqref{eq:b-bounds},
$|\eta_{t+1}|\le B_\eta=B_w/b_{\min}$ and Hoeffding's lemma gives
\begin{align}
\E[\exp(\xi\eta_{t+1})\mid\F_t]
 \le\exp(\xi^2B_\eta^2/2),\qquad \xi\in\R.           \label{eq:eta-mgf}
\end{align}

Set $d_t=\widehat\lambda_t-\lambda$,
$a_t=d_t'\psi_t$, $Q_t=d_t'V_td_t$, and
\begin{align}
\ell_t:=\psi_t'V_{t+1}^{-1}\psi_t.                  \notag
\end{align}
The update \eqref{eq:rls-update} becomes
\begin{align}
d_{t+1}=d_t+V_{t+1}^{-1}\psi_t
 (\alpha e_t+\eta_{t+1}-a_t).                        \label{eq:d-update}
\end{align}
Using $V_{t+1}=V_t+\psi_t\psi_t'$ in
\eqref{eq:d-update} and expanding gives
\begin{align}
Q_{t+1}-Q_t
 =-a_t^2+2a_t(\alpha e_t+\eta_{t+1})
  +\ell_t(\alpha e_t+\eta_{t+1}-a_t)^2.              \label{eq:exact-Q}
\end{align}
On the exploration event,
\begin{align}
|\alpha|
 =\frac{|\widehat b-b_\star|}{\widehat b}
 \le\frac{E_\theta}{b_{\min}}=\overline\alpha.     \notag
\end{align}
Also $V_{t+1}\succeq V_{r+H}\succeq \underline v_HI$, so
\begin{align}
0\le\ell_t\le\frac{\norm{\psi_t}^2}{\underline v_H}
 \le\overline\ell.                                  \notag
\end{align}
By $e_t^2\le a_t^2$ and $(x+y)^2\le2x^2+2y^2$,
\begin{align}
2\alpha a_te_t&\le2\overline\alpha a_t^2,\notag\\
\ell_t(\alpha e_t+\eta_{t+1}-a_t)^2
&\le2\ell_t(1+\overline\alpha)^2a_t^2
  +2\ell_t\eta_{t+1}^2.                             \notag
\end{align}
Substitution into \eqref{eq:exact-Q} proves
\eqref{eq:potential-bound}.

\subsection{Time-Uniform Noise Bound}

For any fixed $\xi>0$, \eqref{eq:eta-mgf} and predictability of $a_t$ imply
that the prefix-indexed process
\begin{align}
\exp\left\{
\xi\sum_{t=t_0}^{t_1}a_t\eta_{t+1}
-\frac{\xi^2B_\eta^2}{2}\sum_{t=t_0}^{t_1}a_t^2
\right\}                                             \notag
\end{align}
is a nonnegative supermartingale starting at one, conditionally on
$\F_{t_0}$.  Its one-step multiplicative increment satisfies
\begin{align*}
 &\E\!\left[
  \exp\!\left\{\xi a_t\eta_{t+1}
       -\frac{\xi^2B_\eta^2a_t^2}{2}\right\}
  \middle|\F_t\right]\le1,
\end{align*}
because $a_t$ is $\F_t$-measurable and \eqref{eq:eta-mgf} applies with
argument $\xi a_t$.

Ville's inequality \citep{ville1939etude,howard2020timeuniform} says that,
with probability at least $1-\delta_0$, this process never exceeds
$1/\delta_0$.  Taking logarithms on that event and dividing by the fixed
$\xi>0$ gives, simultaneously for every exploitation prefix,
\begin{align*}
 \sum_ta_t\eta_{t+1}
 \le\frac{\xi B_\eta^2}{2}\sum_ta_t^2
    +\frac1\xi\log\frac1{\delta_0}.
\end{align*}
Choosing the deterministic value $\xi=c_0/(2B_\eta^2)$ gives
\begin{align}
\sum_ta_t\eta_{t+1}
 \le\frac{c_0}{4}\sum_ta_t^2
 +\frac{2B_\eta^2}{c_0}\log(1/\delta_0).             \label{eq:absorption-full}
\end{align}

\subsection{Elliptical Potential}

For $x_t=\psi_t'V_t^{-1}\psi_t$, the Sherman--Morrison identity
$V_{t+1}^{-1}=V_t^{-1}-V_t^{-1}\psi_t\psi_t'V_t^{-1}/(1+x_t)$
and the matrix determinant lemma give
\begin{align}
\ell_t&=\frac{x_t}{1+x_t},
&\frac{\det V_{t+1}}{\det V_t}&=1+x_t.               \notag
\end{align}
Since $x/(1+x)\le\log(1+x)$ for $x\ge0$,
\begin{align}
\sum_{t=t_0}^{t_1}\ell_t
 \le\log\frac{\det V_{t_1+1}}{\det V_{t_0}}.        \notag
\end{align}
The determinants are positive since $V_{t_0}\succeq \underline v_HI\succ0$.  Set
\begin{align*}
S_n:=\sum_{t=t_0}^{t_1}\psi_t\psi_t',\qquad
A_n:=V_{t_0}^{-1/2}S_nV_{t_0}^{-1/2},\qquad
n:=t_1-t_0+1.
\end{align*}
Then $A_n\succeq0$ and
\begin{align}
\log\frac{\det V_{t_1+1}}{\det V_{t_0}}
&=\log\det(I+A_n)\notag\\
&\le d\log\!\left(1+\frac{\operatorname{tr}(A_n)}{d}\right)\notag\\
&\le d\log\!\left(1+\frac{nB_\psi^2}{d\underline v_H}\right)
=\Gamma(n).                                           \label{eq:Gamma-proof}
\end{align}
The first inequality follows by applying Jensen's inequality to
$\log(1+x)$ over the nonnegative eigenvalues of $A_n$.  For the second,
\begin{align*}
\operatorname{tr}(A_n)
 &=\sum_{t=t_0}^{t_1}\psi_t'V_{t_0}^{-1}\psi_t
 \le\frac1{\underline v_H}\sum_{t=t_0}^{t_1}\norm{\psi_t}^2
 \le\frac{nB_\psi^2}{\underline v_H}.
\end{align*}
For the standard elliptical-potential argument, see
\citet[Lemma~11]{abbasi2011improved}.

\subsection{Completion of the Segment Bound}

Sum \eqref{eq:potential-bound}.  Since $Q_{t_1+1}\ge0$,
$\eta_{t+1}^2\le B_\eta^2$, and
\eqref{eq:absorption-full}--\eqref{eq:Gamma-proof} hold,
\begin{align*}
c_0\sum_ta_t^2
&\le Q_{t_0}+2\sum_ta_t\eta_{t+1}
 +2B_\eta^2\Gamma(n)\\
&\le Q_{t_0}+\frac{c_0}{2}\sum_ta_t^2
 +\frac{4B_\eta^2}{c_0}\log(1/\delta_0)
 +2B_\eta^2\Gamma(n).
\end{align*}
Rearranging yields \eqref{eq:a-sum}.  At initialization,
\begin{align}
Q_{t_0}
&=(\overline\lambda-\lambda)'V_{t_0}
  (\overline\lambda-\lambda)\notag\\
&\le E_\lambda^2\lambda_{\max}(V_{t_0})
 \le E_\lambda^2(\nu+n_HB_\psi^2)=\overline Q_H.             \notag
\end{align}
Finally, \eqref{eq:rls-fixed-system} and the gate charging inequality give
\begin{align*}
\sum_t(y_{t+1}-w_{t+1})^2
=b_\star^2\sum_te_t^2
\le\overline b^2\sum_ta_t^2,
\end{align*}
which proves Theorem~\ref{thm:segment-regret}.

\section{DETAILED DETECTOR PROOFS}
\label{app:detection}

This appendix supplies the conditional power estimates, concentration
arguments, and calendar bounds used in Appendix~\ref{sec:detection}.
The assumptions and tuning of Appendices~\ref{sec:model}
and~\ref{sec:algorithm} remain in force.

\subsection{Finite-Horizon Power of the Auxiliary Process}

Fix a change at control time $\tau$ and write $A=A_\Delta$.
For $s-h\ge\tau$, iterating \eqref{eq:auxiliary} for $h$ steps gives
\begin{align}
\psi_s^{\rm aux}
 =A^h\psi_{s-h}^{\rm aux}
 +\sum_{j=0}^{h-1}A^je_1w_{s-j}.                     \notag
\end{align}
Taking the first coordinate and squaring, all cross terms involving distinct
disturbances have conditional expectation zero.  Thus
\begin{align}
\E[(y_s^{\rm aux})^2\mid\F_{s-h}]
 =(e_1'A^h\psi_{s-h}^{\rm aux})^2
 +\sigma_w^2\sum_{j=0}^{h-1}(e_1'A^je_1)^2.          \label{eq:h-step-power-full}
\end{align}
The stationary power is
\begin{align}
J_y(\lambda^-;\theta^+)
=\sigma_w^2\sum_{j=0}^{\infty}(e_1'A^je_1)^2.        \notag
\end{align}
Since the first term of \eqref{eq:h-step-power-full} is nonnegative,
\begin{align}
\E[(y_s^{\rm aux})^2\mid\F_{s-h}]
&\ge J_y(\lambda^-;\theta^+)
-\sigma_w^2\sum_{j=h}^{\infty}\norm{A^j}^2\notag\\
&\ge J_y(\lambda^-;\theta^+)
-\frac{\sigma_w^2C_\Delta^2\rho_\Delta^{2h}}
 {1-\rho_\Delta^2}.                                  \notag
\end{align}
Equations~\eqref{eq:power-gap} and \eqref{eq:h} yield
\eqref{eq:alt-conditional-power} for every initial state at time $s-h$.

The calculation also holds when $s$ is a finite lag-$h$ predictable random
time.  Indeed, $\{s=t\}\in\F_{t-h}$; multiplying the deterministic-time
identity by $\mathbf 1\{s=t\}$ and summing over $t$ gives
\eqref{eq:h-step-power-full} with $s$ and the stopped sigma-field
$\F_{s-h}$.

\subsection{Lag-Predictable Sampling}

Let $h\le s_1<\cdots<s_m$ be finite lag-$h$ predictable output times with
$s_{k+1}-s_k\ge h$.  Define
\begin{align}
\mathcal H_{k-1}:=\F_{s_k-h},\qquad k=1,\ldots,m,
\qquad \mathcal H_m:=\F_{s_m}.                       \notag
\end{align}
Lag predictability makes these stopped sigma-fields well defined.
The spacing condition gives $\F_{s_k}\subseteq\mathcal H_k$ for $k<m$,
and the inclusion also holds for $k=m$ by definition. Hence
$y_{s_k}^2-\E[y_{s_k}^2\mid\F_{s_k-h}]$ is
$\mathcal H_k$-measurable and has conditional mean zero given
$\mathcal H_{k-1}$.  This is the martingale structure used in
Lemma~\ref{lem:sampled-square}; it need not hold for consecutive outputs
centered at lag $h$, or for samples selected after observing their values.

\subsection{Null Windows}

Fix a potential restart $r$ with a full uncontaminated exploration block
and a potential test endpoint $e$.  In the alarm-suppressed continuation
from $r$, let $S_{r,e}$ be the corresponding block of $m$ calendar samples.
The block is null for $r$ when each sampled output is generated under the
reference parameter.  If that parameter remains active throughout
$[r,\tau)$, a block ending at output time $\tau$ is null, since
$y_\tau$ is generated at control time $\tau-1$.

Condition on an $\F_{r+H}$ history on which $\G_r^{\rm est}$ holds.
The calendar fixes the block before its outputs are observed, and
\eqref{eq:null-decomposition} applies to every stored output of a null
block.  Each $w_s^2-\sigma_w^2$ is centered and lies in an interval of
length $B_w^2$.  The Hoeffding exponential process
\citep{hoeffding1963probability,howard2020timeuniform} therefore gives,
for $x>0$,
\begin{align}
\Prob\!\left(
 \sum_{s\in S_{r,e}}(w_s^2-\sigma_w^2)\ge mx
 \ \middle|\ \F_{r+H}
 \right)
 \le\exp\!\left(-\frac{2mx^2}{B_w^4}\right).         \label{eq:hoeffding-w2}
\end{align}
For the cross term, $v_s$ is predictable and $|w_sv_s|\le B_w\eps_y$.
If $\eps_y=0$, the cross term is identically zero.  Otherwise, the
corresponding Azuma--Hoeffding exponential process
\citep{azuma1967weighted,hoeffding1963probability} gives
\begin{align}
\Prob\!\left(\sum_{s\in S_{r,e}}w_sv_s\ge mx
 \ \middle|\ \F_{r+H}\right)
 \le\exp\!\left(-\frac{mx^2}{2B_w^2\eps_y^2}\right). \label{eq:azuma-cross}
\end{align}
Taking $x=\Delta/8$ in \eqref{eq:hoeffding-w2} and
$x=\Delta/16$ in \eqref{eq:azuma-cross} gives
\eqref{eq:w2-null} and \eqref{eq:cross-null}.  The deterministic term in
\eqref{eq:null-energy} is at most $\eps_y^2\le\Delta/8$.
Outside the two deviation events, the excess empirical power is therefore
at most $3\Delta/8$, below the threshold $\Delta/2$.

If the actual test at $e$ is reached, the alarm-suppressed continuation
agrees with the implemented trajectory through that test.  The conditional
bound thus controls the joint event that the reference succeeds, the test
is reached, its block is null for that reference, and the threshold is
crossed.  There are at most $T^2$ potential restart--endpoint pairs;
applying the $\delta_d/(2T^2)$ one-block bound separately to each pair and
taking a union bound gives a horizon-wide false-alarm contribution of at
most $\delta_d/2$.  Overlapping windows require no independence, and the
argument does not condition on survival to a test.

\subsection{Trajectory Coupling}

Let $\delta_t=\psi_t-\psi_t^{\rm aux}$ and use the input-channel bound
$G_u$ from \eqref{eq:input-channel}.  On $\G_r^{\rm est}$, for control
times after the change at $\tau$ and before the next alarm or true change,
Lemma~\ref{lem:gate} gives
\begin{align}
u_t=(\lambda^-)'\psi_t+e_t,\qquad
\norm{g_u(\theta^+)e_t}\le G_u\eps_u.                \notag
\end{align}
Consequently,
\begin{align}
\delta_{t+1}=A\delta_t+g_u(\theta^+)e_t,\qquad
\delta_\tau=0.                                        \notag
\end{align}
Unrolling and using \eqref{eq:mismatch-stability},
\begin{align}
\norm{\delta_t}
&\le\sum_{j=0}^{t-\tau-1}\norm{A^j}G_u\eps_u
\le\frac{C_\Delta G_u\eps_u}{1-\rho_\Delta}=R_\Delta(H).
                                                               \notag
\end{align}
The gate bound uses feasibility of the old-gain action
$(\lambda^-)'\psi_t$ at the actual state.  The auxiliary action is
unclipped and need not be feasible.  The comparison holds up to the first
alarm or next true change, while the actual plant parameter remains the
fixed $\theta^+$ used in $A$.

\subsection{Calendar and Delay}

Consider a change at control time $\tau$ under the hypotheses of
Lemma~\ref{lem:one-change}: the cycle is still in exploitation after
observing $y_\tau$, $r+H\le\tau$, and $\tau+D\le T$.
In its alarm-suppressed continuation, stored output times form an
arithmetic grid with gap $h$. Let $s_1$ be the first grid output with
$s_1-h\ge\tau$.  The distance from $\tau+h$ to the next grid point is
less than $h$, so
\begin{align}
s_1-\tau<2h.                                          \notag
\end{align}
The $m$th grid output is $s_m=s_1+(m-1)h$, hence
\begin{align}
s_m-\tau<(m+1)h=D.                                   \notag
\end{align}
Every interval $[s_k-h,s_k)$ lies after the change, so the auxiliary
outputs at these times satisfy \eqref{eq:alt-conditional-power}.
The post-change parameter remains active throughout the block by the
hypotheses of Lemma~\ref{lem:one-change}. If any earlier block alarms,
the change has already been detected.  Otherwise, the actual and
alarm-suppressed trajectories agree through the complete post-change test
at $s_m$.  Since times are integer valued, $s_m\le\tau+D-1$;
the test is therefore available within the horizon.

For the induction in Theorem~\ref{thm:detection}, the dwell condition places
the first change after the initial exploration. If change $c<C$ is detected
by $\tau^{(c)}+D$, the next exploration ends by
$\tau^{(c)}+D+H\le\tau^{(c+1)}$. Thus each subsequent change is preceded
by a full uncontaminated reference block on the successful detection
history.

\section{GLOBAL EVENT AND REGRET DECOMPOSITION}
\label{app:global}

\noindent
We combine the restart and detector guarantees on a common event and
partition the control times to prove Theorem~\ref{thm:main}.

\subsection{Global Success Event}

\noindent
Assume $T>H$; the case $T\le H$ is handled deterministically in the proof of
Theorem~\ref{thm:main}. Record each failure at its first adapted decision
time: at the end of a reached full, unchanged exploration block if its Gram
or self-normalized event fails; at the first same-parameter exploitation
prefix violating its time-uniform RLS event; at a false detector crossing;
or at the deadline of a missed change. Ties at an exploration endpoint are
ordered as Gram, self-normalized, and then RLS. Since the true parameter
sequence is deterministic in the analysis, these decisions are measurable
with respect to the history at their decision times. Thus
\begin{align}
\mathcal E_{\le t}^{\mathrm{good}}:=\{\text{no failure has been recorded by time }t\}
\in\F_t.
\notag
\end{align}
For each potential restart $r\in\{0,\ldots,T\}$, define
\begin{align}
\mathcal C_r:=\{r\text{ is reached}\}\cap\mathcal E_{\le r}^{\mathrm{good}}
\cap\{[r,r+H)\text{ is full and has one parameter}\}.
\notag
\end{align}
The first two events are $\F_r$-measurable, and the last is fixed by the
change sequence, so $\mathcal C_r\in\F_r$.

\noindent
Let $\mathcal E_r^G$ and $\mathcal E_r^\theta$ denote the Gram and
self-normalized events. Let $\mathcal E_r^R$ be the time-uniform RLS event
on the subsequent exploitation prefix with the same parameter as in
$[r,r+H)$, ending at the first of the next true change, the next alarm
output time, and the horizon.
Set $\mathcal E_r^R=\Omega$ when this prefix is empty. The controller-failure
events are
\begin{align}
\mathcal B_r^G
&:=\mathcal C_r\cap(\mathcal E_r^G)^c,
\notag\\
\mathcal B_r^\theta
&:=\mathcal C_r\cap\mathcal E_r^G
\cap(\mathcal E_r^\theta)^c,
\notag\\
\mathcal B_r^R
&:=\mathcal C_r\cap\mathcal E_r^G\cap\mathcal E_r^\theta
\cap(\mathcal E_r^R)^c.
\notag
\end{align}
Conditional on $\F_r$, Propositions~\ref{prop:restart-gram} and
\ref{prop:restart-estimate} bound the Gram and self-normalized failure
probabilities by $\delta_0$ each. On their joint success event,
Theorem~\ref{thm:segment-regret} bounds the conditional RLS failure
probability given $\F_{r+H}$ by $\delta_0$. Multiplying by the corresponding
history indicators and applying the tower property gives
\begin{align}
\Prob(\mathcal B_r^G)&\le\delta_0,
&\Prob(\mathcal B_r^\theta)&\le\delta_0,
&\Prob(\mathcal B_r^R)&\le\delta_0.
\notag
\end{align}
Consequently, for
$\mathcal B_{\rm ctrl}:=\bigcup_{r=0}^T
(\mathcal B_r^G\cup\mathcal B_r^\theta\cup\mathcal B_r^R)$,
\begin{align}
\Prob(\mathcal B_{\rm ctrl})
\le3(T+1)\delta_0=\frac{3\delta}{32}.
\label{eq:controller-union}
\end{align}

\noindent
On $\mathcal B_{\rm ctrl}^c$, every reached full, unchanged reference whose
restart precedes the first detector failure is accurate: otherwise the first
controller failure would belong to one of the events above. Hence
$\mathcal B_{\rm ctrl}^c\cap\mathcal E_{\rm det}^c
\subseteq\mathcal E_{\rm ref}\cap\mathcal E_{\rm det}^c$, and
Theorem~\ref{thm:detection} gives
\begin{align}
\Prob(\mathcal B_{\rm ctrl}^c\cap\mathcal E_{\rm det}^c)
\le\delta_d.
\notag
\end{align}
This bounds a joint event and does not condition an early detector test
on the success of a later estimator.

\noindent
Define
$\mathcal E_{\rm good}:=\mathcal B_{\rm ctrl}^c\cap\mathcal E_{\rm det}$.
Then
\begin{align}
\Prob(\mathcal E_{\rm good})
\ge1-\frac{3\delta}{32}-\frac{\delta}{2}
=1-\frac{19\delta}{32}>1-\delta.
\notag
\end{align}
On this event, the initial reference is valid, every change is detected,
and the dwell condition ensures that each nonterminal post-alarm exploration
block is full and uncontaminated.

\subsection{Regret Decomposition}

\noindent
On $\mathcal E_{\rm good}$, let $\widehat\tau^{(c)}$ be the alarm
corresponding to change $c$. Its delay interval is
\begin{align}
\mathcal D_c:=\{\tau^{(c)},\ldots,
\widehat\tau^{(c)}-1\},\qquad |\mathcal D_c|\le D,
\qquad c=1,\ldots,C.
\notag
\end{align}
The initial and post-alarm exploration intervals are
\begin{align}
\mathcal X_0&:=\{0,\ldots,\min\{H,T\}-1\},
\notag\\
\mathcal X_c&:=\{\widehat\tau^{(c)},\ldots,
\min\{\widehat\tau^{(c)}+H,T\}-1\},
\qquad c=1,\ldots,C.
\notag
\end{align}
An interval is empty when its lower endpoint exceeds its upper endpoint.
The dwell condition makes the delay and exploration intervals disjoint,
and $|\mathcal X_c|\le H$. Their complement consists of at most $C+1$
unchanged exploitation prefixes $\mathcal P_c$, one in each true segment,
with empty prefixes allowed. Therefore
\begin{align}
\{0,\ldots,T-1\}
=\left(\bigcup_{c=1}^{C}\mathcal D_c\right)
\mathbin{\dot\cup}
\left(\bigcup_{c=0}^{C}\mathcal X_c\right)
\mathbin{\dot\cup}
\left(\bigcup_{c=0}^{C}\mathcal P_c\right).
\notag
\end{align}

\noindent
For $t\in\mathcal D_c$, the bound
$|y_{t+1}-w_{t+1}|\le B_Y+B_w$ gives cost at most $K_D$.
For $t\in\mathcal X_c$, \eqref{eq:gain-form} gives cost at most $K_X$.
Set $n_c:=|\mathcal P_c|$. Every nonempty $\mathcal P_c$ follows a full,
unchanged exploration block, so Theorem~\ref{thm:segment-regret} bounds its
regret by $R_E(n_c,\delta_0)$, uniformly over the resulting prefix length.
An empty prefix has zero regret, also bounded by $R_E(0,\delta_0)$.
Consequently,
\begin{align}
\Reg_T
&\le\sum_{c=1}^{C}K_D|\mathcal D_c|
+\sum_{c=0}^{C}K_X|\mathcal X_c|
+\sum_{c=0}^{C}R_E(n_c,\delta_0)
\notag\\
&\le CK_DD+(C+1)K_XH+(C+1)R_E(T,\delta_0),
\notag
\end{align}
which is \eqref{eq:main-bound}.

\section{TUNING FEASIBILITY AND ORDERS}
\label{app:tuning}

\noindent
We verify that the exploration and detector requirements in
Appendix~\ref{sec:algorithm} admit the logarithmic choices used in
Theorem~\ref{thm:main}. The asymptotic statements keep the model orders
and the model, noise, initialization, exploration, ridge, stability,
clipping, and gap bounds fixed.

\subsection{Exploration Length}

\noindent
Equation~\eqref{eq:tracking-errors} gives
\begin{align}
\eps_y=5\overline bL_\lambda B_\psi E_\theta(H,\delta_0).
\notag
\end{align}
It is sufficient to require
\begin{align}
E_\theta(H,\delta_0)\le\epsilon_\star,
\notag
\end{align}
where
\begin{align}
\epsilon_\star:=\min\Bigg\{
&1,\frac{b_{\min}}{16},
\frac{\sqrt{\Delta/8}}
{5\overline bL_\lambda B_\psi},
\frac{\Delta(1-\rho_\Delta)}
{80B_oC_\Delta G_uL_\lambda B_\psi},
\frac{B_o}{5\overline bL_\lambda B_\psi}
\Bigg\}.
\notag
\end{align}
The four accuracy terms enforce the leading-coefficient bound, the
null-bias margin, the coupling margin, and $\eps_y\le B_o$, respectively.
The cap at one ensures $\log(1/\epsilon_\star)\ge0$.
All denominators are positive under the standing assumptions, including
$L_\lambda>0$ from Assumption~\ref{ass:model}.

\noindent
The deterministic bound \eqref{eq:Etheta} satisfies
\begin{align}
E_\theta(H,\delta_0)
\le\frac{K_0\{1+\sqrt{\log(1/\delta_0)+s\log(1+K_1H)}\}}
{\sqrt{H-r_0}}
\label{eq:Etheta-order}
\end{align}
for constants $K_0,K_1$ independent of $H,T,\delta$. Thus
\begin{align}
\mathcal H:=\Bigg\{H\in\mathbb N:H>r_0,\quad
&N_H\ge
\frac{10B_\phi^4}{\kappa_{\rm inn}^2}
\log\frac{s}{\delta_0},
\quad E_\theta(H,\delta_0)\le\epsilon_\star,
\notag\\
&\frac{B_\psi^2}{\nu+\kappa_Gn_H}\le\frac1{16}
\Bigg\}
\notag
\end{align}
is nonempty: as $H\to\infty$, $N_H$ grows linearly, while
$E_\theta(H,\delta_0)$ and the leverage bound tend to zero.
Every member satisfies \eqref{eq:block-excitation-size} and
\eqref{eq:H-rls}--\eqref{eq:H-detector}. The smallest feasible integer
selected in Appendix~\ref{sec:algorithm} is therefore at most
$\min\mathcal H$.

\noindent
For $a,b\ge1$, the condition $x\ge2a\log(2ab)$ implies
\begin{align}
x\ge a\log(bx).
\label{eq:log-inversion}
\end{align}
Indeed, setting $z=x/(2a)$ gives
$\log(bx)=\log(2ab)+\log z\le\log(2ab)+z\le2z=x/a$.
Squaring \eqref{eq:Etheta-order}, restricting to $H\ge2r_0$, and applying
\eqref{eq:log-inversion} shows that the accuracy condition is satisfied
with
\begin{align}
H=O\!\left(
\epsilon_\star^{-2}
\{\log(1/\delta_0)+s\log(1/\epsilon_\star)+1\}
\right).
\notag
\end{align}
The Gram-size and leverage conditions require only
$O(\log(1/\delta_0)+1)$ additional lower bounds on $H$ and are absorbed
in this order. Since $\delta_0=\delta/[32(T+1)]$,
\begin{align}
H=O(\log((T+1)/\delta)).
\label{eq:H-order}
\end{align}

\subsection{Detector Spacing and Delay}

\noindent
Since $\rho_\Delta\in(0,1)$, the spacing in \eqref{eq:h} is finite.
Writing $\log_+(x):=\max\{0,\log x\}$, its explicit form is
\begin{align}
h=\max\left\{1,
\left\lceil
\frac{\log_+\!\left(
\dfrac{8\sigma_w^2C_\Delta^2}
{(1-\rho_\Delta^2)\Delta}\right)}
{2\log(1/\rho_\Delta)}
\right\rceil\right\}.
\notag
\end{align}
Thus $h$ is independent of $T$. Equation~\eqref{eq:mD} gives
\begin{align}
m=O\!\left(\frac{B_o^4}{\Delta^2}\log\frac{T+1}{\delta}\right),
\qquad
D=O\!\left(\frac{hB_o^4}{\Delta^2}\log\frac{T+1}{\delta}\right).
\label{eq:D-order}
\end{align}

\subsection{Regret Constants}

\noindent
Under \eqref{eq:H-rls},
\begin{align}
c_0
&\ge1-\frac18-\frac18\left(\frac{17}{16}\right)^2
>0.73.
\notag
\end{align}
Equation~\eqref{eq:Etheta} also yields
\begin{align}
\overline Q_H
&=E_\lambda^2(\nu+n_HB_\psi^2)
=O\!\left(\log(1/\delta_0)+s\log(1+H)\right)
\notag\\
&=O(\log((T+1)/\delta)),
\notag
\end{align}
while
\begin{align}
\Gamma(T)
&=d\log\!\left(1+\frac{TB_\psi^2}{d\underline v_H}\right)
=O(d\log T).
\notag
\end{align}
Therefore $R_E(T,\delta_0)=O(\log((T+1)/\delta))$.
Substituting this bound together with \eqref{eq:H-order} and
\eqref{eq:D-order} into \eqref{eq:main-bound} proves
\eqref{eq:main-rate}.

\section{RELAXATIONS OF CLIPPING AND MISMATCH STABILITY}
\label{app:relaxations}

\noindent
This appendix establishes the feasibility and detection results used in
Sections~\ref{sec:a-nonclipping} and~\ref{sec:a-relaxations}.
The extensions use either eventual nonclipping or own-regime feasibility
with a boundary alarm.

\subsection{Closed Clipping Feasibility}

\noindent
Let $P(x)=\operatorname{clip}_{[-B_u,B_u]}(x)$. Projection onto a closed
interval is nonexpansive. If $x\in[-B_u,B_u]$, including $x=\pm B_u$, then
\begin{align}
|P(z)-x|=|P(z)-P(x)|\le |z-x|.
\label{eq:closed-clip-proof}
\end{align}
Equation~\eqref{eq:closed-clip-proof} is the only consequence of the clipping
condition used in Lemma~\ref{lem:gate}. That lemma and every later detector,
RLS, and global-regret argument therefore hold under the closed feasibility
condition in Assumption~\ref{ass:clipping}; no positive margin enters a proof
constant.

\noindent
For the stable-mismatch theorem, Assumption~\ref{ass:clipping} can be
replaced by feasibility of the true reference gain along every unchanged
exploitation prefix and every complete alarm-suppressed continuation used
in the detector analysis.

\subsection{Eventual Inactivity of Clipping}
\label{sec:eventual-nonclipping}

\noindent
A quantitative steady-state slack condition yields eventual inactivity of
clipping and a burn-in time uniform over restarts. This gives the variant
of PIECE-CD in Theorem~\ref{thm:a-nonclipping}.

\noindent
Define the forced-response output bound and its corresponding regressor bound
by
\begin{align}
Y_\infty(B_u)
&:=\frac{C_y}{1-\rho_y}(B_w+B_bB_u),
\notag\\
P_\infty(B_u)
&:=\left\{Y_\infty(B_u)^2+(q-1)B_u^2\right\}^{1/2},
\notag\\
\mu_{\rm nc}
&:=B_u-\eps_u-\Lambda P_\infty(B_u).
\label{eq:nc-slack}
\end{align}
Here $\eps_u=\frac54 g_HB_\psi$, defined in
\eqref{eq:tracking-errors}, bounds the raw-proposal error on
$\G_r^{\rm est}$.
Suppose $\mu_{\rm nc}>0$. With the convention $\tau_{\rm nc}=0$ when
$\Lambda C_yB_Y=0$, put otherwise
\begin{align}
\tau_{\rm nc}
:=\max\!\left\{0,
\left\lceil
\frac{\log(\Lambda C_yB_Y/\mu_{\rm nc})}
{\log(1/\rho_y)}-H
\right\rceil\right\}.
\label{eq:tau-nc}
\end{align}

\begin{proposition}[Uniform eventual inactivity of clipping]
\label{prop:eventual-nonclipping}
Fix a reached restart $r$ whose full exploration block has one parameter
$\theta^-$, and suppose $\G_r^{\rm est}$ holds. At every exploitation time
before the next restart satisfying
\begin{align*}
t\ge r+H+\tau_{\rm nc},
\end{align*}
the raw gated proposal belongs to the clipping interval:
\begin{align}
|z_t|\le B_u,
\qquad u_t=z_t.
\label{eq:nc-conclusion}
\end{align}
The conclusion remains valid after a plant change and does not require the
old-gain/new-plant mismatch matrix to be Schur.
\end{proposition}

\begin{proof}
The implemented input is clipped at every time, so $|u_t|\le B_u$ regardless
of whether the conclusion has already become true. Starting the state
recursion \eqref{eq:Y-state} at $r$, uniform switching stability therefore
gives, for every $t\ge r$ and for any admissible changes after $r$,
\begin{align*}
\norm{Y_t}
&\le C_y\rho_y^{t-r}\norm{Y_r}+Y_\infty(B_u)\\
&\le C_y\rho_y^{t-r}B_Y+Y_\infty(B_u).
\end{align*}
The $q-1$ stored inputs in $\psi_t$ each have magnitude at most $B_u$.
Consequently,
\begin{align*}
\norm{\psi_t}
&\le\left[
\{Y_\infty(B_u)+C_y\rho_y^{t-r}B_Y\}^2
+(q-1)B_u^2
\right]^{1/2}\\
&\le P_\infty(B_u)+C_y\rho_y^{t-r}B_Y.
\end{align*}
On $\G_r^{\rm est}$, the raw (pre-clipping) part of the gate calculation in
Lemma~\ref{lem:gate} does not use feasibility. On the recursive branch its
distance from $(\lambda^-)^\trans\psi_t$ is at most
$(g_H+g_H/4)\norm{\psi_t}$, and on the reference branch it is at most
$(g_H/4)\norm{\psi_t}$. Hence, on either branch and also after a change,
\begin{align*}
|z_t-(\lambda^-)^\trans\psi_t|\le\eps_u.
\end{align*}
Since $\norm{\lambda^-}\le\Lambda$, the preceding two displays imply
\begin{align*}
|z_t|
&\le\eps_u+\Lambda P_\infty(B_u)
+\Lambda C_yB_Y\rho_y^{t-r}\\
&=B_u-\mu_{\rm nc}
+\Lambda C_yB_Y\rho_y^{t-r}.
\end{align*}
Definition \eqref{eq:tau-nc} makes the last term at most $\mu_{\rm nc}$
whenever $t-r\ge H+\tau_{\rm nc}$, proving
\eqref{eq:nc-conclusion}. It also gives
$|(\lambda^-)^\trans\psi_t|\le B_u-\eps_u$.
The calculation is pathwise on the reference event and uniform over
admissible switching sequences.
\end{proof}

\noindent
Proposition~\ref{prop:restart-estimate} and the restart union in
\eqref{eq:reference-event} give
$\Prob(\mathcal E_{\rm ref})\ge1-2(T+1)\delta_0=1-\delta/16$.
On this event, the conclusion holds simultaneously for every reached full,
unchanged exploration block whose restart precedes the first detector failure.

\noindent
Both $P_\infty$ and $\eps_u$ depend on $B_u$, so a positive
$\mu_{\rm nc}$ is not guaranteed by increasing the clipping threshold.
An explicit sufficient condition is obtained by setting
\begin{align*}
a_0:=\frac{C_yB_w}{1-\rho_y},\qquad
a_1:=\frac{C_yB_b}{1-\rho_y},\qquad
\kappa_u:=1-\Lambda\{a_1+\sqrt{q-1}\}.
\end{align*}
If $\kappa_u>0$, choose
\begin{align}
B_u>\frac{2\Lambda a_0}{\kappa_u}
\quad\text{and then choose $H$ so that}\quad
\eps_u\le\frac{\kappa_uB_u}{2}.
\label{eq:explicit-Bu-rule}
\end{align}
Indeed,
$P_\infty(B_u)\le a_0+(a_1+\sqrt{q-1})B_u$, so
\eqref{eq:explicit-Bu-rule} implies $\mu_{\rm nc}>0$. For fixed $B_u$ the
required inequality for $\eps_u$ is achieved by a sufficiently large feasible
$H$ (when the horizon permits) under the exploration bound in
\eqref{eq:Etheta}. A direct sufficient check for immediate nonclipping is
$(\Lambda+g_H)B_\psi\le B_u$: since
$\norm{\overline\lambda_r}\le\Lambda$, every raw exploitation proposal
then lies in the clipping interval and no burn-in is needed.

\noindent
Bounded noise alone does not ensure eventual nonclipping. Consider the
known scalar system
\begin{align*}
y_{t+1}=\tfrac12y_t+u_t+w_{t+1},\qquad
B_u=\tfrac14,\qquad
\Prob(w_t=1)=\Prob(w_t=-1)=\tfrac12.
\end{align*}
The exact minimum-variance gain is $\lambda=-1/2$. If the exact
minimum-variance law were unclipped from some finite time onward, then
$y_{t+1}=w_{t+1}$ from that time onward, and every following raw action would
have magnitude $1/2>B_u$, a contradiction. Thus no finite $\tau_{\rm nc}$
exists for this example under that law.

\paragraph{Burn-in variant.}
Before the burn-in ends, clipping may invalidate gate charging and the
detector estimates. After each full exploration block, apply reference
control for $\tau_{\rm nc}$ rounds with detection and RLS updates disabled
and the detector buffer empty. Use
$z_t=\overline\lambda_r^\trans\psi_t$ and
$u_t=\operatorname{clip}_{[-B_u,B_u]}(z_t)$. At
$t_0=r+H+\tau_{\rm nc}$ set
$V_{t_0}=\nu I_d+\sum_{t\in\I_r}\psi_t\psi_t^\trans$ and
$\widehat\lambda_{t_0}=\overline\lambda_r$, then use the shifted output grid
$\mathcal S_r^{\rm nc}=\{t_0+h+jh:j\ge0\}$. Truncate exploration or
burn-in at $T$. Require the first change no
earlier than $H+\tau_{\rm nc}$, successive changes to be separated by at
least $D+H+\tau_{\rm nc}$, and the terminal segment to have length at least
$D$. From $t_0$ onward clipping is inactive by
\eqref{eq:nc-conclusion}, so the baseline gate, false-alarm, coupling, and RLS
proofs apply without Assumption~\ref{ass:clipping}. The regret bound gains
only the deterministic term
$(C+1)K_D\tau_{\rm nc}$, and the logarithmic order is unchanged whenever
$\tau_{\rm nc}=O(\log((T+1)/\delta))$ (in particular, when it is a fixed
problem-dependent constant).

\subsection{Own-Regime Feasibility and a Boundary Alarm}
\label{sec:Own-Regime Feasibility}

\noindent
A boundary alarm controls the post-change action error under feasibility
of each gain only on regressors generated by its own plant.

\noindent
Call a regressor an \emph{own-regime successor} for $\theta$ if it can occur
after at least $r_0$ consecutive controls generated by plant $\theta$, starting
from any regressor with norm at most $B_\psi$, under disturbances bounded by
$B_w$ and arbitrary inputs in $[-B_u,B_u]$. This definition includes the
successors produced during exploration and the terminal state at a segment
boundary; it is not restricted to states generated by the ideal policy.

\begin{assumption}[Closed own-regime clipping feasibility]
\label{ass:own-regime-feasibility}
For every $\theta\in\Theta$ and every own-regime successor $\psi$ for
$\theta$,
\begin{align}
|\lambda(\theta)^\trans\psi|
\le B_u.
\label{eq:own-regime-feasibility}
\end{align}
Equality is allowed, and no positive clipping margin is required.
\end{assumption}

\noindent
Assumption~\ref{ass:own-regime-feasibility} is an alternative feasibility
condition: it checks each gain only on its own-regime successors and does not
require the old gain to remain feasible on states produced by the new regime.
Because the successor definition permits any norm-bounded starting regressor,
including unreachable states, this condition is not asserted to follow from
Assumption~\ref{ass:clipping}.

\noindent
Define the deterministic reference uncertainty and safe tracking
bounds
\begin{align}
\zeta_H&:=E_\lambda(H,\delta_0)B_\psi,
\notag\\
\overline\eps_u
&:=\eps_u+2\zeta_H
=7L_\lambda B_\psi E_\theta(H,\delta_0),
\notag\\
\overline\eps_y&:=\overline b\,\overline\eps_u.
\label{eq:inflated-tracking-bounds}
\end{align}

\noindent
The branch-free union-class version of PIECE-CD adds the following test.
After an exploitation round $t$, once $y_{t+1}$ has been observed and
$\psi_{t+1}$ has been formed, raise one alarm and restart at control time
$t+1$ if
\begin{align}
|\overline\lambda_r^\trans\psi_{t+1}|
>B_u+\zeta_H.
\label{eq:boundary-alarm}
\end{align}
The strict inequality is part of the test: under the null, equality with the
upper bound is possible for discrete disturbances.
If both the boundary and energy tests cross at the same output, they count as
one alarm. Algorithm~\ref{alg:stability-agnostic} states the complete
composite procedure.

\begin{lemma}[Boundary-alarm dichotomy]
\label{lem:boundary-dichotomy}
Under Assumption~\ref{ass:own-regime-feasibility}, fix a reached cycle
whose full exploration block has parameter $\theta^-$, and suppose
$\G_r^{\rm est}$ holds. Then the boundary test cannot cross while
the reference parameter remains active. On such an
unchanged exploitation interval, with
$a_t:=(\widehat\lambda_t-\lambda(\theta^-))^\trans\psi_t$ and
$e_t:=u_t-\lambda(\theta^-)^\trans\psi_t$, one also has
$e_t^2\le a_t^2$. After a change within this exploitation cycle, at every
control time before the cycle's next alarm,
\begin{align}
\bigl(|\lambda(\theta^-)^\trans\psi_t|-B_u\bigr)_+
&\le2\zeta_H,
\label{eq:boundary-saturation}\\
|u_t-\lambda(\theta^-)^\trans\psi_t|
&\le\overline\eps_u.
\label{eq:boundary-tracking}
\end{align}
These inequalities also hold in the alarm-suppressed continuation at
control states satisfying
$|\overline\lambda_r^\trans\psi_t|\le B_u+\zeta_H$.
\end{lemma}

\begin{proof}
While the reference parameter $\theta^-$ remains active, every exploitation
regressor is an own-regime successor. On $\G_r^{\rm est}$,
$\norm{\overline\lambda_r-\lambda(\theta^-)}\le
E_\lambda(H,\delta_0)$, so \eqref{eq:own-regime-feasibility} gives
\begin{align*}
|\overline\lambda_r^\trans\psi_t|
&\le |\lambda(\theta^-)^\trans\psi_t|
+\norm{\overline\lambda_r-\lambda(\theta^-)}\norm{\psi_t}
\notag\\
&\le B_u+\zeta_H.
\end{align*}
The strict test in \eqref{eq:boundary-alarm} therefore cannot give a boundary
false alarm. Since the old ideal action is feasible on the unchanged
interval, the gate-charging and projection argument of
Lemma~\ref{lem:gate} applies verbatim and gives $e_t^2\le a_t^2$.

\noindent
For a change occurring during the exploitation cycle in the lemma, at the
first post-change control the regressor is still an own-regime successor for
$\theta^-$, so its saturation deficit is zero. At every later post-change
control time with no previous boundary alarm,
\begin{align*}
|\lambda(\theta^-)^\trans\psi_t|
&\le |\overline\lambda_r^\trans\psi_t|
+E_\lambda(H,\delta_0)B_\psi
\le B_u+2\zeta_H,
\end{align*}
which proves \eqref{eq:boundary-saturation}. The raw gated proposal satisfies
$|z_t-\lambda(\theta^-)^\trans\psi_t|\le\eps_u$. For
$P(x)=\operatorname{clip}_{[-B_u,B_u]}(x)$ and any $x,z\in\mathbb R$,
\begin{align*}
|P(z)-x|
&\le |P(z)-P(x)|+|P(x)-x|\\
&\le |z-x|+(|x|-B_u)_+.
\end{align*}
Combining this inequality with \eqref{eq:boundary-saturation} gives
\eqref{eq:boundary-tracking}. The same calculation applies pointwise in
the alarm-suppressed continuation whenever
$|\overline\lambda_r^\trans\psi_t|\le B_u+\zeta_H$.
\end{proof}

\subsection{Detection under the Union Condition}
\label{sec:ARX Non-Schur}

\noindent
The boundary tracking bound \eqref{eq:boundary-tracking} permits a common
finite-window detector for stable, directly visible, and non-Schur ARX
mismatches.

\noindent
For an admissible adjacent pair $z=(\theta^-,\theta^+)$, write
\begin{align*}
A_z:=A(\theta^-,\theta^+),
\qquad
G_h(A_z):=\sigma_w^2\sum_{j=1}^{h-1}(e_1^\trans A_z^je_1)^2.
\end{align*}
The supplied admissible adjacent-pair class $\mathcal P\subseteq\Theta^2$
has a cover $\mathcal P\subseteq\mathcal P_{\sf S}\cup\mathcal P_{\sf F}$.
The finite-window component may itself be covered by
$\mathcal P_{\sf F}\subseteq\mathcal P_{\sf D}\cup\mathcal P_{\sf A}$.
The set $\mathcal P_{\sf D}$ uses a directly supplied finite-order gap, whereas
$\mathcal P_{\sf A}$ uses only the ARX non-Schur condition
$\rho(A_z)\ge1$. These sets may overlap.
These classes specify uniform tuning bounds.

\noindent
Define the finite-order innovation visibility
\begin{align}
\mathcal V(z)
:=\sigma_w^2\sum_{j=1}^{d}
\{e_1^\trans A_z^je_1\}^2.
\label{eq:visibility-index}
\end{align}

\noindent
For $\theta\in\Theta$, define the monic old-input polynomial
\begin{align*}
\overline B_\theta(\xi)
:=\xi^{q-1}+\frac{b_2(\theta)}{b_1(\theta)}\xi^{q-2}
+\cdots+\frac{b_q(\theta)}{b_1(\theta)},
\end{align*}
with $\overline B_\theta\equiv1$ when $q=1$. Under the minimum-phase
convention in Assumption~\ref{ass:model}, every root of
$\overline B_\theta$ lies in the open unit disk. Continuity of polynomial
roots and compactness therefore give a safe uniform number
\begin{align}
r_B\ge \sup_{\theta\in\Theta}
\max\{|\beta|:\overline B_\theta(\beta)=0\},
\qquad 0\le r_B<1,
\label{eq:old-input-root-bound}
\end{align}
where the maximum over the empty set is zero. Also let
\begin{align}
R_A\ge \max\left\{1,\sup_{(\theta^-,\theta^+)\in\Theta^2}
\norm{A_{\rm cl}(\lambda(\theta^-),\theta^+)}_2\right\}<\infty,
\qquad
K_A:=\sum_{k=1}^d R_A^{d-k}
\sum_{j=0}^{k-1}\binom{d}{j}R_A^j.
\label{eq:spectral-visibility-constants}
\end{align}
These are offline model-class bounds. The uniform root bound follows from
continuity of polynomial zeros and compactness
\citep[Appendix~D, Theorem~D1, and Appendix~E]{horn2013matrix}.

\paragraph{Finite-order visibility of non-Schur mismatches.}
Define
\begin{align}
\underline\Delta_{\sf A}
:=\sigma_w^2\left\{
\frac{(1-r_B)^{q-1}}{K_A}
\right\}^2>0.
\label{eq:spectral-derived-gap}
\end{align}
Then every ARX mismatch matrix $A_z$ satisfying $\rho(A_z)\ge1$ obeys
\begin{align}
\mathcal V(z)\ge\underline\Delta_{\sf A}.
\label{eq:spectral-implies-visibility}
\end{align}
This implication uses the ARX and minimum-phase structure; it is false for a
general square matrix.

\noindent
To prove it, put $h_j=e_1^\trans A_z^je_1$ and
\begin{align*}
D_z(\xi)&:=\det(\xi I-A_z)
=\xi^d+\sum_{k=1}^d d_k\xi^{d-k},\\
Q_-(\xi)&:=\xi^p\overline B_{\theta^-}(\xi)
=\xi^d+\sum_{k=1}^d q_k\xi^{d-k}.
\end{align*}
A cofactor expansion using the shift rows of $A_z$ gives
\begin{align}
e_1^\trans\operatorname{adj}(\xi I-A_z)e_1
=\xi^{p-1}\overline B_{\theta^-}(\xi).
\label{eq:arx-cofactor-identity}
\end{align}
Indeed, deleting the first row and column leaves a block-lower-triangular
minor: its output-shift block has determinant $\xi^{p-1}$ and its old-input
companion block has determinant $\overline B_{\theta^-}(\xi)$. The
adjugate formula and the Laurent expansion at infinity then give the exact
disturbance-to-output identity, for $|\xi|>\rho(A_z)$ (equivalently, as a
formal Laurent identity at infinity),
\begin{align}
\xi e_1^\trans(\xi I-A_z)^{-1}e_1
=1+\sum_{j\ge1}h_j\xi^{-j}
=\frac{Q_-(\xi)}{D_z(\xi)}.
\label{eq:arx-resolvent-identity}
\end{align}
Comparing the coefficient of $\xi^{d-k}$ in
$D_z(\xi)(1+\sum_{j\ge1}h_j\xi^{-j})=Q_-(\xi)$ gives, with $d_0=1$,
\begin{align}
q_k-d_k=\sum_{i=1}^k d_{k-i}h_i,
\qquad k=1,\ldots,d.
\label{eq:markov-coefficient-identity}
\end{align}
Let $\eta:=\max_{1\le j\le d}|h_j|$ and choose an eigenvalue $\lambda_\star$
with $|\lambda_\star|=\rho(A_z)\ge1$. Since
$|d_j|\le \binom{d}{j}R_A^j$ and
$|\lambda_\star|\le R_A$ (the spectral radius is bounded by every induced
matrix norm \citep[Theorem~5.6.9]{horn2013matrix}), equations
\eqref{eq:markov-coefficient-identity} and $D_z(\lambda_\star)=0$ imply
\begin{align*}
|Q_-(\lambda_\star)|\le \eta K_A.
\end{align*}
On the other hand, $Q_-$ has $p$ zeros at zero and its other $q-1$ zeros
have modulus at most $r_B$, so
\begin{align*}
|Q_-(\lambda_\star)|
\ge |\lambda_\star|^p(|\lambda_\star|-r_B)^{q-1}
\ge(1-r_B)^{q-1}.
\end{align*}
Combining the last two displays and using
$\mathcal V(z)\ge\sigma_w^2\eta^2$ proves
\eqref{eq:spectral-implies-visibility}.
If a stronger class-wide margin $\rho(A_z)\ge c_a>1$ is available, the same
proof gives the sharper constant
$\underline\Delta_{\sf A}(c_a):=\sigma_w^2
\{c_a^p(c_a-r_B)^{q-1}/K_A\}^2$, which may replace
$\underline\Delta_{\sf A}$ in the tuning. Such a margin is not required for
the theorem below.

\begin{assumption}[Union stable/direct/non-Schur detectability]
\label{ass:finite-lag-detectability}
Every realized adjacent pair
$z_c=(\theta^{(c-1)},\theta^{(c)})$, $c=1,\ldots,C$, belongs to
$\mathcal P$. For a nonempty stable component there are supplied constants
$C_{\sf S}\ge1$, $\rho_{\sf S}\in(0,1)$, and
$\underline\Delta_{\sf S}>0$; for a nonempty direct-visibility component
there is a supplied $\underline\Delta_{\sf D}>0$; and the possibly empty ARX
non-Schur component satisfies
\begin{align}
z\in\mathcal P_{\sf S}&\Longrightarrow
\norm{A_z^k}\le C_{\sf S}\rho_{\sf S}^k\quad(k\ge0),\qquad
J_y(\lambda(\theta^-);\theta^+)-\sigma_w^2
\ge\underline\Delta_{\sf S},
\notag\\
z\in\mathcal P_{\sf D}&\Longrightarrow
\mathcal V(z)\ge\underline\Delta_{\sf D},
\notag\\
z\in\mathcal P_{\sf A}&\Longrightarrow
\rho(A_z)\ge1.
\label{eq:finite-lag-gap}
\end{align}
Each admitted pair satisfies at least one component condition. A stable
pair requires no separate finite-order gap; for a non-Schur pair, that gap
is supplied by \eqref{eq:spectral-derived-gap}.
\end{assumption}

\noindent
Set $\Gamma_{\sf D}=\underline\Delta_{\sf D}$ when
$\mathcal P_{\sf D}\ne\varnothing$ and $\Gamma_{\sf D}=+\infty$ otherwise.
Set $\Gamma_{\sf A}=\underline\Delta_{\sf A}$ when
$\mathcal P_{\sf A}\ne\varnothing$ and $\Gamma_{\sf A}=+\infty$ otherwise.
If $\mathcal P_{\sf F}\ne\varnothing$, define
\begin{align}
\underline\Delta_{\sf F}:=\min\{\Gamma_{\sf D},\Gamma_{\sf A}\}>0.
\label{eq:aggregate-finite-gap}
\end{align}
Equations \eqref{eq:visibility-index} and
\eqref{eq:spectral-implies-visibility} then give the derived common bound
\begin{align}
z\in\mathcal P_{\sf F}\Longrightarrow
\mathcal V(z)\ge\underline\Delta_{\sf F}.
\label{eq:derived-finite-gap}
\end{align}
This is the finite-visible condition used in
Section~\ref{sec:a-relaxations}.

\noindent
For a nonempty stable component, define
\begin{align*}
h_{\sf S}
&:=\min\left\{k\ge1:
\frac{\sigma_w^2C_{\sf S}^2\rho_{\sf S}^{2k}}
{1-\rho_{\sf S}^2}
\le\frac{\underline\Delta_{\sf S}}8\right\},
&\Gamma_{\sf S}&:=\frac{7\underline\Delta_{\sf S}}8.
\end{align*}
If $\mathcal P_{\sf S}=\varnothing$, set
$(h_{\sf S},\Gamma_{\sf S})=(1,+\infty)$. If
$\mathcal P_{\sf F}\ne\varnothing$, set
$(h_{\sf F},\Gamma_{\sf F})=(d+1,\underline\Delta_{\sf F})$; otherwise set
$(h_{\sf F},\Gamma_{\sf F})=(1,+\infty)$. If
$\mathcal P=\varnothing$, set $h_{\sf U}=1$ and choose any fixed finite
$\underline\Delta_{\sf U}>0$ for null calibration; every missed-change
statement is then vacuous. Otherwise at least one cover component is
nonempty, and set
\begin{align}
h_{\sf U}&:=\max\{h_{\sf S},h_{\sf F}\},
&\underline\Delta_{\sf U}&:=\min\{\Gamma_{\sf S},\Gamma_{\sf F}\}>0.
\label{eq:instance-visibility-gap}
\end{align}
The resulting common finite-horizon power separation is
\begin{align}
G_{h_{\sf U}}(A_z)\ge\underline\Delta_{\sf U},
\qquad z\in\mathcal P.
\label{eq:finite-lag-gap-bound}
\end{align}
Indeed, if $z\in\mathcal P_{\sf S}$, then
\begin{align*}
G_{h_{\sf U}}(A_z)
&=J_y(\lambda(\theta^-);\theta^+)-\sigma_w^2
-\sigma_w^2\sum_{j=h_{\sf U}}^\infty
(e_1^\trans A_z^je_1)^2\\
&\ge\underline\Delta_{\sf S}
-\frac{\sigma_w^2C_{\sf S}^2\rho_{\sf S}^{2h_{\sf U}}}
{1-\rho_{\sf S}^2}
\ge\frac{7\underline\Delta_{\sf S}}8
\ge\underline\Delta_{\sf U}.
\end{align*}
If $z\in\mathcal P_{\sf F}$, then $h_{\sf U}\ge d+1$ and the nonnegative
summands give
$G_{h_{\sf U}}(A_z)\ge\mathcal V(z)
\ge\underline\Delta_{\sf F}\ge\underline\Delta_{\sf U}$.
This proves \eqref{eq:finite-lag-gap-bound} without identifying which route
holds for the realized change.

\noindent
The first $d$ impulse coefficients exhaust finite-order visibility by
Cayley--Hamilton \citep[Theorem~2.4.3.2]{horn2013matrix}. Let
$h_j=e_1^\trans A^je_1$, so $h_0=1$, and write the characteristic polynomial
as $p_A(z)=z^d+\kappa_{d-1}z^{d-1}+\cdots+\kappa_0$. If
$h_1=\cdots=h_d=0$, sandwiching $p_A(A)=0$ between $e_1^\trans$ and $e_1$
gives $\kappa_0=0$. Sandwiching $A^kp_A(A)=0$ for $k\ge1$ and using strong
induction then gives $h_j=0$ for every $j\ge1$. Thus every change visible
through a noise-driven impulse coefficient is visible by order $d$.

\noindent
If $\mathcal P=\varnothing$, set
$L_{0,{\sf U}}=L_{w,{\sf U}}=L_{u,{\sf U}}=0$. Otherwise, the same
model-class bound $R_A$ gives the explicit finite-horizon choices
\begin{align}
L_{0,{\sf U}}&:=R_A^{h_{\sf U}},
\notag\\
L_{w,{\sf U}}&:=\sum_{j=0}^{h_{\sf U}-1}R_A^j,
&L_{u,{\sf U}}&:=G_u\sum_{j=0}^{h_{\sf U}-1}R_A^j.
\label{eq:finite-lag-constants}
\end{align}
Indeed, $|x^\trans A^jy|\le\norm{x}_2\norm{A}_2^j\norm{y}_2$.
These bounds are finite because $h_{\sf U}$ is finite, even when some
matrices in $\mathcal P_{\sf A}$ are non-Schur. Sharper supplied bounds may
replace them without changing the proof.

\noindent
Define
\begin{align}
B_{{\rm aux},{\sf U}}&:=L_{0,{\sf U}}B_\psi+L_{w,{\sf U}}B_w,
\notag\\
B_{\sf U}&:=\max\{B_Y,B_{{\rm aux},{\sf U}},B_w\},
&R_{\sf U}(H)&:=L_{u,{\sf U}}\overline\eps_u.
\label{eq:finite-lag-bounds}
\end{align}
In addition to the exploration sample-size condition and
\eqref{eq:H-rls}, choose the smallest $H$ satisfying
\begin{align}
\overline\eps_y^2&\le\frac{\underline\Delta_{\sf U}}8,
&\overline\eps_y&\le B_{\sf U},
&2B_{\sf U}R_{\sf U}(H)&\le\frac{\underline\Delta_{\sf U}}8.
\label{eq:finite-lag-H}
\end{align}
Use spacing $h_{\sf U}$, the energy threshold
$\sigma_w^2+\underline\Delta_{\sf U}/2$, and set
\begin{align}
m_{\sf U}&:=\left\lceil
\frac{512B_{\sf U}^4}{\underline\Delta_{\sf U}^2}
\log\frac{4T^2}{\delta_d}\right\rceil,
&D_{\sf U}&:=(m_{\sf U}+1)h_{\sf U}.
\label{eq:finite-lag-mD}
\end{align}
Assumption~\ref{ass:dwell} is read with $D_{\sf U}$ in place of $D$.

\noindent
Fix $z=(\theta^-,\theta^+)\in\mathcal P$ and a potential restart $r$
whose exploration reference estimates $\lambda(\theta^-)$. Continue the
clipped controller on the same restart-relative calendar with both alarms
suppressed. Its regressor remains bounded by $B_\psi$ on every path.
For a stored output time $s$ with $s-h_{\sf U}\ge r+H$, define the
$h_{\sf U}$-step auxiliary
\begin{align}
\widetilde\psi_{s-h_{\sf U}}^{(s)}&:=\psi_{s-h_{\sf U}},
\notag\\
\widetilde\psi_{t+1}^{(s)}
&=A_z\widetilde\psi_t^{(s)}+e_1w_{t+1},
\qquad s-h_{\sf U}\le t<s.
\label{eq:sample-local-auxiliary}
\end{align}
Let $\widetilde y_s^{(s)}$ be its first coordinate. This auxiliary is
defined on every continuation path, including paths on which the boundary
test would cross.

\begin{lemma}[Sample-local union power and comparison]
\label{lem:sample-local-power}
For the pair $z\in\mathcal P$ and continuation defined above, suppose
$\theta^+$ is active throughout $[s-h_{\sf U},s)$. Then
\begin{align}
|\widetilde y_s^{(s)}|&\le B_{{\rm aux},{\sf U}},
\label{eq:local-aux-bound}\\
\E[(\widetilde y_s^{(s)})^2\mid\F_{s-h_{\sf U}}]
&\ge\sigma_w^2+\underline\Delta_{\sf U}.
\label{eq:local-aux-power}
\end{align}
On $\G_r^{\rm est}$ and on the event that the boundary test has not crossed at
the states used for the controls in $[s-h_{\sf U},s)$,
\begin{align}
|y_s-\widetilde y_s^{(s)}|&\le R_{\sf U}(H),
\notag\\
|y_s^2-(\widetilde y_s^{(s)})^2|
&\le2B_{\sf U}R_{\sf U}(H)
\le\frac{\underline\Delta_{\sf U}}8.
\label{eq:local-square-comparison}
\end{align}
The conditional statements remain valid at finite lag-$h_{\sf U}$
predictable output times.
\end{lemma}

\begin{proof}
Iterating \eqref{eq:sample-local-auxiliary} gives
\begin{align*}
\widetilde\psi_s^{(s)}
=A_z^{h_{\sf U}}\psi_{s-h_{\sf U}}
+\sum_{j=0}^{h_{\sf U}-1}A_z^je_1w_{s-j}.
\end{align*}
The definitions of $L_{0,{\sf U}}$ and $L_{w,{\sf U}}$ give
\eqref{eq:local-aux-bound}. Conditional on $\F_{s-h_{\sf U}}$, the starting
state is fixed and $w_{s-h_{\sf U}+1},\ldots,w_s$ are fresh, independent, and
centered. Hence
\begin{align*}
\E[(\widetilde y_s^{(s)})^2\mid\F_{s-h_{\sf U}}]
&=(e_1^\trans A_z^{h_{\sf U}}\psi_{s-h_{\sf U}})^2
+\sigma_w^2\sum_{j=0}^{h_{\sf U}-1}(e_1^\trans A_z^je_1)^2\\
&\ge\sigma_w^2+G_{h_{\sf U}}(A_z)
\ge\sigma_w^2+\underline\Delta_{\sf U},
\end{align*}
which proves \eqref{eq:local-aux-power}.

\noindent
On the no-boundary event in the statement,
Lemma~\ref{lem:boundary-dichotomy} gives
$u_t=\lambda(\theta^-)^\trans\psi_t+e_t$ with
$|e_t|\le\overline\eps_u$ throughout $[s-h_{\sf U},s)$.
Subtracting the two state recursions and unrolling from their common state at
$s-h_{\sf U}$ gives
\begin{align*}
|y_s-\widetilde y_s^{(s)}|
&\le\overline\eps_u
\sum_{j=0}^{h_{\sf U}-1}|e_1^\trans A_z^jg_u(\theta^+)|
\le L_{u,{\sf U}}\overline\eps_u.
\end{align*}
Both output magnitudes are at most $B_{\sf U}$, so
$|x^2-z^2|\le|x-z|(|x|+|z|)$ proves
\eqref{eq:local-square-comparison}.

\noindent
For a finite lag-$h_{\sf U}$ predictable time, multiply the
deterministic-time identities by $\one\{s=t\}\in\F_{t-h_{\sf U}}$ and sum over
$t$. This gives the same relations with the stopped field
$\F_{s-h_{\sf U}}$.
\end{proof}

\begin{lemma}[One-change detection under the union condition]
\label{lem:finite-lag-one-change}
Adopt the stopping-time, active-cycle, and horizon hypotheses of
Lemma~\ref{lem:one-change}, use Algorithm~\ref{alg:stability-agnostic},
replace $D$ by $D_{\sf U}$, and suppose
Assumptions~\ref{ass:own-regime-feasibility} and
\ref{ass:finite-lag-detectability} hold. Almost surely on the corresponding
$\mathcal A_{r,\tau}$,
\begin{align}
\Prob(\widehat\tau_r>\tau+D_{\sf U}\mid\F_\tau)
\le\frac{\delta_d}{4T^2}.
\label{eq:finite-lag-one-change}
\end{align}
\end{lemma}

\begin{proof}
Condition on an arbitrary $\F_\tau$ history in $\mathcal A_{r,\tau}$ and use
the alarm-suppressed continuation just defined. The deterministic calendar
argument in Appendix~\ref{app:detection} gives stored output times
$s_1,\ldots,s_{m_{\sf U}}$ such that $s_1-h_{\sf U}\ge\tau$,
$s_{k+1}-s_k=h_{\sf U}$, and $s_{m_{\sf U}}-\tau<D_{\sf U}$.
Conditionally on $\F_\tau$, the restart-relative grid is lag-$h_{\sf U}$
predictable, $\F_\tau\subseteq\F_{s_1-h_{\sf U}}$, and the disturbances used
by the sample-local auxiliaries remain fresh.

\noindent
Apply the conditional form of Lemma~\ref{lem:sampled-square} given
$\F_\tau$ to the sample-local variables $\widetilde y_{s_k}^{(s_k)}$, using
the lagged filtration $\mathcal H_{k-1}:=\F_{s_k-h_{\sf U}}$.
Lemma~\ref{lem:sample-local-power} gives the required conditional means, and
$|\widetilde y_{s_k}^{(s_k)}|\le B_{\sf U}$. Therefore
\begin{align}
\Prob\!\left(
\frac1{m_{\sf U}}\sum_{k=1}^{m_{\sf U}}(\widetilde y_{s_k}^{(s_k)})^2
<\sigma_w^2+\frac{7\underline\Delta_{\sf U}}8
\ \middle|\ \F_\tau\right)
&\le\exp\!\left(-\frac{m_{\sf U}\underline\Delta_{\sf U}^2}
{128B_{\sf U}^4}\right)
\notag\\
&\le\frac{\delta_d}{4T^2}.
\label{eq:local-aux-tail}
\end{align}
This concentration bound is not conditioned on survival of the boundary
test.

\noindent
If either test raises an alarm at or before $s_{m_{\sf U}}$, the change
has already been detected. Otherwise, the actual and alarm-suppressed
trajectories agree through the complete block, and
\eqref{eq:local-square-comparison} holds at every sample. Outside the deviation in
\eqref{eq:local-aux-tail}, the actual empirical energy is at least
$\sigma_w^2+3\underline\Delta_{\sf U}/4$, which crosses the threshold
$\sigma_w^2+\underline\Delta_{\sf U}/2$. Consequently,
\begin{align*}
\{\widehat\tau_r>\tau+D_{\sf U}\}
\subseteq
\left\{
\frac1{m_{\sf U}}\sum_{k=1}^{m_{\sf U}}(\widetilde y_{s_k}^{(s_k)})^2
<\sigma_w^2+\frac{7\underline\Delta_{\sf U}}8
\right\}.
\end{align*}
Taking probabilities proves \eqref{eq:finite-lag-one-change}.
\end{proof}

\noindent
A smaller $\underline\Delta_{\sf U}$ increases the exploration and window
requirements in \eqref{eq:finite-lag-H}--\eqref{eq:finite-lag-mD}.
Uniform finite delay requires every admitted Schur pair to satisfy either
the stable-gap or direct-visibility condition.

\paragraph{A scalar non-Schur mismatch.}

\noindent
Take $p=q=1$, $B_u=10$, zero initial history, and Rademacher noise with
$B_w=\sigma_w=1$. Let $\Theta=\{\theta^-,\theta^+\}$ and consider the
one-change sequence $\theta^-\to\theta^+$, where
\begin{align*}
\theta^-=(-0.2,0.2),\qquad \theta^+=(0.2,1).
\end{align*}
Both open-loop coefficients have magnitude $0.2$, and the deterministic
output bound is $B_Y=11/(1-0.2)=13.75$. The interval
$[-B_Y,B_Y]$ is invariant under either plant with $|u_t|\le B_u$ and
$|w_t|\le1$. The two ideal gains are
$\lambda^-=1$ and $\lambda^+=-0.2$. Every one-step own-regime successor
satisfies
\begin{align*}
|\lambda^-y^+|&\le0.2B_Y+0.2B_u+B_w=5.75,\\
|\lambda^+y^+|&\le0.2B_Y=2.75.
\end{align*}
Thus Assumption~\ref{ass:own-regime-feasibility} holds, while the old gain
need not be feasible on every state produced by the new plant. Nevertheless,
$A(\theta^-,\theta^+)=0.2+1\cdot1=1.2$. Here $d=1$, so $h_{\sf F}=2$ and
$\mathcal V(\theta^-,\theta^+)=1.2^2=1.44$. Since $r_B=0$ here, the
non-Schur formula itself supplies the valid positive bound
$\underline\Delta_{\sf A}=1$; no separate visibility assumption is needed.

\subsection{Regret with Saturation}

\noindent
Without clipping feasibility, exploitation regret can be controlled by
the realized saturation energy. For a true gain $\lambda$, define
\begin{align}
d_t^{\mathrm{sat}}:=\bigl(|\lambda^\trans\psi_t|-B_u\bigr)_+.
\label{eq:saturation-deficit}
\end{align}
On an unchanged exploitation segment with a successful reference and
$g_H=4E_\lambda$, the unclipped gate-charging inequality from
Appendix~\ref{app:rls} and the projection distance inequality give
\begin{align}
|u_t-\lambda^\trans\psi_t|\le |a_t|+d_t^{\mathrm{sat}},
\qquad
a_t=(\widehat\lambda_t-\lambda)^\trans\psi_t.
\label{eq:saturation-charge}
\end{align}
Let $\beta_s:=1/8$ and define
\begin{align}
c_s&:=1-2\overline\alpha-\beta_s
-4\overline\ell(1+\overline\alpha)^2,
\notag\\
k_s&:=\overline\alpha^2(\beta_s^{-1}+4\overline\ell).
\label{eq:saturation-constants}
\end{align}
Under \eqref{eq:H-rls}, $c_s>0.467$.

\begin{proposition}[Exploitation regret with saturation energy]
\label{prop:saturation-segment}
Under the estimation, Gram, and unchanged-segment hypotheses of
Theorem~\ref{thm:segment-regret}, but without clipping feasibility, with
conditional probability at least $1-\delta_0$, simultaneously for every
prefix of length $n$,
\begin{align}
\sum_t(y_{t+1}-w_{t+1})^2
\le2\overline b^2\Bigg[&
\frac{2\overline Q_H}{c_s}
+\frac{8B_\eta^2}{c_s^2}\log\frac1{\delta_0}
+\frac{4B_\eta^2}{c_s}\Gamma(n)
\notag\\
&+\left(1+\frac{2k_s}{c_s}\right)\sum_t(d_t^{\mathrm{sat}})^2
\Bigg].
\label{eq:saturation-segment-bound}
\end{align}
\end{proposition}

\begin{proof}
Start from the exact identity \eqref{eq:potential-identity}. From
\eqref{eq:saturation-charge}, $|e_t|\le|a_t|+d_t^{\mathrm{sat}}$. Young's inequality gives
\begin{align*}
2\alpha a_te_t
&\le(2\overline\alpha+\beta_s)a_t^2
+\frac{\overline\alpha^2}{\beta_s}(d_t^{\mathrm{sat}})^2,\\
(\alpha e_t+\eta_{t+1}-a_t)^2
&\le4(1+\overline\alpha)^2a_t^2
+4\overline\alpha^2(d_t^{\mathrm{sat}})^2+2\eta_{t+1}^2.
\end{align*}
Using $\ell_t\le\overline\ell$ yields
\begin{align*}
Q_{t+1}-Q_t
\le-c_sa_t^2+2a_t\eta_{t+1}
+2\ell_t\eta_{t+1}^2+k_s(d_t^{\mathrm{sat}})^2.
\end{align*}
The same time-uniform exponential supermartingale as in
\eqref{eq:martingale-absorb}, now evaluated with $c_s$, and the same
elliptical-potential bound give
\begin{align*}
\sum_ta_t^2
\le\frac{2\overline Q_H}{c_s}
+\frac{8B_\eta^2}{c_s^2}\log\frac1{\delta_0}
+\frac{4B_\eta^2}{c_s}\Gamma(n)
+\frac{2k_s}{c_s}\sum_t(d_t^{\mathrm{sat}})^2.
\end{align*}
Finally,
$(y_{t+1}-w_{t+1})^2\le2\overline b^2(a_t^2+(d_t^{\mathrm{sat}})^2)$, which proves
\eqref{eq:saturation-segment-bound}.
\end{proof}

\noindent
Proposition~\ref{prop:saturation-segment} gives logarithmic
unchanged-segment regret when $\sum_t(d_t^{\mathrm{sat}})^2=O(\log((T+1)/\delta))$.
A global detection guarantee additionally requires a blockwise power gap
for the clipped loop or control of its saturation residual on every
potential detector block. The non-Schur gap
\eqref{eq:spectral-derived-gap} applies to the unclipped auxiliary;
Lemma~\ref{lem:sample-local-power} transfers it to the actual output
using the boundary test.

\begin{proposition}[Saturation lower bound]
\label{prop:saturation-lower-bound}
For every policy with $|u_t|\le B_u$,
\begin{align}
\Reg_T
\ge b_{\min}^2\sum_{t=0}^{T-1}
(|\lambda(\theta_t)^\trans\psi_t|-B_u)_+^2.
\label{eq:saturation-lower-bound}
\end{align}
Consequently, if the realized saturation energy grows faster than
$\log T$, logarithmic regret under \eqref{eq:a-regret} is impossible,
even when the parameter sequence is known.
\end{proposition}

\begin{proof}
For $x_t=\lambda(\theta_t)^\trans\psi_t$ and every feasible action,
\begin{align*}
|u_t-x_t|\ge\operatorname{dist}(x_t,[-B_u,B_u])
=(|x_t|-B_u)_+.
\end{align*}
The gain form \eqref{eq:gain-form} gives
$(y_{t+1}-w_{t+1})^2=b_1(\theta_t)^2(u_t-x_t)^2$.
Use $b_1(\theta_t)\ge b_{\min}$ and sum over $t$.
\end{proof}

\section{A UNIFIED DETECTOR FOR THE UNION CLASS}
\label{app:unified-detector}

\noindent
Algorithm~\ref{alg:stability-agnostic} implements the union-class detector
introduced in Section~\ref{sec:a-relaxations}. Its two tests use
constants supplied offline; the algorithm neither receives the realized
pair's class nor computes its spectral radius. The terms Schur and non-Schur
refer to the old-gain/new-plant matrix
$A_{\rm cl}(\lambda(\theta^-),\theta^+)$. The plant retains
Assumptions~\ref{ass:model}, \ref{ass:switching-stability}, and
\ref{ass:noise}; bounded applied inputs therefore preserve the deterministic
state bound.

\subsection{Complete Algorithm}

\medskip
\refstepcounter{algorithm}\label{alg:stability-agnostic}
\noindent\textbf{Algorithm \thealgorithm: Branch-free union-class PIECE-CD}
\smallskip
\hrule
\smallskip
\begin{enumerate}[leftmargin=1.8em,itemsep=0.25em,topsep=0.25em]
\item[] \textbf{Input:} $p,q,\Theta,T,\delta,B_u,\nu,\sigma_w^2,B_w$; the
exploration distribution and the safe model/excitation bounds used in
$B_\psi,\kappa_G,E_\theta$; $H$ satisfying
\eqref{eq:block-excitation-size}, \eqref{eq:H-rls}, and
\eqref{eq:finite-lag-H}; the offline union-class constants defining
$h_{\sf U}$ and $\underline\Delta_{\sf U}$ in
\eqref{eq:instance-visibility-gap}; the model-class bounds $r_B,R_A$ in
\eqref{eq:old-input-root-bound}--\eqref{eq:spectral-visibility-constants}
when the non-Schur route is used; the finite-horizon constants in
\eqref{eq:finite-lag-constants}; and $m_{\sf U},D_{\sf U}$ from
\eqref{eq:finite-lag-mD}. Set
$\zeta_H=E_\lambda(H,\delta_0)B_\psi$.

\item[] Initialize the restart time $r\leftarrow0$ and clear the energy
buffer.

\item[] \textbf{Loop.} While $r<T$, execute one cycle.
\begin{enumerate}[label=(\alph*),leftmargin=1.8em,itemsep=0.2em]
\item \textbf{Explore.} For
$t=r,\ldots,\min\{r+H-1,T-1\}$, draw a fresh exploratory input, apply it,
and observe $y_{t+1}$. If the horizon is reached during this step, stop.

\item \textbf{Estimate and initialize.} From
$\mathcal I_r=\{r+r_0,\ldots,r+H-1\}$, compute
\begin{align*}
V_r^\theta&:=\nu I_s+\sum_{t\in\mathcal I_r}\phi_t\phi_t^\trans,
&\widetilde\theta_r&:=(V_r^\theta)^{-1}
\sum_{t\in\mathcal I_r}\phi_ty_{t+1},\\
\widehat\theta_r&:=\Pi_\Theta(\widetilde\theta_r),
&\widehat b_r&:=b_1(\widehat\theta_r),
&\overline\lambda_r&:=\lambda(\widehat\theta_r).
\end{align*}
Initialize
\begin{align*}
V_{r+H}:=\nu I_d+\sum_{t\in\mathcal I_r}\psi_t\psi_t^\trans,
\qquad \widehat\lambda_{r+H}:=\overline\lambda_r.
\end{align*}
Clear the energy buffer and declare the predictable output-time grid
\begin{align*}
\mathcal S_r:=\{r+H+h_{\sf U}+jh_{\sf U}:j=0,1,\ldots\}.
\end{align*}

\item \textbf{Exploit and monitor.} For
$t=r+H,r+H+1,\ldots,T-1$, form
\begin{align*}
z_t&:=
\begin{cases}
\widehat\lambda_t^\trans\psi_t,
&|(\widehat\lambda_t-\overline\lambda_r)^\trans\psi_t|
\le g_H\norm{\psi_t},\\
\overline\lambda_r^\trans\psi_t,&\text{otherwise},
\end{cases}
&u_t&:=\operatorname{clip}_{[-B_u,B_u]}(z_t).
\end{align*}
Before observing $y_{t+1}$, mark that output as energy-eligible exactly
when $t+1\in\mathcal S_r$. Apply $u_t$, observe $y_{t+1}$, and form
$\psi_{t+1}$. Make the RLS update
\begin{align*}
x_{t+1}^\lambda&:=u_t-\frac{y_{t+1}}{\widehat b_r},
&V_{t+1}&:=V_t+\psi_t\psi_t^\trans,\\
\widehat\lambda_{t+1}
&:=\widehat\lambda_t+V_{t+1}^{-1}\psi_t
(x_{t+1}^\lambda-\widehat\lambda_t^\trans\psi_t).
\end{align*}

Set the boundary flag
\begin{align*}
B_{t+1}:=\one\!\left\{
|\overline\lambda_r^\trans\psi_{t+1}|>B_u+\zeta_H
\right\}.
\end{align*}
The strict inequality is part of the algorithm. If $y_{t+1}$ was marked
eligible, append $y_{t+1}^2$ to the energy buffer and retain only the most
recent $m_{\sf U}$ stored values. Set
\begin{align*}
E_{t+1}:=
\begin{cases}
\one\!\left\{
|\mathrm{buffer}|=m_{\sf U},\
\displaystyle\frac1{m_{\sf U}}\sum_{z\in\mathrm{buffer}}z
\ge\sigma_w^2+\frac{\underline\Delta_{\sf U}}2
\right\},&t+1\in\mathcal S_r,\\
0,&t+1\notin\mathcal S_r.
\end{cases}
\end{align*}

If $B_{t+1}=1$ or $E_{t+1}=1$, raise one alarm at output time $t+1$,
discard the current estimates, set $r\leftarrow t+1$, clear the buffer,
and return to step (a). If both flags equal one, still record only one
alarm. If neither flag crosses and $t=T-1$, stop; otherwise continue
exploitation.
\end{enumerate}
\end{enumerate}
\smallskip
\hrule
\medskip

\noindent
Eligibility is declared before $y_{t+1}$ is observed, so the sampling grid is
predictable. The boundary flag is checked after every exploitation output,
not only at energy-sampling times. It is therefore available before the next
control is chosen.

\subsection{Common Tuning}

\noindent
Under Assumption~\ref{ass:finite-lag-detectability},
\eqref{eq:instance-visibility-gap} specifies a common lag $h_{\sf U}$ and
separation $\underline\Delta_{\sf U}$. For the excess innovation power
$G_h$ defined in Appendix~\ref{app:relaxations},
\eqref{eq:finite-lag-gap-bound} gives
\begin{align*}
G_{h_{\sf U}}(A)\ge\underline\Delta_{\sf U}
\end{align*}
for every admitted mismatch $A=A_{\rm cl}(\lambda(\theta^-),\theta^+)$.
The finite-horizon bounds in \eqref{eq:finite-lag-constants}--
\eqref{eq:finite-lag-bounds}, the accuracy conditions
\eqref{eq:finite-lag-H}, and the window in \eqref{eq:finite-lag-mD}
therefore provide one calibration for all three detection conditions.

\subsection{Detection and Regret Guarantee}

\begin{theorem}[Branch-free union-class detector and regret guarantee]
\label{thm:stability-agnostic}
Retain Assumptions~\ref{ass:model}, \ref{ass:switching-stability}, and
\ref{ass:noise}. Replace global cross-regime clipping feasibility and stable
mismatch by closed own-regime feasibility
(Assumption~\ref{ass:own-regime-feasibility}) and the union
stable/direct/non-Schur condition
(Assumption~\ref{ass:finite-lag-detectability}), with the common
separation \eqref{eq:finite-lag-gap-bound} and finite-horizon bounds
\eqref{eq:finite-lag-constants}. Choose $H$ as the smallest feasible integer
satisfying \eqref{eq:block-excitation-size}, \eqref{eq:H-rls}, and
\eqref{eq:finite-lag-H}; choose $m_{\sf U},D_{\sf U}$ by
\eqref{eq:finite-lag-mD}; and
read Assumption~\ref{ass:dwell} with $D_{\sf U}$ in place of $D$.

\noindent
Then Algorithm~\ref{alg:stability-agnostic} satisfies, with probability at
least $1-\delta$, no false alarms and exactly one alarm
$\widehat\tau^{(c)}$ for every true change, with
\begin{align}
\tau^{(c)}<\widehat\tau^{(c)}
\le\tau^{(c)}+D_{\sf U},
\label{eq:unified-delay}
\end{align}
and
\begin{align}
\Reg_T
\le CK_DD_{\sf U}+(C+1)K_XH+(C+1)R_E(T,\delta_0).
\label{eq:unified-regret}
\end{align}
For fixed problem constants and fixed $C$,
$\Reg_T=O(\log((T+1)/\delta))$. More generally,
\begin{align}
\Reg_T=O\!\left((C+1)\log\frac{T+1}{\delta}\right).
\label{eq:unified-rate}
\end{align}
\end{theorem}

\begin{proof}
\noindent
If $T\le H$, the dwell condition forces $C=0$.
Algorithm~\ref{alg:stability-agnostic} performs only
truncated exploration and stops, so
$\Reg_T\le K_XT\le K_XH$, which is bounded by the right-hand side of
\eqref{eq:unified-regret}. Hence assume $T>H$ below.

\noindent
Read the restart, reference, and detector events from
Theorem~\ref{thm:detection} and Appendix~\ref{app:global} using the first
composite crossing and $D_{\sf U}$ in place of $D$. Fix a potential restart
$r$ and test endpoint $e$ for which the reference parameter remains active
throughout $[r,e)$. Condition on an
$\F_{r+H}$ history with a full exploration block and
$\G_r^{\rm est}$. Since $H>r_0$, the action and boundary-test regressors
are own-regime successors throughout the unchanged exploitation interval.
Lemma~\ref{lem:boundary-dichotomy} therefore excludes boundary crossings
and gives $e_t^2\le a_t^2$, as required by
Theorem~\ref{thm:segment-regret}. This includes a test at output time
$e=\tau$ when the first subsequent change occurs at control time $\tau$.

\noindent
For each fixed null energy block, the proof of
Lemma~\ref{lem:false-block} applies with
$(\Delta,B_o,m)$ replaced by
$(\underline\Delta_{\sf U},B_{\sf U},m_{\sf U})$.
Indeed, $|v_s|\le\eps_y\le\overline\eps_y$,
$\overline\eps_y^2\le\underline\Delta_{\sf U}/8$, and
$B_{\sf U}\ge\max\{B_w,\overline\eps_y\}$. The conditional crossing
probability is bounded by
\begin{align*}
&\exp\!\left(-\frac{m_{\sf U}\underline\Delta_{\sf U}^2}{32B_w^4}\right)
+\exp\!\left(-\frac{m_{\sf U}\underline\Delta_{\sf U}^2}
{512B_w^2\overline\eps_y^2}\right)
\le\frac{\delta_d}{2T^2},
\end{align*}
where the second term is omitted when $\overline\eps_y=0$.
If the actual test is reached, its trajectory agrees with the
alarm-suppressed continuation through that test. The bound therefore
controls the joint event that the reference succeeds, the test is reached
with the stated null history, and the threshold is crossed. It does not
condition on survival to the test. Before the first detector failure,
every false-alarm candidate has such an unchanged null history.

\noindent
After a change, any earlier crossing of either test completes detection.
Otherwise, the absence of a boundary crossing bounds the old ideal action's
saturation deficit by $2\zeta_H$ and the applied-action error by
$\overline\eps_u$. Lemma~\ref{lem:sample-local-power} then compares each
output in the complete post-change block with its $h_{\sf U}$-step auxiliary.
These auxiliaries are defined on the full alarm-suppressed continuation, so
Lemma~\ref{lem:sampled-square} applies without conditioning on survival of
either test. The event inclusion in
Lemma~\ref{lem:finite-lag-one-change} gives a composite alarm by
$\tau+D_{\sf U}$ except with conditional probability
$\delta_d/(4T^2)$ on each eligible $\F_\tau$ history.

\noindent
The null bounds are unioned over at most $T^2$ restart--endpoint pairs.
At a fixed true change, the successful-history active-cycle events are
disjoint $\F_\tau$-measurable subsets of the corresponding
$\mathcal A_{r,\tau}$. Multiplying the conditional missed-detection bound
by these indicators and summing over $r$ costs at most
$\delta_d/(4T^2)$ for that change. The chronological induction in
Theorem~\ref{thm:detection} and Appendix~\ref{app:global} therefore applies
with total detector failure probability at most $\delta_d$; boundary
crossings add no null failure probability on successful-reference
histories. Controller failures cost at most $3(T+1)\delta_0$, so the
allocation in \eqref{eq:confidence-allocation} is below $\delta$.
The delay intervals cost at most $CK_DD_{\sf U}$, exploration costs at
most $(C+1)K_XH$, and unchanged exploitation prefixes cost at most
$(C+1)R_E(T,\delta_0)$. This proves
\eqref{eq:unified-delay}--\eqref{eq:unified-regret}.

\noindent
It remains to check the order. Because
$\overline\eps_u=7L_\lambda B_\psi E_\theta(H,\delta_0)$ and
$\overline\eps_y=7\overline bL_\lambda B_\psi
E_\theta(H,\delta_0)$, the detector conditions follow from
\begin{align}
E_\theta(H,\delta_0)\le\epsilon_{\sf U}^{\rm uni}:=\min\Bigg\{
&\frac{b_{\min}}{16},
\frac{\sqrt{\underline\Delta_{\sf U}/8}}
{7\overline bL_\lambda B_\psi},
\frac{B_{\sf U}}{7\overline bL_\lambda B_\psi},
\notag\\[-0.2em]
&\frac{\underline\Delta_{\sf U}}
{112B_{\sf U}L_{u,{\sf U}}L_\lambda B_\psi}
\Bigg\},
\label{eq:unified-epsilon-star}
\end{align}
Omit the last entry when $L_{u,{\sf U}}=0$. Appendix~\ref{app:tuning} and
\eqref{eq:finite-lag-mD} then prove \eqref{eq:unified-rate}.
\end{proof}


\begin{thebibliography}{}

\bibitem[Abbasi-Yadkori et~al., 2011]{abbasi2011improved}
Abbasi-Yadkori, Y., P{\'a}l, D., and Szepesv{\'a}ri, C. (2011).
\newblock Improved algorithms for linear stochastic bandits.
\newblock In {\em Advances in Neural Information Processing Systems},
  volume~24.

\bibitem[Audibert and Bubeck, 2009]{audibert2009minimax}
Audibert, J.-Y. and Bubeck, S. (2009).
\newblock Minimax policies for adversarial and stochastic bandits.
\newblock In {\em Proceedings of the 22nd Annual Conference on Learning
  Theory}.

\bibitem[Azuma, 1967]{azuma1967weighted}
Azuma, K. (1967).
\newblock Weighted sums of certain dependent random variables.
\newblock {\em Tohoku Mathematical Journal}, 19(3):357--367.

\bibitem[Becker et~al., 1985]{becker1985adaptive}
Becker, A., Kumar, P.~R., and Wei, C.-Z. (1985).
\newblock Adaptive control with the stochastic approximation algorithm:
  Geometry and convergence.
\newblock {\em IEEE Transactions on Automatic Control}, 30(4):330--338.

\bibitem[Besbes et~al., 2014]{besbes2014stochastic}
Besbes, O., Gur, Y., and Zeevi, A. (2014).
\newblock Stochastic multi-armed-bandit problem with non-stationary rewards.
\newblock In {\em Advances in Neural Information Processing Systems},
  volume~27, pages 199--207.

\bibitem[Besson et~al., 2022]{besson2022efficient}
Besson, L., Kaufmann, E., Maillard, O.-A., and Seznec, J. (2022).
\newblock Efficient change-point detection for tackling piecewise-stationary
  bandits.
\newblock {\em Journal of Machine Learning Research}, 23(77):1--40.

\bibitem[Cao et~al., 2019]{cao2019nearly}
Cao, Y., Wen, Z., Kveton, B., and Xie, Y. (2019).
\newblock Nearly optimal adaptive procedure with change detection for
  piecewise-stationary bandit.
\newblock In {\em Proceedings of the Twenty-Second International Conference on
  Artificial Intelligence and Statistics}, volume~89 of {\em Proceedings of
  Machine Learning Research}, pages 418--427. PMLR.

\bibitem[Garivier and Moulines, 2011]{garivier2011upper}
Garivier, A. and Moulines, E. (2011).
\newblock On upper-confidence bound policies for switching bandit problems.
\newblock In {\em Proceedings of the 22nd International Conference on
  Algorithmic Learning Theory}, pages 174--188.

\bibitem[Goodwin et~al., 1981]{goodwin1981discrete}
Goodwin, G.~C., Ramadge, P.~J., and Caines, P.~E. (1981).
\newblock Discrete time stochastic adaptive control.
\newblock {\em SIAM Journal on Control and Optimization}, 19(6):829--853.

\bibitem[Hoeffding, 1963]{hoeffding1963probability}
Hoeffding, W. (1963).
\newblock Probability inequalities for sums of bounded random variables.
\newblock {\em Journal of the American Statistical Association},
  58(301):13--30.

\bibitem[Horn and Johnson, 2013]{horn2013matrix}
Horn, R.~A. and Johnson, C.~R. (2013).
\newblock {\em Matrix Analysis}.
\newblock Cambridge University Press, Cambridge, 2 edition.

\bibitem[Howard et~al., 2020]{howard2020timeuniform}
Howard, S.~R., Ramdas, A., McAuliffe, J., and Sekhon, J. (2020).
\newblock Time-uniform {C}hernoff bounds via nonnegative supermartingales.
\newblock {\em Probability Surveys}, 17:257--317.

\bibitem[Huang et~al., 2026]{huang2026detection}
Huang, Y.-H., Gerogiannis, A., Bose, S., and Veeravalli, V.~V. (2026).
\newblock Detection augmented bandit procedures for piecewise stationary
  {MAB}s: A modular approach.
\newblock {\em IEEE Transactions on Information Theory}.


\bibitem[Kumar and Praly, 1987]{kumar1987self}
Kumar, P.~R. and Praly, L. (1987).
\newblock Self-tuning trackers.
\newblock {\em SIAM Journal on Control and Optimization}, 25(4):1053--1071.

\bibitem[Lai, 1986]{lai1986asymptotically}
Lai, T.~L. (1986).
\newblock Asymptotically efficient adaptive control in stochastic regression
  models.
\newblock {\em Advances in Applied Mathematics}, 7(1):23--45.

\bibitem[Lai and Robbins, 1985]{lai1985asymptotically}
Lai, T.~L. and Robbins, H. (1985).
\newblock Asymptotically efficient adaptive allocation rules.
\newblock {\em Advances in Applied Mathematics}, 6(1):4--22.

\bibitem[Lai and Wei, 1987]{lai1987asymptotically}
Lai, T.~L. and Wei, C.-Z. (1987).
\newblock Asymptotically efficient self-tuning regulators.
\newblock {\em SIAM Journal on Control and Optimization}, 25(2):466--481.

\bibitem[Liu et~al., 2018]{liu2018change}
Liu, F., Lee, J., and Shroff, N.~B. (2018).
\newblock A change-detection based framework for piecewise-stationary
  multi-armed bandit problem.
\newblock In {\em Proceedings of the AAAI Conference on Artificial
  Intelligence}, volume~32, pages 3651--3658.

\bibitem[Mukherjee and Maillard, 2019]{mukherjee2019distribution}
Mukherjee, S. and Maillard, O.-A. (2019).
\newblock Distribution-dependent and time-uniform bounds for piecewise i.i.d.
  bandits.
\newblock {\em arXiv preprint arXiv:1905.13159}.

\bibitem[Simchowitz and Foster, 2020]{simchowitz2020naive}
Simchowitz, M. and Foster, D. (2020).
\newblock Naive exploration is optimal for online {LQR}.
\newblock In {\em Proceedings of the 37th International Conference on Machine
  Learning}, volume 119 of {\em Proceedings of Machine Learning Research},
  pages 8937--8948. PMLR.

\bibitem[Singh et~al., 2024]{singhfinite}
Singh, R., Mete, A., Kar, A., and Kumar, P. R. (2024).
\newblock Finite time logarithmic regret bounds for self-tuning regulation.
\newblock In {\em Proceedings of the 41st International Conference on Machine
  Learning}, volume 235 of {\em Proceedings of Machine Learning Research},
  pages 45571--45636. PMLR.

\bibitem[Tartakovsky et~al., 2014]{tartakovsky2014sequential}
Tartakovsky, A., Nikiforov, I., and Basseville, M. (2014).
\newblock {\em Sequential Analysis: Hypothesis Testing and Changepoint
  Detection}.
\newblock CRC Press.

\bibitem[Tropp, 2011]{tropp2011freedman}
Tropp, J.~A. (2011).
\newblock Freedman's inequality for matrix martingales.
\newblock {\em Electronic Communications in Probability}, 16:262--270.

\bibitem[Veeravalli and Banerjee, 2014]{veeravalli2014quickest}
Veeravalli, V.~V. and Banerjee, T. (2014).
\newblock Quickest change detection.
\newblock {\em Academic Press Library in Signal Processing}, 3:209--255.

\bibitem[Ville, 1939]{ville1939etude}
Ville, J. (1939).
\newblock {\em {\'E}tude critique de la notion de collectif}.
\newblock Gauthier-Villars, Paris.

\bibitem[Xie et~al., 2021]{xie2021sequential}
Xie, L., Zou, S., Xie, Y., and Veeravalli, V.~V. (2021).
\newblock Sequential (quickest) change detection: Classical results and new
  directions.
\newblock {\em IEEE Journal on Selected Areas in Information Theory},
  2(2):494--514.

\bibitem[Zhou et~al., 2020]{zhou2020near}
Zhou, H., Wang, L., Varshney, L.~R., and Lim, E.-P. (2020).
\newblock A near-optimal change-detection based algorithm for
  piecewise-stationary combinatorial semi-bandits.
\newblock In {\em Proceedings of the AAAI Conference on Artificial
  Intelligence}, volume~34, pages 6933--6940.

\end{thebibliography}
\end{document}